\documentclass[10pt,twocolumn,letterpaper]{article}

\usepackage{cvpr}              
\definecolor{cvprblue}{rgb}{0.21,0.49,0.74}
\usepackage[pagebackref,breaklinks,colorlinks,allcolors=cvprblue]{hyperref}

\usepackage{graphicx}
\usepackage{amsmath}

\usepackage{multirow}
\usepackage{rotating}
\usepackage{array}
\usepackage{bbding}
\usepackage{wrapfig}  

\usepackage{bm} 
\usepackage{amssymb} 
\usepackage{amsthm}  
\usepackage[accsupp]{axessibility}

\def\paperID{9173} 
\def\confName{CVPR}
\def\confYear{2026}

\title{LaDy: Lagrangian-Dynamic Informed Network for Skeleton-based Action Segmentation via Spatial-Temporal Modulation}

\author{
    Haoyu Ji\textsuperscript{1,$\dagger$}, 
    Xueting Liu\textsuperscript{2,$\dagger$},
    Yu Gao\textsuperscript{1},
    Wenze Huang\textsuperscript{1}, 
    Zhihao Yang\textsuperscript{1}, \\
    Weihong Ren\textsuperscript{1},
    Zhiyong Wang\textsuperscript{1,*},
    Honghai Liu\textsuperscript{1,3}\\[1mm]
    \textsuperscript{1}Harbin Institute of Technology, Shenzhen \quad \\
    \textsuperscript{2}Southern University of Science and Technology  \quad 
    \textsuperscript{3}Southeast University \\
}

\begin{document}
\maketitle

{
\let\thefootnote\relax\footnotetext{$^\dagger$Equal contribution \quad $^*$Corresponding author}
}

\begin{abstract}
Skeleton-based Temporal Action Segmentation (STAS) aims to densely parse untrimmed skeletal sequences into frame-level action categories. However, existing methods, while proficient at capturing spatio-temporal kinematics, neglect the underlying physical dynamics that govern human motion. This oversight limits inter-class discriminability between actions with similar kinematics but distinct dynamic intents, and hinders precise boundary localization where dynamic force profiles shift. To address these, we propose the Lagrangian-Dynamic Informed Network (LaDy), a framework integrating principles of Lagrangian dynamics into the segmentation process. Specifically, LaDy first computes generalized coordinates from joint positions and then estimates Lagrangian terms under physical constraints to explicitly synthesize the generalized forces. To further ensure physical coherence, our Energy Consistency Loss enforces the work-energy theorem, aligning kinetic energy change with the work done by the net force. The learned dynamics then drive a Spatio-Temporal Modulation module: Spatially, generalized forces are fused with spatial representations to provide more discriminative semantics. Temporally, salient dynamic signals are constructed for temporal gating, thereby significantly enhancing boundary awareness. Experiments on challenging datasets show that LaDy achieves state-of-the-art performance, validating the integration of physical dynamics for action segmentation. Code is available at \url{https://github.com/HaoyuJi/LaDy}.
\end{abstract}    
\section{Introduction}
\label{sec:intro}

Skeleton-based Temporal Action Segmentation (STAS) has emerged as a critical research frontier, aiming to parse untrimmed skeletal sequences into dense, frame-level actions. The skeleton modality offers a compelling balance of computational efficiency and intrinsic robustness to environmental variations~\cite{xin2023transformer,LaSA}, making it ideal for applications from human-computer interaction~\cite{wang2025foundation} to autonomous driving~\cite{huo20253d,guo2024skeleton} and clinical surveillance~\cite{AutoENP,Auto_freezing}.

\begin{figure}[tb]
    \centering
    \includegraphics[width=0.96\linewidth]{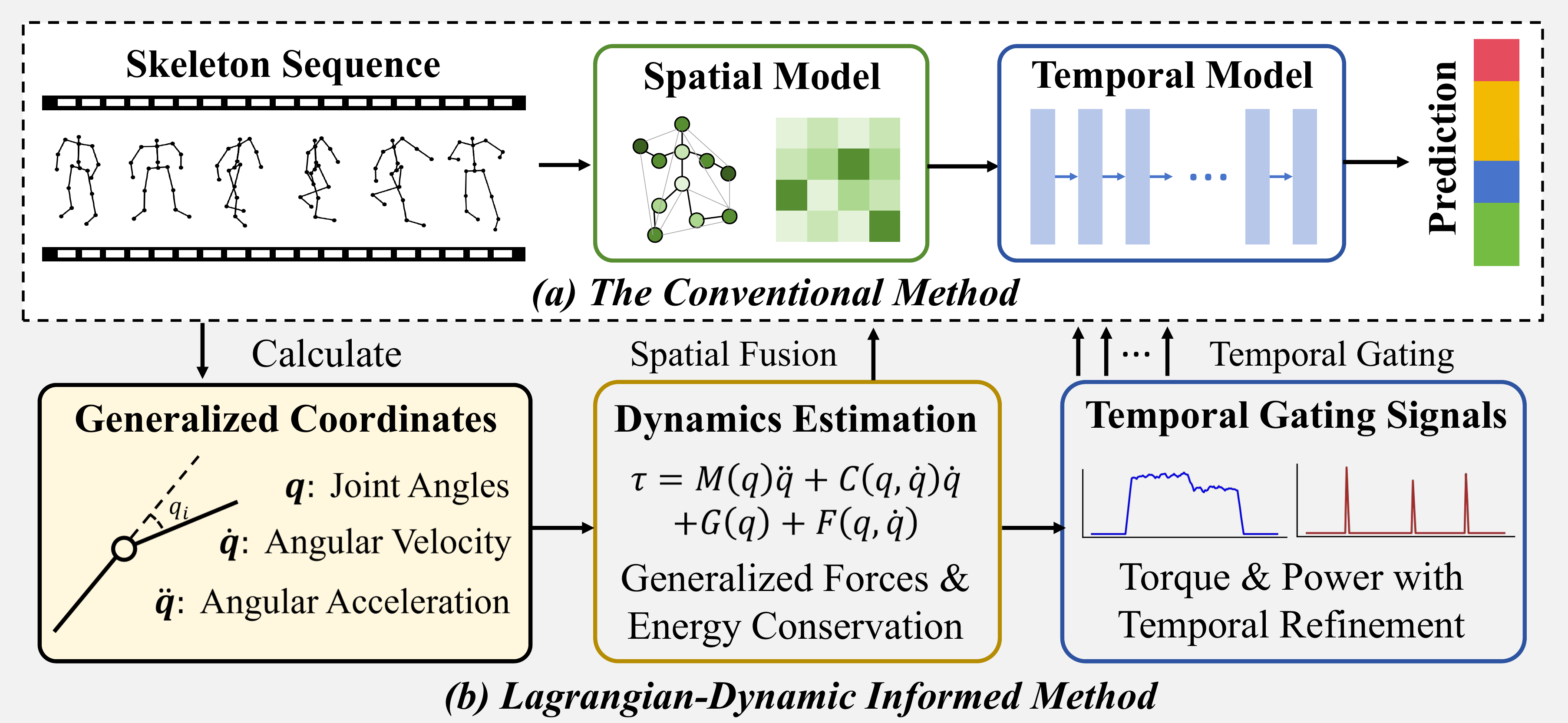}
    
    \caption{Conceptual comparison of STAS frameworks.
     (a) Conventional models capture only spatio-temporal kinematic patterns.
     (b) Our LaDy framework introduces Lagrangian dynamics to derive generalized forces that modulate the spatio-temporal features, leading to improved discriminability and boundary localization.}
    \label{fig:1}
\end{figure}

Prevailing STAS methods employ deep architectures to capture intricate spatio-temporal dependencies, as shown in Fig.~\ref{fig:1}(a). Spatially, Graph Convolutional Networks (GCNs)~\cite{ST-GCN,MS-GCN,IDT-GCN} and attention mechanisms~\cite{STGA-Net,SFA} model the non-Euclidean topology of the body, while temporally, Temporal Convolutional Networks (TCNs)~\cite{TCN,MS-TCN,MS-GCN} and Transformers~\cite{Asformer,ME-ST,DeST} capture long-range dependencies. These methods excel at learning complex spatio-temporal kinematic patterns---the ``what'' and ``how'' of motion---achieving considerable success.

However, a fundamental limitation underlies current methods: they operate almost exclusively in the domain of kinematics, remaining largely agnostic to the underlying dynamics that govern the motion. Human movement is not merely a sequence of poses; it is the physical manifestation of forces dictated by the principles of classical mechanics. Neglecting this dynamic foundation discards a rich source of information defining the intent and causality of an action.
This ``dynamics-agnostic'' approach leads to two significant, unresolved challenges. First, it compromises inter-class discriminability for actions that share similar kinematic profiles but arise from disparate dynamic intents (e.g., ``Moving Cart'' and ``Walking'', which are kinematically similar but dynamically disparate). 
Second, it hinders precise boundary localization, because action transitions are not merely kinematic shifts but are fundamentally defined by abrupt changes in the dynamic force profile---changes that kinematic-only models often blur.

To bridge this ``kinematics-dynamics'' gap, we introduce the Lagrangian-Dynamic Informed Network (LaDy), which integrates the principles of Lagrangian dynamics directly into action segmentation, as illustrated in Fig.~\ref{fig:1}(b). Instead of solely observing how a person moves, LaDy reasons about why they move by explicitly estimating the generalized forces (joint torques) driving the skeletal system. Specifically, our Lagrangian Dynamics Synthesis (LDS) module computes generalized coordinates from joint configurations and then utilizes physically-constrained modules to estimate the constituent Lagrangian terms (e.g., Inertia, Coriolis, Gravity) for synthesizing the generalized forces. To ensure these forces are physically coherent, we introduce an Energy Consistency Loss (ECLoss). This physics-based regularizer enforces the work-energy theorem, supervising the network by aligning the system's kinetic energy change with the net work done by all forces. Finally, the learned dynamic features are leveraged by a Spatio-Temporal Modulation (STM) mechanism, which fuses forces with spatial representations and uses salient dynamic signals for hierarchical temporal gating.

We validate LaDy through extensive experiments on six challenging STAS datasets~\cite{PKU-MMD,LARA,MCFS,tcg}. Our method achieves new state-of-the-art performance, validating our central hypothesis that integrating physical dynamics provides a more discriminative and precise foundation for action segmentation. Our main contributions are threefold:
(i) We propose LaDy, the first framework to introduce Lagrangian dynamics into STAS. Its core is the Lagrangian Dynamics Synthesis (LDS) module, which estimates physics-informed generalized forces to assist segmentation.
(ii) We introduce an Energy Consistency Loss (ECLoss), a physics-based regularizer that enforces the work-energy theorem to ensure the physical coherence of the forces.
(iii) We design a dynamics-driven Spatio-Temporal Modulation (STM) that leverages the forces via spatial fusion and hierarchical temporal gating, enhancing action discriminability and boundary localization.

\section{Related Work}
\label{sec:related_work}

\noindent\textbf{Temporal Action Segmentation.}
Temporal Action Segmentation (TAS) research is largely differentiated by input modality.
Video-based TAS methods focus on learning from RGB or flow features. The field has evolved by enhancing core temporal architectures, including TCNs~\cite{TCN, MS-TCN, MS-TCN++, Global2local, RF-Next, DPRN, ETSN, C2F-TCN}, Graph-based models~\cite{GTRM, DTGRM, Semantic2Graph}, Transformers~\cite{Asformer, UVAST, EUT, LTContext, BaFormer, FACT}, diffusion models~\cite{Diffusion,Diffact++,Difformer} and boundary refinement techniques~\cite{BCN,ASRF}. 
Skeleton-based TAS, our focus, models spatio-temporal dependencies from skeleton sequences. The dominant paradigm, established by MS-GCN~\cite{MS-GCN}, combines GCNs for spatial topology with TCNs for temporal modeling. This pipeline has been the subject of numerous refinements, including decoupled architectures~\cite{DeST}, motion-aware enhancements~\cite{MTST-GCN}, discriminative graph networks~\cite{IDT-GCN}, contrastive learning~\cite{SCSAS}, multi-scale pyramids~\cite{MSTAS} and data processing strategies~\cite{Tai_Chi, LAC, CTC}. Alternative approaches leverage spatio-temporal attention~\cite{SFA, STGA-Net, ME-ST} and language-assisted modeling~\cite{LaSA, LPL, TRG-Net}.
However, these methods are fundamentally kinematic, modeling motion appearance while ignoring its underlying dynamics, thus limiting inter-class discriminability and boundary precision. To our knowledge, LaDy is the first framework to introduce Lagrangian dynamics into STAS, directly addressing this gap.

\noindent\textbf{Physics-Informed Dynamics Modeling.}
Integrating physical principles into deep learning is a significant research area, often utilizing Lagrangian~\cite{DeLaN, La-caVAE, LNNs} or Hamiltonian~\cite{HNNs} mechanics. In the context of human motion, these principles are leveraged in three main ways:
(1) As Feature Enhancers: Physics-derived quantities (e.g., forces, torques) are used as auxiliary inputs or features to deep networks for tasks like motion prediction~\cite{PIMNet, PhysMoP} or action generation~\cite{FinePhys}.
(2) As Supervisory Signals: Physics is used as a loss function or regularizer to enforce physical plausibility, often for trajectory refinement~\cite{PhysPT, HTP}, pose correction~\cite{xie2021physics}, or guiding representation learning via physics-augmented decoders~\cite{PhysDiff, guo2023physics}.
(3) As Control Signals: Dynamic properties are used to modulate the information flow within a network, such as using velocity to gate LSTMs~\cite{VA-STA-ResLSTM} or fusing torques in a Graph-GRU~\cite{MS-MGNN}.
Despite this progress, leveraging explicit, complete skeletal dynamics for spatio-temporal feature processing remains underexplored. Our work presents the first physics-informed network for STAS, leveraging the Lagrangian dynamics equation and the work-energy theorem to estimate forces for spatio-temporal modulation.

\section{Method}
\label{sec:method}

\begin{figure*}[t]
  \centering
   \includegraphics[width=0.95\linewidth]{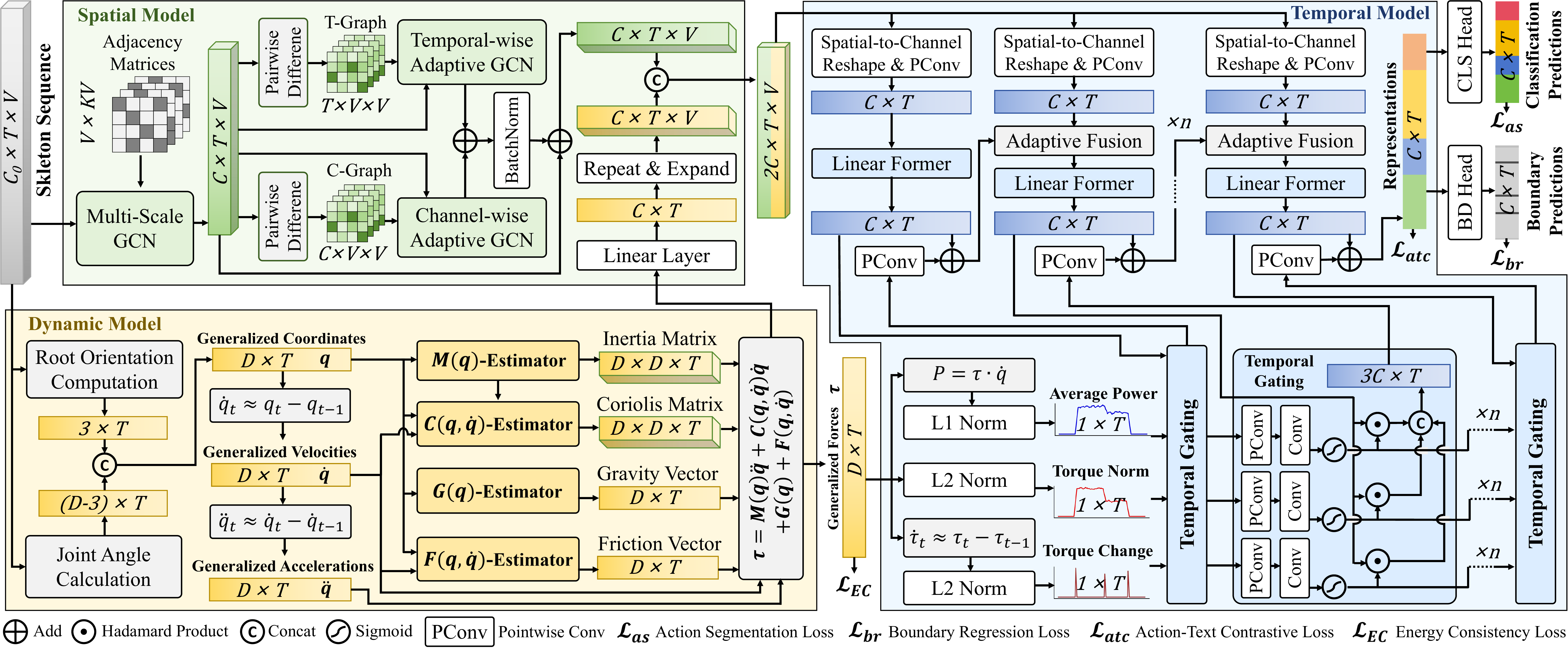}

   \caption{Overview of the LaDy framework. The Lagrangian Dynamic Model (lower-left) synthesizes generalized forces from coordinates, constrained by an Energy Consistency Loss for physical coherence. Concurrently, the Spatial Model (upper-left) extracts kinematic features (GCNs) and fuses them with these physics-aware dynamics. The multi-stage Temporal Model (right) processes the fused spatial representation, where each stage is hierarchically gated by force-driven signals to enhance boundary awareness.}
   \label{fig:2}
\end{figure*}

The goal of STAS is to densely parse an untrimmed skeletal sequence. We denote the input sequence as $X\in \mathbb{R}^{C_0 \times T \times V}$, with $T$ frames, $V$ joints and $C_0$ input channels (e.g., coordinates). STAS maps $X$ to a sequence of frame-level one-hot labels $Y \in \mathbb{R}^{Q \times T}$, where $Q$ is the number of action classes.

\subsection{Overall Architecture}
\label{sec:3.1}

As shown in Fig.~\ref{fig:2}, LaDy achieves action segmentation through two parallel streams: a standard and conventional Spatio-Temporal Model as main branch (top) and our novel Lagrangian Dynamic Model as auxiliary branch (bottom), which interact via Spatio-Temporal Modulation (STM).

The Lagrangian Dynamic Model, the core of our innovation, first computes generalized coordinates, velocities, and accelerations ($q, \dot{q}, \ddot{q} \in \mathbb{R}^{T \times D}$) from input coordinates $X$. These are fed into the physically-constrained Lagrangian Dynamics Synthesis (LDS) module to estimate the generalized forces $\tau \in \mathbb{R}^{T \times D}$ (Sec.~\ref{sec:3.2}). These forces are supervised by our Energy Consistency Loss (ECloss) to ensure physical rationality (Sec.~\ref{sec:3.3}). This branch outputs two signals for modulation: (1) a dynamic feature $F_{dyn} \in \mathbb{R}^{C \times T \times V}$ (via a Linear layer) and (2) three salient dynamic signals (Power, Torque, Torque Change) $\in \mathbb{R}^{1 \times T}$.

In the main branch, the Spatial Model first extracts kinematic features $F_{kin}$ from $X$ using Multi-Scale~\cite{DeST} and Adaptive GCNs~\cite{MTST-GCN}. Spatial modulation (Sec.~\ref{sec:3.4}) then fuses $F_{kin}$ with the dynamic feature $F_{dyn}$ to produce an enriched representation $F_{sp} \in \mathbb{R}^{2C \times T \times V}$. $F_{sp}$ is then fed into the $L$-stage Temporal Model. Each stage comprises a Linear Transformer~\cite{Linformer1}, an adaptive fusion module, and a temporal modulation (Sec.~\ref{sec:3.4}), which performs hierarchical gating driven by three salient dynamic signals.

The final $L$-th stage representation $F_R$ is fed to Classification and Boundary Heads, yielding initial classification predictions $Y_c^0$ and boundary predictions $Y_b^0$. $F_R$ is also supervised by an action-text contrastive loss ($\mathcal{L}_{atc}$), following~\cite{LaSA,TRG-Net}. These initial predictions are refined by a multi-stage refinement module to produce final outputs $Y_c^F$ and $Y_b^F$. The model is trained end-to-end with a composite objective: standard segmentation losses~\cite{DeST,ASRF}, $\mathcal{L}_{atc}$, and our $\mathcal{L}_{EC}$. Inference employs the post-processing~\cite{ASRF} scheme.

\subsection{Lagrangian Dynamics Synthesis}
\label{sec:3.2}

The Lagrangian Dynamics Synthesis (LDS) module moves beyond black-box estimation, synthesizing generalized forces from kinematics by structurally parameterizing the Lagrangian equation with embedded physical constraints.

\subsubsection{Generalized Coordinates Computation}
The skeleton is modeled as an open kinematic chain $\mathcal{G} = (\mathcal{V}, \mathcal{E})$, where $\mathcal{V}$ is the set of $V$ joints and $\mathcal{E}$ defines the bone connectivity. The generalized coordinates $q$ comprise the root orientation $q^{root}$ and local joint rotations $q^{local}$.

\textbf{Root Orientation} ($q^{root}$):
We define a time-varying local body frame $\{ \vec{x}_{axis}(t), \vec{y}_{axis}(t), \vec{z}_{axis}(t) \}$ at the root joint $j_{ro}$-SpineBase using the Cartesian positions $p^j \in \mathbb{R}^3$ of $j_{ro}$ and its neighbors ($j_{m}$-SpineMid, $j_{rh}$-RightHip, $j_{lh}$-LeftHip). The primary axes are:
\begin{equation}
\vec{v}_{y} = p^{j_{m}}(t) - p^{j_{ro}}(t), \quad
\vec{v}_{x, raw} = p^{j_{rh}}(t) - p^{j_{lh}}(t).
\end{equation}
An orthonormal basis is obtained via Gram-Schmidt orthogonalization:
\begin{equation}
\begin{gathered}
\vec{y}_{axis} = \mathcal{N}(\vec{v}_{y}), \quad
\vec{z}_{axis} = \mathcal{N}(\vec{v}_{x, raw} \times \vec{y}_{axis}), \\
\vec{x}_{axis} = \mathcal{N}(\vec{y}_{axis} \times \vec{z}_{axis}),
\end{gathered}
\end{equation}
where $\mathcal{N}(\cdot)$ is L2 normalization. These axes form the rotation matrix $R_{root}(t) = [\vec{x}_{axis} | \vec{y}_{axis} | \vec{z}_{axis}] \in SO(3)$, which is converted to an axis-angle representation $q^{root}(t) \in \mathbb{R}^3$.

\textbf{Local Joint Rotations ($q^{local}$):}
For each joint $j$ with parent $p = pa(j)$ and grandparent $gp = pa(p)$, we define the parent bone $v_{pa}(t) = p^p(t) - p^{gp}(t)$ and child bone $v_{child}(t) = p^j(t) - p^p(t)$. The local rotation $q^j(t) \in \mathbb{R}^3$ is the axis-angle vector that aligns the direction of $v_{pa}(t)$ with $v_{child}(t)$:
\begin{equation}
\begin{gathered}
\vec{a}(t) = \mathcal{N}( \mathcal{N}(v_{pa}(t)) \times \mathcal{N}(v_{child}(t)) ), \\
\theta(t) = \arccos( \mathcal{N}(v_{pa}(t)) \cdot \mathcal{N}(v_{child}(t)) ), \\
q^j(t) = \theta(t) \cdot \vec{a}(t).
\end{gathered}
\end{equation}

Note: For 2D data, $q^{root} \in \mathbb{R}^1$ is the root bone angle and $q^j \in \mathbb{R}^1$ is the signed angle between $v_{pa}$ and $v_{child}$.

\textbf{State Variables Assembly:}
The generalized coordinate vector $q(t) \in \mathbb{R}^D$ is the concatenation of $q^{root}(t)$ and all $q^j(t)$ for $j \in \mathcal{V}' \subset \mathcal{V}$. We approximate the generalized velocity $\dot{q}$ and acceleration $\ddot{q}$ using first-order backward finite differences (assuming $\Delta t=1$ and $q(-1)=0$):
\begin{equation}
\dot{q}(t) = q(t) - q(t-1), \quad
\ddot{q}(t) = \dot{q}(t) - \dot{q}(t-1).
\end{equation}
The resulting state variables $\{q, \dot{q}, \ddot{q}\} \in \mathbb{R}^{D \times T}$ serve as the input to the dynamics estimation.

\subsubsection{Physics-Constrained Dynamics Estimation}

\begin{figure}[tb]
    \centering
    \includegraphics[width=\linewidth]{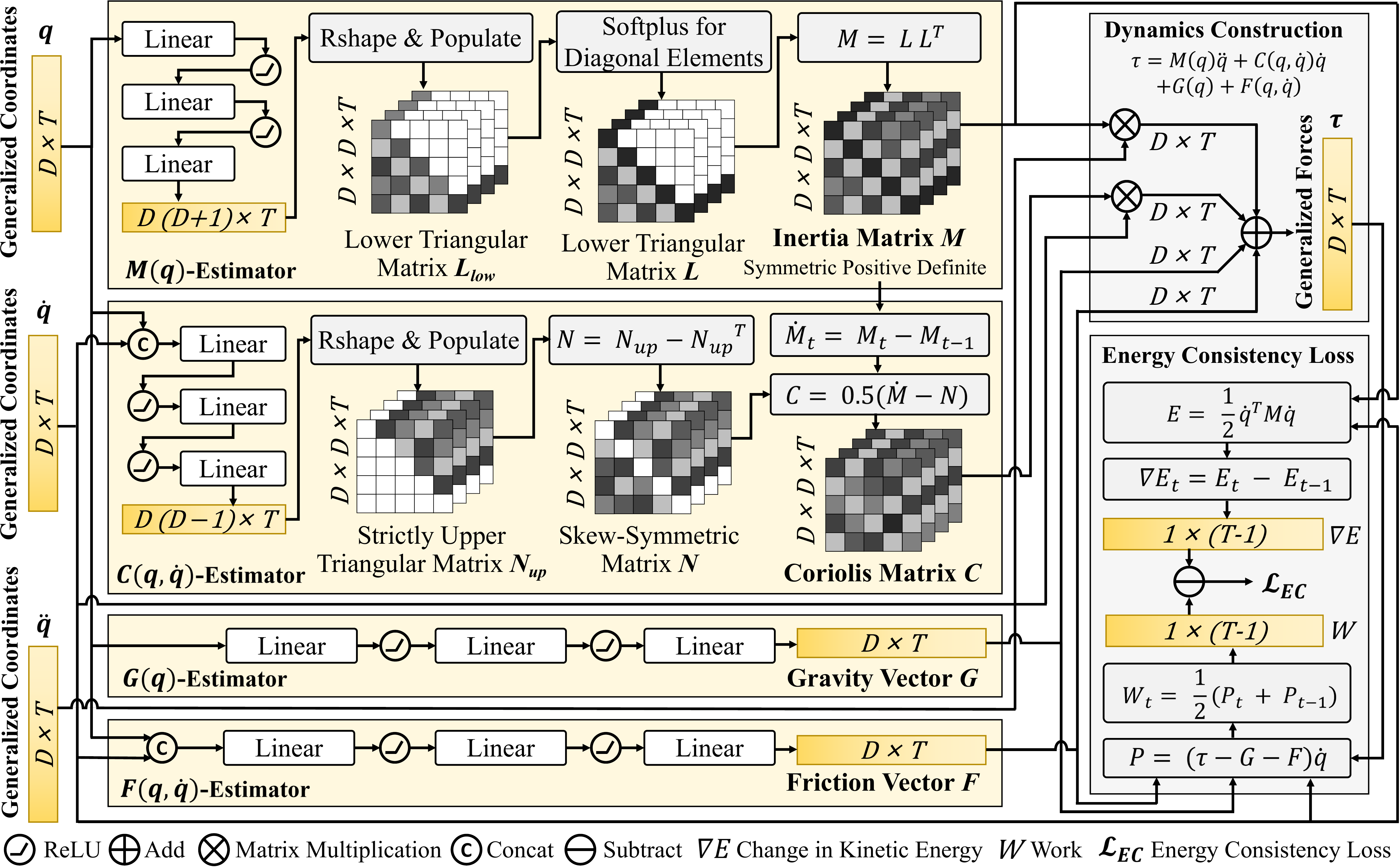}
    
    \caption{Physics-constrained dynamics synthesis and energy-based supervision. Generalized states ($q, \dot{q}, \ddot{q}$) are fed into physics-constrained estimators to derive the Lagrangian terms ($M, C, G, F$). These terms are then assembled via the Lagrangian equation to synthesize the generalized forces ($\tau$). Concurrently, the Energy Consistency Loss ($\mathcal{L}_{EC}$) regularizes the forces.}
    \label{fig:3}
\end{figure}

We estimate the terms of the Lagrangian dynamics equation, $\tau = M(q)\ddot{q} + C(q, \dot{q})\dot{q} + G(q) + F(q, \dot{q})$, using four neural networks ($\mathcal{F}_M, \mathcal{F}_C, \mathcal{F}_G, \mathcal{F}_F$) that embed physical constraints, as illustrated in Fig.~\ref{fig:3}.

\textbf{Inertia Matrix $M(q)$:}
The inertia matrix $M(q) \in \mathbb{R}^{D \times D \times T}$ must be symmetric positive definite. We enforce this by parameterizing $M(q)$ via Cholesky decomposition, $M(q) = L(q)L(q)^T$. An MLP, $\mathcal{F}_M$, predicts a vector of $D(D+1)/2$ elements from $q$, which populates a lower-triangular matrix $L^{low}(q)$. To ensure $M(q)$ is positive definite, the diagonal entries of $L(q)$ are enforced to be positive:
\begin{equation}
L_{ii}(q) = \text{softplus}(L^{low}_{ii}(q)) + \epsilon,
\end{equation}
where $L_{ii}'(q)$ is the $i$-th diagonal entry of $L(q)$. The small constant $\epsilon > 0$. The matrix is then $M(q) = L(q)L(q)^T$.

\textbf{Coriolis Matrix $C(q, \dot{q})$:}
The Coriolis and centrifugal matrix $C(q, \dot{q}) \in \mathbb{R}^{D \times D \times T}$ is constrained by the passivity property, requiring $\dot{M}(q) - 2C(q, \dot{q})$ to be skew-symmetric. We enforce this by defining $N = \dot{M} - 2C$, where $N(q, \dot{q}) = -N(q, \dot{q})^T$ is a skew-symmetric matrix. An MLP, $\mathcal{F}_C$, takes the channel-wise concatenation $\text{Concat}(q, \dot{q})$ as input and outputs a vector of $D(D-1)/2$ elements, which populates a strictly upper-triangular matrix $N^{up}(q, \dot{q})$. The full matrices $N$ and $C$ are then constructed as:
\begin{equation}
\begin{gathered}
N(q, \dot{q}) = N^{up}(q, \dot{q}) - N^{up}(q, \dot{q})^T, \\
C(q, \dot{q}) = 0.5(\dot{M}(q) - N(q, \dot{q})),
\end{gathered}
\end{equation}
$\dot{M}(q)$ is approximated via finite differences, $\dot{M}(q(t)) \approx M(q(t)) - M(q(t-1))$. This construction mathematically guarantees the passivity constraint.

\textbf{Gravity Vector $G(q)$ and Other Forces $F(q, \dot{q})$:}
The gravitational torque $G(q) \in \mathbb{R}^{D \times T}$ and non-conservative forces $F(q, \dot{q}) \in \mathbb{R}^{D \times T}$ (e.g., friction, external interactions) are estimated directly using two separate MLPs:
\begin{equation}
G(q) = \mathcal{F}_G(q), \quad
F(q, \dot{q}) = \mathcal{F}_F(\text{Concat}(q, \dot{q})).
\end{equation}

\textbf{Generalized Force Synthesis:}
The estimated generalized force $\tau \in \mathbb{R}^{D \times T}$ is synthesized by assembling the learned components according to the Lagrangian equation:
\begin{equation}
\tau = M(q)\ddot{q} + C(q, \dot{q})\dot{q} + G(q) + F(q, \dot{q}),
\end{equation}
where the products (e.g., $M(q)\ddot{q}$ and $C(q, \dot{q})\dot{q}$) denote batched matrix-vector multiplications performed independently at each time step $t$. This force $\tau$, along with $M$, $G$, $F$, and $\dot{q}$, is then passed to the energy consistency loss and spatio-temporal modulation module.

\subsection{Energy Consistency Loss}
\label{sec:3.3}

To ensure physically coherent force synthesis, we introduce the Energy Consistency Loss ($\mathcal{L}_{EC}$), a physics-based regularizer that enforces the Work-Energy Theorem. This theorem equates the change in kinetic energy ($\Delta E_K$) to the work ($W$) done by the net force. In the rotational skeleton system, this work is done by the net torque $\tau_{net} = \tau - G - F$, where $\tau$ is the predicted actuation and $G$ and $F$ are load torques (gravity, friction). The integral form over a discrete time step $\Delta t = [t-1, t]$ is:
\begin{equation}
E_K(t) - E_K(t-1) = \int_{t-1}^{t} P(s) ds,
\end{equation}
where $P(s) = \tau_{net}(s) \cdot \dot{q}(s)$ is the instantaneous power.

\textbf{Discrete-Time Work-Energy Calculation}:
We first compute the kinetic energy at each time step $t$:
\begin{equation}
E_K(t) = 0.5 \cdot \dot{q}(t)^T M(t) \dot{q}(t),
\end{equation}
where $\dot{q}(t) \in \mathbb{R}^{D \times 1}$ and $M(t) \in \mathbb{R}^{D \times D}$. The change in kinetic energy is $\Delta E_K(t) = E_K(t) - E_K(t-1)$ during the interval $[t-1, t]$.
Next, the net power $P(t)$ is:
\begin{equation}
P(t) = \sum_{i=1}^D [(\tau(t) - G(t) - F(t)) \cdot \dot{q}(t)],
\end{equation}
where $\cdot$ denotes the Hadamard product. The work $W(t)$ over the interval $[t-1, t]$ is approximated as $W(t) = 0.5 \cdot (P(t) + P(t-1))$ using the trapezoidal rule.

\textbf{Relative Energy Residual Loss}:
A naive L1 loss between $\Delta E_K$ and $W$ would be highly sensitive to the magnitude of motion. We therefore design a scale-invariant relative energy residual $r_E$:
\begin{equation}
r_E(t) = \frac{\Delta E_K(t) - W(t)}{|\Delta E_K(t)| + |W(t)| + \delta} \cdot \mathcal{M}(t).
\end{equation}
The denominator normalizes the residual for robustness to varying energy scales, while $\delta$ ensures numerical stability. 
The mask $\mathcal{M}(t)$ directly zeros out the residual $r_E(t)$ where the denominator is near-zero, preventing noise amplification during static or micro-movement phases. 
The final loss $\mathcal{L}_{EC}$ is the Huber loss computed between the residual $r_E(t)$ and $\mathbf{0}$ and averaged over all time steps, which provides robustness to potential outliers in the dynamic estimates.

\subsection{Spatio-Temporal Modulation}
\label{sec:3.4}
The dynamics offer causal insights crucial for resolving kinematic ambiguities and localizing boundaries. We leverage this via the Spatio-Temporal Modulation (STM), which injects these insights into the spatial and temporal streams.

\subsubsection{Spatial Modulation}

Spatial modulation enriches the spatial representation by fusing force information, enhancing inter-class discriminability. Given the kinematic features $F_{kin} \in \mathbb{R}^{C \times T \times V}$ from the GCN and the generalized forces $\tau \in \mathbb{R}^{D \times T}$, we first align their dimensions. $\tau$ is projected into the kinematic channel dimension $C$ using a linear layer, yielding the dynamic feature $F'_{dyn} \in \mathbb{R}^{C \times T}$. This feature is then expanded $V$ times along the joint dimension to create $F_{dyn} \in \mathbb{R}^{C \times T \times V}$. Finally, $F_{dyn}$ and $F_{kin}$ are concatenated along the channel dimension, producing the final spatial representation $F_{sp} \in \mathbb{R}^{2C \times T \times V}$. This integrated feature, encoding both geometric pose and its causal dynamics, serves as the input to the temporal model.

\subsubsection{Temporal Modulation}
Temporal modulation employs salient dynamic signals to hierarchically gate the temporal feature flow, enhancing sensitivity to high-energy events and abrupt dynamic shifts characteristic of action boundaries.

\textbf{Salient Dynamic Signal Construction}: 
We first distill three salient 1D signals from the dynamic state variables to quantify the motion's instantaneous energy and transients: (1) Instantaneous Power ($g_P$): The L1 norm of the element-wise power ($\tau \cdot \dot{q}$), capturing total energy expenditure: $g_P(t) = \|\tau(t) \cdot \dot{q}(t)\|_1$. (2) Torque Norm ($g_\tau$): The L2 norm of the generalized forces, indicating the overall actuation magnitude: $g_\tau(t) = \|\tau(t)\|_2$. (3) Torque Change ($g_{\dot{\tau}}$): The L2 norm of the first-order temporal difference of $\tau$ (via finite difference), quantifying dynamic transients that often signal boundary shifts: $g_{\dot{\tau}}(t) = \|\tau(t) - \tau(t-1)\|_2$. These three signals ($g_P, g_\tau, g_{\dot{\tau}} \in \mathbb{R}^{T}$) are concatenated, forming the initial dynamic gate $G^{(0)}_{dyn} \in \mathbb{R}^{3 \times T}$ for the first stage.

\textbf{Multi-Stage Gating}:
Our temporal model features a parallel gating branch of $L$ stages, mirroring the main temporal backbone to enable hierarchical modulation. At each stage $l \in \{1, \dots, L\}$, the module takes two inputs: the main temporal features $H^{(l)}_T \in \mathbb{R}^{C \times T}$ from stage $l$ and the gate signal $G^{(l-1)}_{dyn} \in \mathbb{R}^{3 \times T}$ from the previous gating stage. The process at stage $l$ unfolds in two steps:

First, the incoming gate signal $G^{(l-1)}_{dyn}$ is refined. A key design choice is to process the three dynamic signals in parallel, independent streams (Fig.~\ref{fig:2}) to preserve their distinct physical interpretations (Power, Torque, Torque Change) without channel cross-talk. Specifically, each channel $g_k^{(l-1)}$ (where $k \in \{P, \tau, \dot{\tau}\}$) is processed by a separate 1D temporal convolution block $\mathcal{G}_k^{(l)}$, followed by a sigmoid activation $\sigma$ to produce normalized gate weights:
\begin{equation}
g_k^{(l)} = \sigma(\mathcal{G}_k^{(l)}(g_k^{(l-1)})), \quad \text{for } k \in \{P, \tau, \dot{\tau}\}.
\end{equation}
The resulting single-channel gates are concatenated to form the updated gate $G^{(l)}_{dyn} = \text{Concat}(g_P^{(l)}, g_\tau^{(l)}, g_{\dot{\tau}}^{(l)}) \in \mathbb{R}^{3 \times T}$, which is propagated to stage $l+1$.

Second, the temporal features $H^{(l)}_T$ are dynamically modulated by the new gate $G^{(l)}_{dyn}$. Each of the three refined gate signals $g_k^{(l)}$ individually modulates the feature map $H^{(l)}_T$ via broadcasted Hadamard product ($\odot$):
\begin{equation}
H^{(l)}_{k} = H^{(l)}_T \odot g_k^{(l)}, \quad \text{for } k \in \{P, \tau, \dot{\tau}\}.
\end{equation}
These three modulated features are concatenated along the channel dimension and fused via a 1x1 convolution ($\mathcal{F}_{fuse}$). A residual connection with the original $H^{(l)}_T$ is added for stable training and information preservation:
\begin{equation}
\tilde{H}^{(l)} = \mathcal{F}_{fuse}(\text{Concat}(H^{(l)}_{P}, H^{(l)}_{\tau}, H^{(l)}_{\dot{\tau}})) + H^{(l)}_T.
\end{equation}
The gated feature $\tilde{H}^{(l)}$ is passed to the next stage, forming a hierarchical mechanism that adaptively modulates the feature flow at each temporal scale via dynamic cues, progressively highlighting boundary-sensitive features.

\begin{table*}[tb]
    \footnotesize
    \centering
    \caption{Comparison with the latest results on PKU-MMD v2 (X-view and X-sub), and LARa.  \textbf{Bold} and \underline{underline} indicate the best and second-best results in each column. FLOPs denote the computational complexity when the input is $X \in \mathbb{R}^{12 \times 6000 \times 19 }$ on LARa dataset. }
    \setlength{\tabcolsep}{4pt}
    \renewcommand{\arraystretch}{0.85}
    \begin{tabular}{l|cc|ccccc|ccccc|ccccc}
        \toprule
        Dataset & & & \multicolumn{5}{c|}{\textbf{PKU-MMD v2 (X-view)}} & \multicolumn{5}{c|}{\textbf{PKU-MMD v2 (X-sub)}} & \multicolumn{5}{c}{\textbf{LARa}} \\
        \cmidrule(lr){4-8} \cmidrule(lr){9-13} \cmidrule(lr){14-18}
        Metric & FLOPs$\downarrow$ & Param.$\downarrow$ & Acc$\uparrow$ & Edit$\uparrow$ & \multicolumn{3}{c|}{F1@\{10,25,50\}$\uparrow$} & Acc$\uparrow$ & Edit$\uparrow$ & \multicolumn{3}{c|}{F1@\{10,25,50\}$\uparrow$} & Acc$\uparrow$ & Edit$\uparrow$ & \multicolumn{3}{c}{F1@\{10,25,50\}$\uparrow$}\\
        \midrule
        MS-TCN~\cite{MS-TCN} & 3.13G & 0.52M & 58.2 & 56.6 & 58.6 & 53.6 & 39.4 & 65.5 & - & - & - & 46.3 & 65.8 & - & - & - & 39.6  \\
        ASFormer~\cite{Asformer} & 6.20G & 1.01M & 64.8 & 62.6 & 65.1 & 60.0 & 45.8 & 68.3 & 68.1 & 71.9 & 68.2 & 54.5 & 72.2 & 62.2 & 66.1 & 61.9 & 49.2  \\
        ASRF~\cite{ASRF} & 7.00G & 1.18M & 60.4 & 59.3 & 62.5 & 58.0 & 46.1 & 67.7 & 67.1 & 72.1 & 68.3 & 56.8 & 71.9 & 63.0 & 68.3 & 65.3 & 53.2  \\
        MS-GCN~\cite{MS-GCN} & 30.97G & 0.63M & 65.3 & 58.1 & 61.3 & 56.7 & 44.1 & 68.5 & - & - & - & 51.6 & 65.6 & - & - & - & 43.6  \\
        MTST-GCN~\cite{MTST-GCN} & 108.87G & 2.58M & 66.5 & 64.0 & 67.1 & 62.4 & 49.9 & 70.0 & 65.8 & 68.5 & 63.9 & 50.1 & 73.7 & 58.6 & 63.8 & 59.4 & 47.6  \\
        DeST-{\scriptsize TCN}~\cite{DeST} & 5.72G & 0.77M & 62.4 & 58.2 & 63.2 & 59.2 & 47.6 & 67.6 & 66.3 & 71.7 & 68.0 & 55.5 & 72.6 & 63.7 & 69.7 & 66.7 & 55.8  \\
        DeST-{\scriptsize Former}~\cite{DeST} & 7.71G & 1.10M & 67.3 & 64.7 & 69.3 & 65.6 & 52.0 & 70.3 & 69.3 & 74.5 & 71.0 & 58.7 & 75.1 & 64.2 & 70.3 & 68.0 & 57.7  \\
        MSTAS~\cite{MSTAS} & 34.60G & 1.02M & 71.8 & 63.0 & 68.0 & 63.4 & 52.8 & 67.6 & 67.0 & 72.1 &  69.0 & 57.1 & 74.8 & 60.8 & 66.2 & 62.2 & 50.4 \\
        SCSAS~\cite{SCSAS} & 32.70G & 1.80M & 67.9 & 64.0 & 70.4 & 66.8 & 55.5 & 70.3 & 68.2 & 72.2 & 68.2 & 56.2 & 74.2 & 61.7 & 67.3 & 65.2 & 53.5 \\
        ME-ST~\cite{ME-ST} & 97.07G & 3.16M & \underline{74.1} & \underline{70.5} & \underline{76.6} & \underline{73.2} & \underline{62.4} & 68.5 & 67.2 & 72.3 & 68.8 & 58.1 & 74.2 & 65.0 & 71.0 & 68.2 & 57.1  \\
        LaSA~\cite{LaSA} & 10.44G & 1.58M & 69.5 & 67.8 & 72.9 & 69.2 & 57.0 & 73.5 & \underline{73.4} & 78.3 & 74.8 & 63.6 & 75.3 & \underline{65.7} & 71.6 & 69.0 & 57.9  \\
        LPL~\cite{LPL}& 20.30G & 3.22M & 70.0 & 67.7 & 73.3 & 70.0 & 58.5 & \underline{74.7} & 72.7 & \underline{78.6} & \underline{75.8} & \underline{64.3} & \textbf{76.1} & 65.2 & \underline{72.3} & \underline{70.0} & \underline{58.6}  \\
        
        \cmidrule(lr){1-18}
        \textbf{LaDy} & 13.67G & 1.83M & \textbf{77.0} & \textbf{74.7} & \textbf{80.5} & \textbf{77.3} & \textbf{67.6} & \textbf{76.2} & \textbf{75.1} & \textbf{80.1} & \textbf{77.3} & \textbf{67.0} &  \underline{75.6} &  \textbf{65.9} &  \textbf{72.9} &  \textbf{70.4} &  \textbf{59.7} \\
        \bottomrule
    \end{tabular}
    \label{tab:SOTA1}
\end{table*}

\begin{table*}[th]
    \footnotesize
    \centering
    \caption{Comparison with the latest results on MCFS-22, MCFS-130, and TCG-15 datasets. }
    \setlength{\tabcolsep}{4pt}
    \renewcommand{\arraystretch}{0.85}
    \begin{tabular}{l|ccccc|ccccc|ccccc}
        \toprule
        Dataset & \multicolumn{5}{c|}{\textbf{MCFS-22}} & \multicolumn{5}{c|}{\textbf{MCFS-130}} & \multicolumn{5}{c}{\textbf{TCG-15}} \\
        \cmidrule(lr){2-6} \cmidrule(lr){7-11} \cmidrule(lr){12-16}
        Metric & Acc$\uparrow$ & Edit$\uparrow$ & \multicolumn{3}{c|}{F1@\{10, 25, 50\}$\uparrow$} & Acc$\uparrow$ & Edit$\uparrow$ & \multicolumn{3}{c|}{F1@\{10, 25, 50\}$\uparrow$} & Acc$\uparrow$ & Edit$\uparrow$ & \multicolumn{3}{c}{F1@\{10, 25, 50\}$\uparrow$} \\
        \midrule
        MS-TCN~\cite{MS-TCN} & 75.6 & 74.2 & 74.3 & 69.7 & 59.5 & 65.7& 54.5& 56.4 & 52.2 & 42.5 & 85.5 & 67.7 & 70.0 & 67.3 & 59.1 \\
        ASFormer~\cite{Asformer} & 78.7 & 82.3 & 82.8 & 77.9 & 66.9 & 67.5& 69.1 & 68.3 & 64.0 & 55.1 &82.3& 69.1 & 71.3 & 67.7 & 57.6 \\
        ASRF~\cite{ASRF} & 75.5 & 77.3 & 83.3 & 80.1 & 69.2 & 65.6& 65.6 & 66.7 & 62.3 & 51.9 & 83.9 & 65.3 & 71.5 & 69.9 & 62.4 \\
        
        MS-GCN~\cite{MS-GCN}& 75.5 & 72.6 & 75.7 & 70.5 & 57.9 & 64.9 & 52.6 & 52.4 & 48.8 & 39.1 & 86.6& 68.7& 70.9&69.7&59.4 \\
        IDT-GCN~\cite{IDT-GCN} & 79.9 & 84.5 & 88.0 & 84.9 & 74.9  & 68.6 & 70.2 & 70.7 & 67.3 & 58.6 & - & - & - & - & -\\
        DeST-{\scriptsize TCN}~\cite{DeST} & 78.7 & 82.3 & 86.6 & 83.5 & 73.2 & 70.5 & 73.8 & 74.0 & 70.7 & 61.8 & 87.1 & 73.7& 78.7 & 76.9&69.6\\
        DeST-{\scriptsize Former}~\cite{DeST} & 80.4 & 85.2 & 87.4 & 84.5 & 75.0 & 71.4 & 75.8 & 75.8 & 72.2 & 63.0 &88.1&	75.7&	80.0&	78.2&	71.2 \\
        MSTAS~\cite{MSTAS} & - & - & - & - & - & - & - & - &  - & - & 88.3 & 69.5 & 73.5 & 70.4 & 63.9 \\
        SCSAS~\cite{SCSAS} & - & - & - & - & - & 67.4 & 66.9 & 65.8 & 62.2 & 53.1 & 87.4 & 69.6 & 72.8 & 71.6 & 67.0 \\
        ME-ST~\cite{ME-ST} & - & - & - & - & - & - & - & - & - & - & 88.7 & 77.3 & 81.8 & 78.6 & 72.4  \\
        LaSA~\cite{LaSA} & \underline{80.8} & \underline{86.7} & \underline{89.3} & \underline{86.2} & \underline{76.3} & \underline{72.6} & \underline{79.3} & \underline{79.3} & \underline{75.8} & \underline{66.6} & 88.7 & 77.4& 81.2&79.5&73.3 \\
        LPL~\cite{LPL} & - &- &- &- &- & - & - & - & - & -& \underline{88.8}& \underline{77.4}& \underline{81.9} & \underline{80.0} & \underline{73.8} \\

        \cmidrule(lr){1-16}
        \textbf{LaDy} & \textbf{81.5} & \textbf{89.8} & \textbf{90.8} & \textbf{86.7} & \textbf{76.7} & \textbf{72.8} & \textbf{79.7} & \textbf{80.0}& \textbf{76.8} & \textbf{66.9} & \textbf{89.3} & \textbf{77.9} & \textbf{82.1} & \textbf{80.1} & \textbf{74.3} \\
        \bottomrule
    \end{tabular}
    \label{tab:SOTA2}
\end{table*}

\subsection{Overall Training Objective}
The LaDy is trained end-to-end by minimizing a composite loss function $\mathcal{L}_{total}$, defined as:
\begin{equation}
\mathcal{L}_{total} = \mathcal{L}_{as} + \lambda_{1} \mathcal{L}_{br} + \lambda_{2} \mathcal{L}_{atc} + \lambda_{3} \mathcal{L}_{EC},
\end{equation}
where $\mathcal{L}_{as}$ is a standard action segmentation loss~\cite{ASRF,DeST} applied to all stage outputs $Y_c$. $\mathcal{L}_{br}$ is a binary cross-entropy loss~\cite{ASRF} supervising $Y_b$. $\mathcal{L}_{atc}$ denotes the action-text contrastive loss~\cite{LaSA,TRG-Net}. Finally, $\mathcal{L}_{EC}$ is our ECLoss. The values of $\lambda_{1}$, $\lambda_{2}$, and $\lambda_{3}$ are set to 1.0, 0.8, and 0.1.

\section{Experiment}
\label{sec:experiment}

\begin{figure*}[t]
  \centering
  \begin{subfigure}{0.48\textwidth} 
    \centering
    \includegraphics[width=\linewidth]{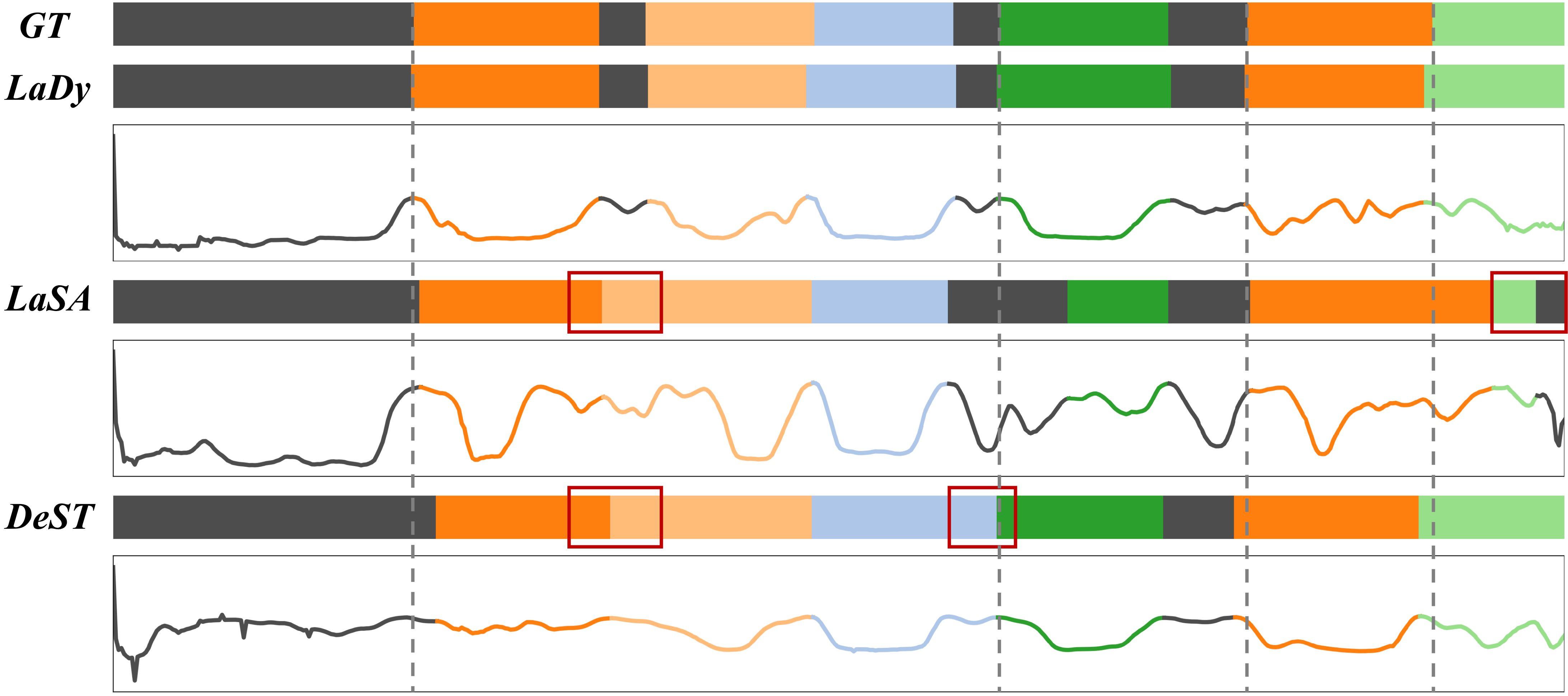} 
    \caption{PKU-MMD v2 (X-view)}
    \label{fig:4-1}
  \end{subfigure}
  \begin{subfigure}{0.48\textwidth}
    \centering
    \includegraphics[width=\linewidth]{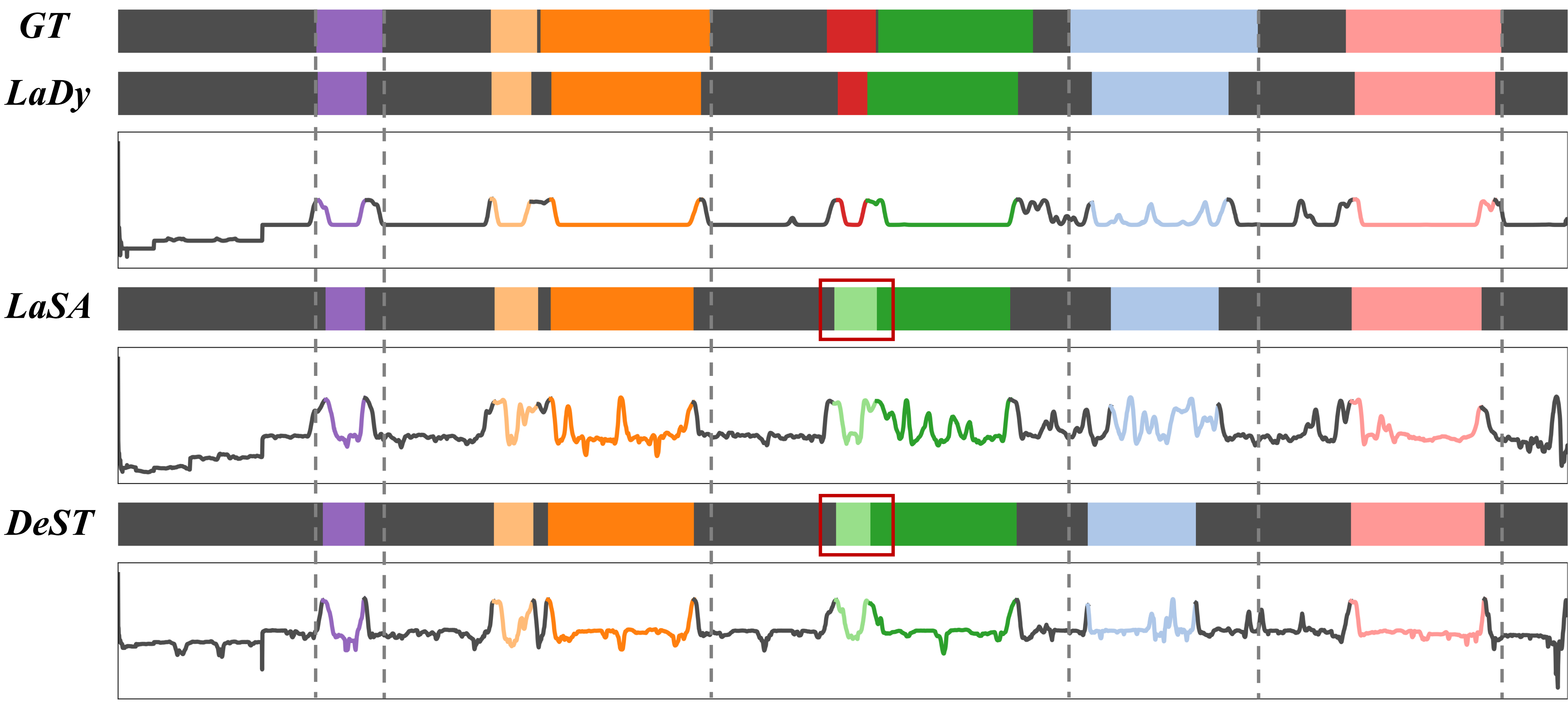} 
    \caption{MCFS-130}
    \label{fig:4-2}
  \end{subfigure}

  \caption{Qualitative results on PKU-MMD v2 and MCFS-130. The top row is the Ground Truth, followed by the segmentation results (bars) and boundary confidence scores (curves) for LaDy, LaSA, and DeST. Different colors denote distinct action classes. Red boxes and gray dashed vertical lines highlight misclassifications and larger boundary deviations observed in other methods compared to LaDy.}
  \label{fig:4}
\end{figure*}

\begin{figure}[t]
  \centering
  \begin{subfigure}[b]{0.48\textwidth}
    \includegraphics[width=\textwidth]{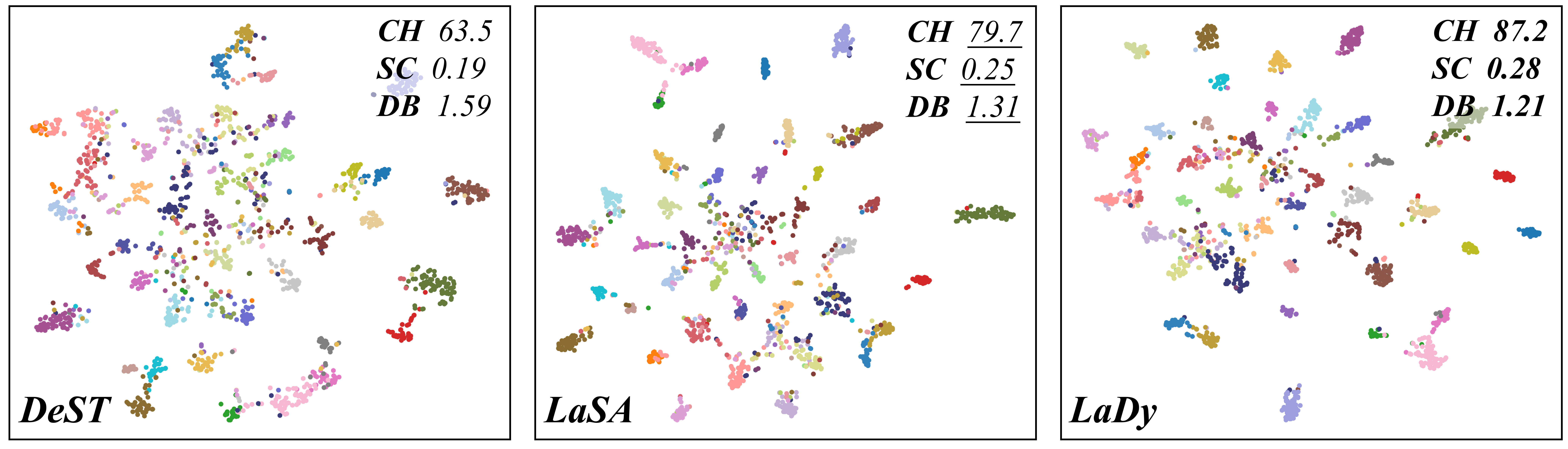}
    \caption{The Representation Space on PKU-MMD v2 (X-sub)}
    \label{fig:5-1}
  \end{subfigure}
  \hfill
  \begin{subfigure}[b]{0.48\textwidth}
    \includegraphics[width=\textwidth]{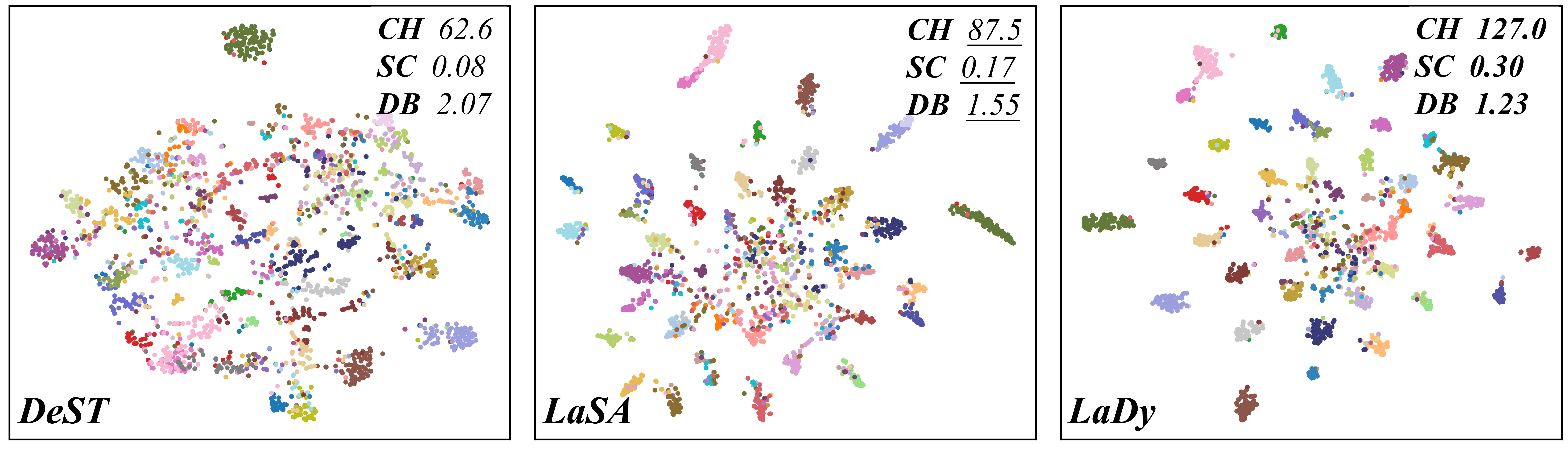}
    \caption{The Representation Space on PKU-MMD v2 (X-view)}
    \label{fig:5-2}
  \end{subfigure}
  
  \caption{Visualization of the representation space on PKU-MMD v2. Each point represents an action segment feature, colored according to its ground-truth class. Quantitative clustering metrics are reported in the top-right corner: Silhouette Coefficient (SC $\uparrow$), Calinski-Harabasz Index (CH $\uparrow$), and Davies-Bouldin Index (DB $\downarrow$). Higher ($\uparrow$) is better for SC/CH; lower ($\downarrow$) is better for DB.}
  \label{fig:5}
\end{figure}

\subsection{Setup}
\label{sec:setup}

\noindent\textbf{Datasets.} We evaluate LaDy on six public benchmarks. 
PKU-MMD v2~\cite{PKU-MMD} is a 52-class daily activity dataset; we report on its cross-subject (X-sub) and cross-view (X-view) splits. 
MCFS-22 \& MCFS-130~\cite{MCFS} are figure skating datasets with 22 and 130 action classes. 
LARa~\cite{LARA} is a warehouse activity dataset with 8 classes. 
TCG-15~\cite{tcg} is a traffic control gesture dataset with 15 classes.

\noindent\textbf{Evaluation Metrics.} Following standard TAS protocols~\cite{MS-TCN}, we use frame-wise accuracy (Acc), segmental edit distance (Edit), and segmental F1 scores at IoU thresholds of 10\%, 25\%, and 50\% (F1@\{10, 25, 50\}).

\noindent\textbf{Implementation Details.} The feature dimension is $C=64$. The linear layers for Lagrangian term estimation use 128 channels. 
We optimize models using Adam with a learning rate of $10^{-3}$.
Consistent with prior works~\cite{DeST}, batch sizes are 8 for PKU-MMD v2/LARa, 2 for TCG-15, and 1 for MCFS-22/130. 
We train for 300 epochs, except for LARa (60 epochs) due to faster convergence. 
All experiments are conducted on a single NVIDIA RTX 3090 GPU.

\subsection{Comparisons with the State-of-the-Arts}
\label{sec:comparison}

\noindent\textbf{Quantitative Comparison.}
We compare LaDy with recent STAS methods on six public datasets, with detailed results in Tab.~\ref{tab:SOTA1} and Tab.~\ref{tab:SOTA2}. LaDy consistently achieves state-of-the-art performance across nearly all metrics. Its superiority is particularly evident on PKU-MMD v2, where it surpasses the prior best-performing method by significant margins: \textbf{+5.2\%} in F1@50, \textbf{+4.2\%} in Edit, and \textbf{+2.9\%} in Acc on the X-view setting; and \textbf{+2.7\%} in F1@50, \textbf{+1.7\%} in Edit, and \textbf{+1.5\%} in Acc on the X-sub setting.
LaDy achieves these substantial gains while remaining computationally efficient (13.67G FLOPs, 1.83M params). This blend of high accuracy and efficiency validates our physics-informed Lagrangian dynamics approach for spatio-temporal modulation in action segmentation.

\noindent\textbf{Qualitative Analysis.}
As illustrated in Fig.~\ref{fig:4}, LaDy's segmentation results are visibly more accurate, showing fewer classification errors and smaller boundary deviations than its counterparts.
Furthermore, LaDy's boundary confidence scores (curves) are significantly more stable and less volatile. While other methods exhibit noisy fluctuations and spurious peaks within action segments, LaDy effectively suppresses this internal noise. Consequently, LaDy produces sharper and more distinct confidence peaks that align precisely with the true action boundaries. This underscores our model's effectiveness in enhancing class discrimination and precise temporal localization.

\noindent\textbf{Feature Space Analysis.}
To further assess the learned representations, we visualize the feature space on PKU-MMD v2 using t-SNE (Fig.~\ref{fig:5}). Visually, LaDy's representations exhibit significantly more compact intra-class aggregation and clearer inter-class separation than previous methods, indicating a more structured semantic distribution.
This observation is confirmed by quantitative clustering metrics in Fig.~\ref{fig:5}: Silhouette Coefficient (SC $\uparrow$), Calinski-Harabasz Index (CH $\uparrow$), and Davies-Bouldin Index (DB $\downarrow$), where LaDy achieves the best scores across all three metrics. This demonstrates that LaDy learns a more discriminative feature space, enhancing inter-class separability.

\subsection{Ablation Studies and Analysis}
\label{sec:ablation}

\begin{table}[tb]
    \footnotesize
    \centering
    \setlength{\tabcolsep}{4pt}
    \renewcommand{\arraystretch}{1.0}
    
    \caption{Ablation study of Lagrangian Dynamics Synthesis (LDS), Spatio-Temporal Modulation (STM), and Energy Consistency Loss (ECLoss) on PKU-MMD v2 (X-sub) dataset.}
    \begin{tabular}{l|ccccc}
        \toprule
        \textbf{Components} & \textbf{Acc} & \textbf{Edit} & \multicolumn{3}{c}{\textbf{F1@\{10, 25, 50\}}} \\
        \midrule
        Baseline & 73.6& 73.0& 78.2& 74.6& 64.3
\\
        \midrule
        +LDS+SM & 74.8& 74.0& 78.9& 76.0& 65.7
\\
        +LDS+TM & 74.9& 73.5& 78.1& 75.2& 65.8
\\
        +LDS+STM & 75.9& 74.6& 79.6& 76.9& 66.3
\\
        +LDS+SM+ECLoss & 74.7& 74.3& 79.1& 76.3& 66.0
\\
        +LDS+TM+ECLoss & 75.4& 74.5& 79.6& 76.7& 66.5
\\
        \midrule
        \textbf{+LDS+STM+ECLoss (LaDy)} & \textbf{76.2} &\textbf{75.1} & \textbf{80.1} & \textbf{77.3} & \textbf{67.0} \\
        \bottomrule
    \end{tabular}
    \label{tab:Allmodules}
\end{table}

We validate LaDy's components---Lagrangian Dynamics Synthesis (LDS), Spatio-Temporal Modulation (STM), and Energy Consistency Loss (ECLoss)---via ablations on PKU-MMD v2 (X-sub), summarized in Tab.~\ref{tab:Allmodules}. The strong Baseline is built upon DeST, equipped with CTR-GCN's adaptive GCN~\cite{CTR-GCN,TRG-Net} and LaSA's $\mathcal{L}_{atc}$~\cite{LaSA} (Fig.~\ref{fig:2}, upper).
First, integrating the LDS module with either Spatial Modulation (+LDS+SM) or Temporal Modulation (+LDS+TM) markedly outperforms the dynamics-agnostic Baseline. This confirms our hypothesis that the dynamics prior is highly effective, directly enhancing spatial feature discriminability and temporal boundary perception.
Next, combining both modulations (+LDS+STM) achieves a substantial, synergistic performance jump to 75.9\% Acc and 66.3\% F1@50. This demonstrates that spatial and temporal modulation are complementary, effectively addressing the distinct challenges of inter-class ambiguity (via SM) and precise boundary localization (via TM).
Finally, we applied the ECLoss to enforce physical plausibility. This physics-based regularization provides a final, consistent performance boost. This is evident as our full model (+LDS+STM+ECLoss) improves upon the version without the loss (+LDS+STM) across all metrics, achieving the peak performance of 76.2\% Acc and 67.0\% F1@50. This validates that enforcing the work-energy theorem refines the dynamic representations for more precise segmentation.

\begin{figure}[tb]
    \centering
    \includegraphics[width=0.95\linewidth]{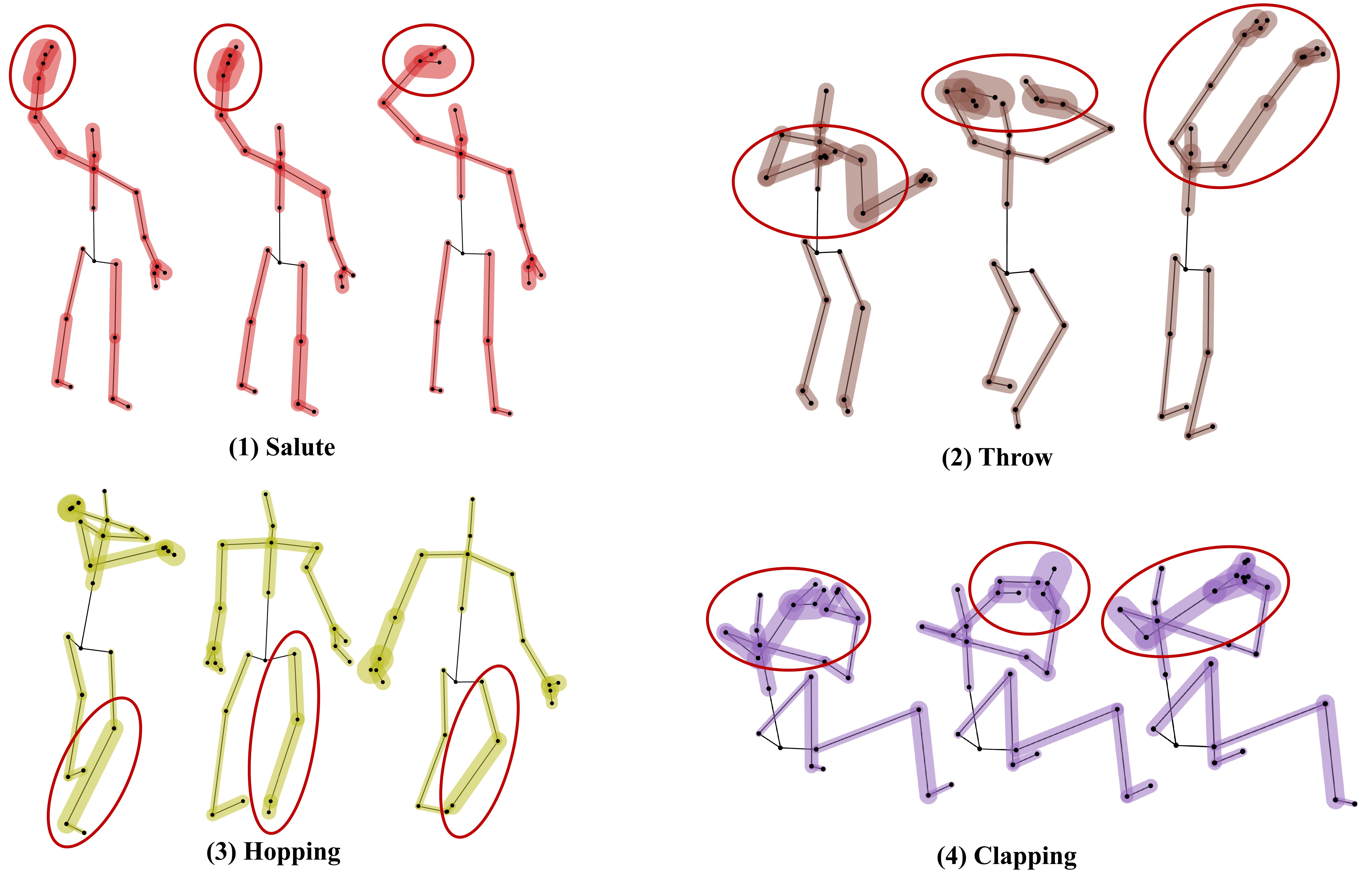}
    
    \caption{Visualization of synthesized generalized force magnitudes for four actions. Limb thickness is proportional to the norm of the estimated torque. Red circles highlight the primary active limbs, illustrating that the learned dynamics are both class-discriminative and physically plausible.}
    \label{fig:6}
\end{figure}

\textbf{Analysis of Synthesized Dynamics.} 
To intuitively validate our model's adherence to physical principles, we visualize the synthesized generalized forces, as shown in Fig.~\ref{fig:6}. We compute the norm of the 3-axis force for each joint and render it as the thickness of the corresponding skeletal limb. The visualization clearly shows that LaDy learns highly discriminative and physically plausible dynamic profiles. For instance, (1) \textit{Salute} shows high force concentration in the saluting arm; (2) \textit{Throw} activates the entire upper body and arms; (3) \textit{Hopping} correctly identifies the legs as the primary force source; and (4) \textit{Clapping} focuses on the hands and forearms. This interpretable result strongly validates that our physics-constrained LDS module, guided by the ECLoss, successfully synthesizes meaningful and physically coherent dynamics, rather than arbitrary features.

\begin{figure}[tb]
    \centering
    \includegraphics[width=\linewidth]{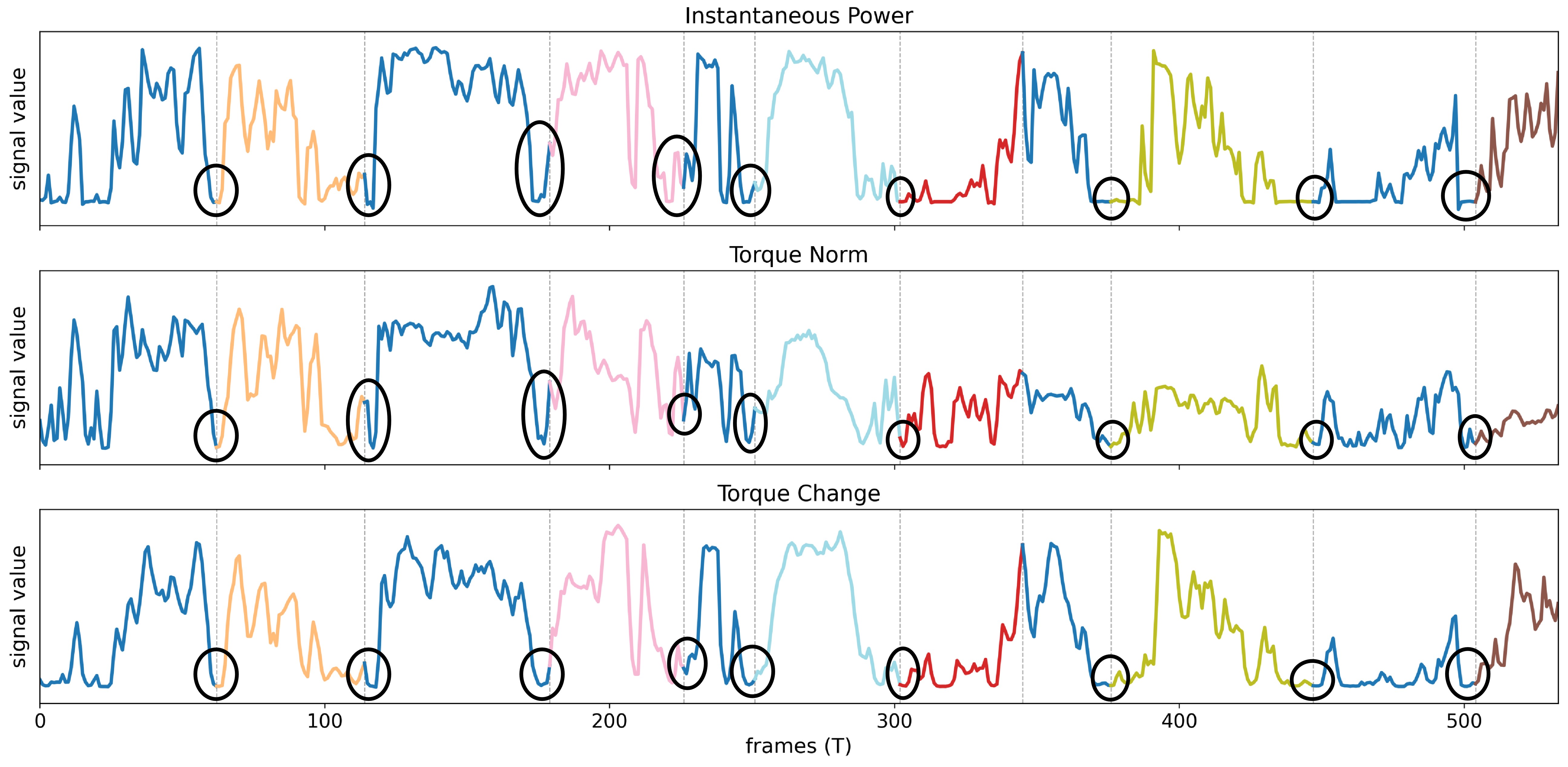}
    
    \caption{Visualization of layer-averaged dynamic gating signals. The time-series plots for Instantaneous Power, Torque Norm, and Torque Change show signals averaged across all $L$ temporal layers. Note the distinct, sharp troughs (highlighted by ovals) that precisely align with the ground-truth action boundaries (dashed lines), validating their role in boundary localization.}
    \label{fig:7}
\end{figure}

\textbf{Analysis of Temporal Gating Signals.} To further validate the efficacy of our Temporal Modulation, we visualize the learned gating signals. We average the $L$-stage gating signals and plot them against the ground-truth action segments in Fig.~\ref{fig:7}. The visualization reveals a key finding: all three signals exhibit sharp, distinct troughs (valleys) that precisely coincide with the ground-truth action boundaries (marked by gray dashed lines).
This alignment is particularly pronounced in the ``Torque Change'' signal. This phenomenon is highly intuitive: action transitions are fundamentally defined by abrupt shifts in the body's dynamic force profile, and LaDy successfully learns to identify these physical, dynamic events as explicit cues for semantic action boundaries. This mechanism provides the temporal network with superior boundary awareness, explaining the strong performance of Temporal Modulation.

\section{Conclusion}
\label{sec:conclusion}


In this paper, we addressed the dynamics-agnostic nature of existing Skeleton-based Temporal Action Segmentation (STAS) methods by proposing the Lagrangian-Dynamic Informed Network (LaDy). Our Lagrangian Dynamics Synthesis (LDS) module synthesizes generalized forces, which are made physically coherent by a novel Energy Consistency Loss (ECLoss) that enforces the work-energy theorem. The synthesized dynamic features are then leveraged via Spatio-Temporal Modulation (STM) to enhance action discriminability through direct spatial fusion and promote precise boundary localization via hierarchical temporal gating. Extensive experiments on six benchmarks demonstrated that LaDy achieves state-of-the-art performance. Ablation studies and visualizations further confirmed the efficacy and physical interpretability of our approach, demonstrating that integrating physical dynamics provides a more discriminative and precise foundation for segmentation.

\section*{Acknowledgement}
This work was supported in part by the National Key Research and Development Program of China under Grant 2025YFE0218500; in part by the National Natural Science Foundation of China under Grant 62503139, Grant 62573163, and Grant 62521006; in part by Shenzhen Science and Technology Program under Grant JCYJ20240813105137049; and in part by the GuangDong Basic and Applied Basic Research Foundation under Grant 2024A1515012028, and Grant 2026A1515011822; and in part by the Shenzhen Medical Research Fund under Grant A2502034.

{
    \small
    \bibliographystyle{ieeenat_fullname}
    \bibliography{main}
}

\clearpage
\setcounter{page}{1}
\maketitlesupplementary

\newtheorem{theorem}{Theorem}
\newtheorem{proposition}{Proposition}
\newtheorem{definition}{Definition}
\newtheorem{lemma}{Lemma}

\newcommand{\Lg}{\mathcal{L}} 
\newcommand{\q}{\bm{q}}       
\newcommand{\dq}{\dot{\bm{q}}}  
\newcommand{\ddq}{\ddot{\bm{q}}} 
\newcommand{\M}{\bm{M}}       
\newcommand{\C}{\bm{C}}       
\newcommand{\G}{\bm{G}}       
\newcommand{\F}{\bm{F}}       
\newcommand{\T}{\bm{\tau}}     
\newcommand{\bmtau}{\bm{\tau}}     

\newcommand{\vect}[1]{\bm{#1}}
\newcommand{\mat}[1]{\mathbf{#1}}
\newcommand{\Ek}{E_K} 
\newcommand{\Ep}{E_P} 

This supplementary material provides a comprehensive exposition of the theoretical underpinnings, implementation specifics, and extended empirical validation of the proposed Lagrangian-Dynamic Informed Network (LaDy). We commence by establishing the rigorous physical framework in \textbf{Sec.~\ref{sec:theory}}, detailing the Euler-Lagrange derivation and the axiomatic basis of the Energy Consistency Loss. \textbf{Sec.~\ref{sec:supp:method}} elaborates on the architectural specifications of the backbone modules and loss formulations, followed by \textbf{Sec.~\ref{sec:supp:setup}}, which provides an exhaustive account of experimental protocols to ensure reproducibility. To further substantiate our claims, \textbf{Sec.~\ref{sec:supp:CE}} presents granular comparative analyses, encompassing performance-efficiency trade-offs, boundary precision, per-class breakdowns, extended qualitative evaluations, and robustness to input perturbations. \textbf{Sec.~\ref{sec:supp:DAS}} then offers extensive ablation studies scrutinizing the internal mechanisms of the dynamics synthesis, energy-based loss, and modulation modules. Finally, \textbf{Sec.~\ref{sec:discussion}} concludes with a broader discussion exploring the framework's generalization capabilities, alongside a critical assessment of failure cases, intrinsic limitations, and potential avenues for future research.

\section{Theoretical Foundations of Dynamic Model}
\label{sec:theory}

In this section, we provide the rigorous physical framework underpinning our model (LaDy). Moving beyond black-box approximations, we first derive the governing dynamic equations via the Euler-Lagrange formulation (Sec.~\ref{sec:e-l}), establishing a structurally explicit basis for force synthesis. We then articulate the Work-Energy Theorem (Sec.~\ref{sec:work-energy}), which serves as the axiomatic foundation for our energy consistency supervision. Subsequently, we formalize the fundamental physical constraints---specifically the positive definiteness of inertia and the skew-symmetry of the Coriolis matrix---that our network architecture is strictly constrained to satisfy (Sec.~\ref{sec:properties}). Finally, we define the kinematic mapping that theoretically bridges the observable Cartesian inputs with the generalized coordinates utilized in our dynamic formulation (Sec.~\ref{sec:kinematics}).

\subsection{Derivation of Lagrange Dynamic Equation}
\label{sec:e-l}

The dynamics of a complex, articulated system (like a human skeleton) can be elegantly described using the Euler-Lagrange formulation. This approach is based on the system's scalar energy functions rather than vector-based Newtonian forces.

\begin{definition} [The Lagrangian]
The Lagrangian $\Lg$ of a mechanical system is defined as the difference between its total kinetic energy ($E_K$) and its total potential energy ($E_P$).
\begin{equation}
\Lg(\q, \dq) = E_K(\q, \dq) - E_P(\q),
\end{equation}
where $\q \in \mathbb{R}^D$ are the generalized coordinates, and $\dq \in \mathbb{R}^D$ are the generalized velocities.
\end{definition}

For a rigid-body system with $D$ degrees of freedom, the kinetic and potential energies are given by:
\begin{align}
    E_K(\q, \dq) &= \frac{1}{2}\dq^T \M(\q) \dq, \label{eq:kinetic_energy} \\
    E_P(\q) &= E_g(\q),
\end{align}
where $E_g$ is defined as the gravitational potential energy, and $\M(\q) \in \mathbb{R}^{D \times D}$ is the symmetric, positive-definite mass-inertia matrix.

The \textbf{Euler-Lagrange Equation} provides the equations of motion by stating that the path taken by the system minimizes the ``action'' (the integral of the Lagrangian over time). For a system subject to non-conservative generalized forces $\bmtau_{nc}$ (which include actuation torques $\bmtau$ and friction/external forces $\F$), the equation is:
\begin{equation}
\frac{d}{dt}\left(\frac{\partial \Lg}{\partial \dq}\right) - \frac{\partial \Lg}{\partial \q} = \bmtau_{nc} = \bmtau - \F. \label{eq:e-l_base}
\end{equation}
We derive the equation of motion by computing each term in Eq.~\eqref{eq:e-l_base}:

\textbf{Momentum Term:} The derivative of $\Lg$ with respect to velocity $\dq$:
\begin{equation}
\frac{\partial \Lg}{\partial \dq} = \frac{\partial}{\partial \dq} \left( \frac{1}{2}\dq^T \M(\q) \dq - E_P(\q) \right) = \M(\q)\dq.
\end{equation}

\textbf{Time Derivative of Momentum:} Taking the total time derivative:
\begin{equation}
\frac{d}{dt}\left(\frac{\partial \Lg}{\partial \dq}\right) = \frac{d}{dt}(\M(\q)\dq) = \dot{\M}(\q)\dq + \M(\q)\ddq.
\end{equation}

\textbf{Lagrangian Gradient Term:} 
The derivative of $\Lg$ with respect to position $\q$:
\begin{equation}
\frac{\partial \Lg}{\partial \q} = \frac{\partial E_K}{\partial \q} - \frac{\partial E_P}{\partial \q} = \frac{\partial}{\partial \q}\left(\frac{1}{2}\dq^T \M(\q) \dq\right) - \G(\q),
\end{equation}
where we define the gravitational vector $\G(\q) = \frac{\partial E_P(\q)}{\partial \q}$. Substituting these back into Eq.~\eqref{eq:e-l_base}:
\begin{equation}
\left( \dot{\M}(\q)\dq + \M(\q)\ddq \right) - \left( \frac{\partial E_K}{\partial \q} - \G(\q) \right) = \bmtau - \F.
\end{equation}
Rearranging the terms yields:
\begin{equation}
\M(\q)\ddq + \left( \dot{\M}(\q)\dq - \frac{\partial E_K}{\partial \q} \right) + \G(\q) = \bmtau - \F.
\end{equation}
It is a standard result in robotics to define the \textbf{Coriolis and Centrifugal Matrix} $\C(\q, \dq)$ such that:
\begin{equation}
\C(\q, \dq)\dq = \dot{\M}(\q)\dq - \frac{\partial E_K}{\partial \q}.
\end{equation}
This allows us to write the dynamics in the canonical form (here $\F$ represents all non-conservative forces like friction and external forces, which our model learns as $F(q, \dot{q})$):
\begin{equation}
\M(\q)\ddq + \C(\q, \dq)\dq + \G(\q) + \F = \bmtau.
\label{Eq:MCFG}
\end{equation}
This derivation formally establishes the origins of the dynamic equation $\bmtau = \M\ddq + \C\dq + \G + \F$ that our LaDy model synthesizes, where $\bmtau$ represents the generalized actuation forces.

\subsection{The Work-Energy Theorem and Power Balance}
\label{sec:work-energy}

Following the derivation of the dynamic equation (Eq.~\eqref{Eq:MCFG}), we introduce the fundamental physical principle the system must obey: the Work-Energy Theorem.

\begin{theorem} [The Work-Energy Theorem]
The change in a system's kinetic energy, $\Delta E_K$, over a time interval $[t_1, t_2]$ is equal to the total work, $W$, done on the system by the net generalized forces $\bmtau_{net}$ during that interval.
\begin{equation}
\Delta E_K = E_K(t_2) - E_K(t_1) = \int_{t_1}^{t_2} P_{net}(t) dt = W,
\end{equation}
where $P_{net}(t) = \dq(t)^T \bmtau_{net}(t)$ is the instantaneous power.
\end{theorem}

In our system, this net generalized force $\bmtau_{net}$ comprises all forces that alter the kinetic state: $\bmtau_{net} = \bmtau - \G - \F$. Per the dynamics equation (Eq.~\eqref{Eq:MCFG}), this is equivalent to:
\begin{equation}
\bmtau_{net} = \bmtau - \G - \F = \M(\q)\ddq + \C(\q, \dq)\dq.
\label{eq:tau_net}
\end{equation}
Therefore, the differential form of the Work-Energy Theorem, which represents the system's \textbf{Power Balance}, is the fundamental principle:
\begin{equation}
P_{net}(t) = \dot{E}_K(t).
\label{eq:power_balance}
\end{equation}

\textbf{Implication for LaDy:}
This theorem provides the physical justification for the Energy Consistency Loss $\mathcal{L}_{EC}$ (Sec.~3.3). Our loss function is built upon this axiom. By minimizing the residual between the discrete-time computations of $\Delta E_K$ and $W$, $\mathcal{L}_{EC}$ enforces that all synthesized dynamic components ($\M, \C, \G, \F, \bmtau$) are mutually consistent and collectively obey this fundamental law of physics.

\subsection{Fundamental Properties of the Dynamic Model}
\label{sec:properties}

The estimators in our LDS module (Sec.~3.2.2) are constrained to respect fundamental physical properties. We formalize these properties below.

\subsubsection{Positive Definiteness of Inertia Matrix}

\begin{proposition}
The inertia matrix $\M(\q)$ is symmetric and positive definite (SPD), i.e., $\M(\q) = \M(\q)^T$ and $\bm{x}^T \M(\q) \bm{x} > 0$ for all $\bm{x} \neq \bm{0}$.
\end{proposition}

\begin{proof}
Symmetry $\M(\q) = \M(\q)^T$ arises from the definition of the kinetic energy quadratic form (Eq.~\eqref{eq:kinetic_energy}). Positive definiteness is a direct physical constraint. The kinetic energy $E_K$ of a physical system must be non-negative, and it can only be zero if the system is at rest (zero velocity). 

From Eq.~\eqref{eq:kinetic_energy}, we have $E_K = \frac{1}{2}\dq^T \M(\q) \dq$. By physical principle, $E_K \geq 0$ for any possible velocity $\dq$. Furthermore, $E_K = 0$ if and only if $\dq = \bm{0}$. This is the mathematical definition of the matrix $\M(\q)$ being positive definite. This property is crucial as it guarantees that any non-zero motion $\dq$ corresponds to positive kinetic energy.
\end{proof}

\textbf{Implication for LaDy:}
Our model (Sec.~3.2.2) enforces this physical axiom by parameterizing $\M(\q)$ via its Cholesky decomposition, $\M(\q) = \bm{L}(\q)\bm{L}(\q)^T$. The $\mathcal{F}_M$ estimator predicts the elements of $\bm{L}(\q)$ and ensures its diagonal entries are strictly positive using a softplus function and a small constant $\epsilon$. This construction mathematically guarantees that the synthesized $\M(\q)$ is SPD.

\subsubsection{Passivity and the Skew-Symmetric Property}
A core property of rigid-body dynamics is passivity: the internal forces (inertia and Coriolis/centrifugal) do not generate or dissipate energy. This is captured by a key relationship between $\M$ and $\C$.

\begin{proposition}
For a valid Coriolis matrix $\C(\q, \dq)$ derived from the Lagrangian formulation, the matrix $N = \dot{\M}(\q) - 2\C(\q, \dq)$ is skew-symmetric, i.e., $\bm{N} = -\bm{N}^T$.
\label{prop:2}
\end{proposition}

\begin{proof}
We analyze the power balance of the system. The rate of change of kinetic energy, $\dot{E}_K$, must be equal to the net power, $P_{net}$, delivered by the net generalized forces responsible for that change. As established in Eq.~\eqref{eq:tau_net}, this net generalized force is $\tau_{net} = \M(\q)\ddq + \C(\q, \dq)\dq$. The instantaneous net power $P_{net}$ is:
\begin{equation}
P_{net} = \dq^T \tau_{net} = \dq^T (\M\ddq + \C\dq).
\end{equation}
Separately, we can find the rate of change of kinetic energy, $\dot{E}_K$, by taking the time derivative of Eq.~\eqref{eq:kinetic_energy}:
\begin{equation}
\begin{gathered}
    \dot{E}_K = \frac{d}{dt} \left( \frac{1}{2}\dq^T \M(\q) \dq \right) \\
    = \frac{1}{2} \left( \ddq^T \M \dq + \dq^T \dot{\M} \dq + \dq^T \M \ddq \right).
\end{gathered}
\end{equation}
Since $\M$ is symmetric, $\ddq^T \M \dq = (\M \ddq)^T \dq = \dq^T (\M \ddq)$.
\begin{equation}
\dot{E}_K = \dq^T (\M \ddq) + \frac{1}{2}\dq^T \dot{\M} \dq. \label{eq:dot_ek}
\end{equation}

From the power-balance principle (Eq.~\eqref{eq:power_balance}), the net power $P_{net}$ equals the rate of change of kinetic energy $\dot{E}_K$:
\begin{equation}
\begin{gathered}
P_{net} = \dot{E}_K, \\
\dq^T (\M\ddq + \C\dq) = \dq^T (\M \ddq) + \frac{1}{2}\dq^T \dot{\M} \dq.
\end{gathered}
\end{equation}
Canceling the $\dq^T (\M \ddq)$ terms, we are left with:
\begin{equation}
\dq^T \C \dq = \frac{1}{2}\dq^T \dot{\M} \dq.
\end{equation}
This can be rewritten as:
\begin{equation}
\dq^T \left( \dot{\M} - 2\C \right) \dq = 0.
\end{equation}
For this quadratic form to be zero for \textit{any} velocity vector $\dq$, the matrix $\bm{N} = (\dot{\M} - 2\C)$ must be skew-symmetric ($\bm{N} = -\bm{N}^T$).
\end{proof}

\textbf{Implication for LaDy:}
This property is fundamental. It ensures that the Coriolis forces are non-dissipative. As detailed in Sec.~3.2.2, our model enforces this by parameterizing $\C(\q, \dq)$ directly. We define $\C(\q, \dq) = 0.5(\dot{\M}(\q) - N(\q, \dq))$, where $N(\q, \dq)$ is an explicitly learned skew-symmetric matrix. $\dot{\M}(\q)$ is itself approximated via finite differences. This construction mathematically guarantees the passivity property is satisfied.

\subsubsection{Supplement: Christoffel Construction of C}

While our LaDy model uses a direct parameterization to enforce passivity (Prop.~\ref{prop:2}), it is theoretically insightful to know that $\C$ is not unique. A standard (but not the only) method for constructing a $\C$ matrix that satisfies the passivity property is by using the \textbf{Christoffel symbols of the first kind}, $\Gamma_{ijk}$:
\begin{equation}
\Gamma_{ijk}(\q) = \frac{1}{2} \left( \frac{\partial M_{ij}}{\partial q_k} + \frac{\partial M_{ik}}{\partial q_j} - \frac{\partial M_{jk}}{\partial q_i} \right).
\end{equation}
The elements of the $\C$ matrix can then be defined as:
\begin{equation}
C_{ij}(\q, \dq) = \sum_{k=1}^D \Gamma_{ijk}(\q) \dot{q}_k.
\end{equation}
This specific construction of $\C$ can be shown to satisfy $\C(\q, \dq)\dq = \dot{\M}(\q)\dq - \frac{\partial E_K}{\partial \q}$ and also guarantees the skew-symmetric property of $\bm{N} = \dot{\M} - 2\C$. Our model's direct parameterization of $\bm{N}$ is a more computationally tractable approach for a deep learning context, while arriving at the same fundamental physical constraint.

\subsection{Cartesian to Generalized Coordinates}
\label{sec:kinematics}

While the dynamics (Sec.~\ref{sec:e-l}) are formulated in generalized coordinates $\q$, the input to our system (Sec.~3.1) is in Cartesian (world) coordinates $\bm{p} \in \mathbb{R}^{3V}$. The mapping between these spaces is defined by the system's kinematics.

\begin{definition} [Forward Kinematics]
Forward kinematics (FK) is the function $f$ that maps the $D$ generalized coordinates $\q$ to the Cartesian positions of all $V$ points in the system.
\begin{equation}
\bm{p} = f(\q).
\end{equation}\end{definition}

\textbf{Implication for LaDy:} This theoretical framework provides the justification for our ``Generalized Coordinates Computation'' module (Sec.~3.2.1). That module implements a direct and differentiable geometric computation based on the system's open kinematic chain. It analytically computes the generalized coordinates $q(t)$ (internal joint angles) from the Cartesian positions $\bm{p}(t)$ (external world positions). This computation is, in effect, a precise analytical solution to the Inverse Kinematics (IK) problem ($q = f^{-1}(\bm{p})$) tailored for the defined skeletal structure. Subsequently, by applying finite differences to the $q(t)$ sequence (as detailed in Sec.~3.2.1), we obtain a practical approximation of the generalized velocities $\dq(t)$ and accelerations $\ddq(t)$. This entire process grounds our abstract Lagrangian dynamics in the observable Cartesian input data.


\section{Method Supplement}
\label{sec:supp:method}

This supplement provides a detailed exposition of the baseline architectural components employed in LaDy, which serve as the foundation for our novel dynamics-informed modules. We elaborate on the spatial modeling pipeline (Sec.~\ref{app:method:spatial}), the temporal modeling architecture (Sec.~\ref{app:method:temporal}), the multi-stage prediction refinement branches (Sec.~\ref{app:method:prediction}), and the formulation of the standard loss functions (Sec.~\ref{app:method:loss}).

\subsection{Spatial Modeling Supplement}
\label{app:method:spatial}

In the Spatial Model, the input skeleton sequence $X \in \mathbb{R}^{C_0 \times T \times V}$ is processed via multi-scale and adaptive Graph Convolutional Networks (GCNs) to extract the final kinematic feature $F_{kin} \in \mathbb{R}^{C \times T \times V}$.

\textbf{Multi-Scale GCN.}
We adopt the multi-scale GCN mechanism~\cite{DeST} to capture multi-hop skeletal connectivity. We first define a set of $k$-adjacency matrices $A^k \in \{0,1\}^{V \times V}$, where $A^{k}_{ij} = 1$ if the shortest path distance between joints $i$ and $j$ is $k$, or if $i=j$. A unified multi-scale adjacency matrix $A^{MS} \in \{0,1\}^{V \times KV}$ (for $K$ scales) is constructed by concatenating all normalized $k$-adjacency matrices:
\begin{equation}
\begin{split}
A^{MS} = & [(\tilde{D}^1)^{-\frac{1}{2}} A^{1} (\tilde{D}^1)^{-\frac{1}{2}}] \oplus \cdots \\
 & \oplus [(\tilde{D}^K)^{-\frac{1}{2}} A^{K} (\tilde{D}^K)^{-\frac{1}{2}}],
\end{split}
\end{equation}
where $\oplus$ denotes concatenation and $\tilde{D}^k$ is the diagonal degree matrix for normalization. An initial spatial feature $F_{ms}$ is produced via:
\begin{equation}
F_{ms} = \mathrm{ReLU} \left( \mathrm{reshape} \left[ (A^{MS} + B) \cdot X \right] \cdot W_{ms} \right),
\end{equation}
where $B \in \mathbb{R}^{V \times KV}$ is a learnable matrix encoding non-local correlations, $\mathrm{reshape}(\cdot)$ transforms the feature to $\mathbb{R}^{KC_0 \times T \times V}$, and $W_{ms}$ is a point-wise convolution for channel projection to $C$.

\textbf{Adaptive GCNs.}
To further capture fine-grained, dynamic inter-joint dependencies, we employ adaptive GCNs~\cite{MTST-GCN}, which compute feature-adaptive graph structures. Given $F_{ms}$, two parallel convolutional heads produce intermediate embeddings $P, Q \in \mathbb{R}^{C_1 \times T \times V}$.
Temporal-wise ($G^{T} \in \mathbb{R}^{T \times V \times V}$) and channel-wise ($G^{C} \in \mathbb{R}^{C \times V \times V}$) adaptive graphs are computed by pooling $P, Q$ and measuring pairwise joint dissimilarities:
\begin{equation}
G^{T}_{t,i,j} = P^{T}_{t,i} - Q^{T}_{t,j}, \quad
G^{C}_{c,i,j} = P^{C}_{c,i} - Q^{C}_{c,j}.
\end{equation}
Furthermore, inspired by~\cite{TRG-Net}, we construct a static Text-derived Joint Graph $G^{txt} \in \mathbb{R}^{V \times V}$ using the inverse-normalized L2 distances of BERT~\cite{BERT}-encoded joint descriptions. This semantic prior $G^{txt}$ is broadcasted and added element-wise to both $G^{T}$ and $G^{C}$. 
The final kinematic feature $F_{kin} \in \mathbb{R}^{C \times T \times V}$ is computed by modulating a refined input feature $F_{as} \in \mathbb{R}^{C \times T \times V}$ (derived from $F_{ms}$) with these enriched graphs:
\begin{equation}
F_{kin} = \mathrm{ReLU}\left( \mathrm{BN} \left[ F_{as} G^{T} + F_{as} G^{C} \right] \right) + F_{ms},
\end{equation}
where the products (e.g., $F_{as} G^{T}$ and $F_{as} G^{C}$) denote batched matrix multiplication over the joint axis, enabling the model to capture frame- and channel-specific spatial dependencies.

\subsection{Temporal Modeling Supplement}
\label{app:method:temporal}

As outlined in Sec.~3.1 and Fig.~2, each of the $L$ temporal stages comprises three main components: a Linear Transformer, an Adaptive Fusion module, and our novel Temporal Modulation (detailed in Sec.~3.2).

\textbf{Linear Transformer for Temporal Interaction.}
To efficiently model global temporal context, we adopt the Linear Transformer~\cite{Linformer1}, which reduces attention complexity from $\mathcal{O}(T^2)$ to $\mathcal{O}(T)$. At stage $l$, given the input feature $\tilde{H}^{(l-1)} \in \mathbb{R}^{C \times T}$ from the previous stage's modulation (or $H^{(0)}$ for $l=1$), the Linear Transformer computes the feature $H^{(l)}_{LT}$:
\begin{equation}
\begin{gathered}
Q_t^l = W_{Q}^l \tilde{H}^{(l-1)}, \ K_t^l = W_{K}^l \tilde{H}^{(l-1)}, \ V_t^l = W_{V}^l \tilde{H}^{(l-1)}, \\
H^{(l)}_{LT} = \mathrm{ReLU}[\phi(Q_t^l) \left(\phi(K_t^l)^\top V_t^l\right) \cdot W_t^l + \tilde{H}^{(l-1)}],
\end{gathered}
\end{equation}
where $W_Q$, $W_K$, $W_V$, $W_t$ are linear layers and $\phi(\cdot)$ is the sigmoid activation.

\textbf{Adaptive Feature Fusion.}
To ensure core spatial information is retained at each temporal scale, each stage $l$ also receives input from a Spatial-Channel Fusion head. This head processes the main spatial representation $F_{sp} \in \mathbb{R}^{2C \times T \times V}$ using a point-wise convolution, a reshape operation (merging $V$ and $C$ dimensions), and another point-wise convolution to produce a skip-connection feature $F_{T}^{l} \in \mathbb{R}^{C \times T}$.
The Adaptive Fusion module then integrates this feature with the Linear Transformer output $H^{(l)}_{LT}$ to produce the fused temporal feature $H^{(l)}_T$:
\begin{equation}
H^{(l)}_T = \mathrm{GELU}\left[\left(F_{T}^{l} \oplus H^{(l)}_{LT}\right) \cdot W_f \cdot W_l\right] + H^{(l)}_{LT},
\end{equation}
where $\oplus$ denotes channel-wise concatenation and $W_f, W_l$ are point-wise convolutions. This $H^{(l)}_T$ is the feature subsequently passed to our Temporal Modulation module (Sec.~3.4).

\subsection{Prediction Refinement Supplement} \label{app:method:prediction}

LaDy follows the standard prediction refinement paradigm, comprising two complementary branches.

\textbf{Classification Prediction Branch.}
The Classification Head produces an initial frame-wise class prediction $Y_c^0 \in \mathbb{R}^{Q \times T}$ from the final temporal representation $F_R$. This is refined through $S_c$ stages. At stage $h$, the previous prediction $Y_c^{h-1}$ is processed by a stack of Linear Transformer layers (using cross-attention, where $Q, K$ are from $Y_c^{h-1}$ and $V$ is from $F_R$) to produce a more precise prediction $Y_c^h$. The final stage yields the refined class prediction $Y_c^F$.

\textbf{Boundary Prediction Branch.}
Similarly, the Boundary Head produces an initial boundary confidence map $Y_b^0 \in \mathbb{R}^{1 \times T}$ from $F_R$. This is refined across $S_b$ stages. At each stage $h$, the previous output $Y_b^{h-1}$ is processed by a stack of dilated 1D TCNs to yield the refined boundary prediction $Y_b^h$. The final output is $Y_b^F$.

\subsection{Loss Function Supplement} \label{app:method:loss}

The overall training objective $\mathcal{L}_{total}$ is a composite loss:
$\mathcal{L}_{total} = \mathcal{L}_{as} + \lambda_{1} \mathcal{L}_{br} + \lambda_{2} \mathcal{L}_{atc} + \lambda_{3} \mathcal{L}_{EC}$. Our proposed Energy Consistency Loss $\mathcal{L}_{EC}$ is detailed in Sec.~3.3. The formulations for the other three standard loss components are detailed below.

\textbf{Action-Text Contrastive Loss ($\mathcal{L}_{atc}$).}
Following~\cite{LaSA}, we align visual features with textual embeddings. The final representation $F_R$ is segmented based on ground-truth, and action-level visual features $A^F$ are obtained via temporal mean pooling. Corresponding textual embeddings $A^E$ are obtained using a pre-trained BERT~\cite{BERT}. We compute a pairwise cosine similarity matrix $S^A = \mathrm{sim}(A^F, A^E)$ and apply bidirectional KL divergence loss against the ground-truth identity matrix $S^{GT}$:
\begin{equation}
\mathcal{L}_{atc} = \frac{1}{2}[\mathcal{D}_{KL}(S^{GT}| S^A_f)+\mathcal{D}_{KL}(S^{GT}| S^A_e)],
\end{equation}
where $S^A_f$ and $S^A_e$ are row- and column-normalized similarity matrices, respectively. $\mathcal{D}_{KL}$ is the KL divergence.

\textbf{Action Segmentation Loss ($\mathcal{L}_{as}$).}
This loss supervises the classification predictions $Y_c$ at all stages. It combines a standard frame-wise cross-entropy loss $\mathcal{L}_{ce}$ and a Gaussian Similarity-weighted Truncated Mean Squared Error (GS-TMSE) smoothing loss $\mathcal{L}_{gs\text{-}tmse}$~\cite{ASRF, DeST} to penalize over-segmentation:
\begin{equation}
\begin{split}
\mathcal{L}_{as}= \mathcal{L}_{ce} + \mathcal{L}_{gs\text{-}tmse} 
= -\frac{1}{T} \sum_{t}{\log(\hat{y}_{t, \hat{c}})} \\
+ \frac{1}{TC} \sum_{t,c} e^{-\frac{\|x_{t}-x_{t-1}\|^2}{2\sigma^2}}\min({|\log(\frac{\hat{y}_{t,c}}{\hat{y}_{t-1,c}})|^2}, \kappa),
\end{split}
\end{equation}
where $\hat{y}_{t} \in \mathbb{R}^C$ represents the predicted class probabilities at frame $t$, $\hat{y}_{t, \hat{c}}$ is the predicted probability for the ground-truth class $\hat{c}$, and $x_t$ is the input feature at frame $t$. The GS-TMSE term dynamically scales the temporal smoothing penalty based on the Gaussian similarity of adjacent features. 
The parameter $\sigma$ controls the similarity sensitivity, and $\kappa$ is a predefined threshold that truncates excessively large gradients.

\textbf{Boundary Prediction Loss ($\mathcal{L}_{br}$).}
This loss supervises the boundary predictions $Y_b$ at all stages. It is a standard binary cross-entropy loss:
\begin{equation}
\mathcal{L}_{br} = -\frac{1}{T} \sum_{t} \left( b_t \log(\hat{b}_t) + (1 - b_t) \log(1 - \hat{b}_t) \right),
\end{equation}
where $b_t \in \{0,1\}$ is the binary ground-truth boundary label at frame $t$, and $\hat{b}_t \in [0,1]$ is the predicted boundary probability.

\section{Detailed Experimental Setup}
\label{sec:supp:setup}

This section provides a comprehensive account of the experimental configuration to ensure reproducibility. Supplementing the primary implementation details outlined in the main paper, we elaborate on dataset specifications, preprocessing protocols, evaluation metrics, training strategy, and specific hyperparameter settings adopted in the LaDy framework.

\subsection{Datasets}
\label{sec:supp:datasets}
We conduct extensive evaluations on six challenging benchmarks, covering diverse domains from daily activities to specialized sports and gestures.

\textbf{PKU-MMD v2}~\cite{PKU-MMD} is a large-scale dataset captured via Kinect v2, containing approx. 50 hours of data with 52 action categories. It provides 3-axis spatial coordinates for 25 body joints. Following standard protocols, we evaluate on two splits: (1) \textit{Cross-Subject (X-sub)}, with 775 training and 234 testing videos; and (2) \textit{Cross-View (X-view)}, training on middle/right views and testing on the left. 

\textbf{LARa}~\cite{LARA} focuses on logistics activities (e.g., picking, packing) recorded by an optical MoCap system. The high-fidelity markers provide 19 joint positions and orientations. It comprises 377 sequences across 8 classes (758 mins). 

\textbf{MCFS-22 \& MCFS-130}~\cite{MCFS} are sourced from the same figure skating repository but annotated at different granularities. Collected via OpenPose (30 Hz), this dataset contains 271 videos and provides 2D joint coordinates for 25 body joints. MCFS-130 provides fine-grained annotations for 130 atomic actions, while MCFS-22 aggregates them into 22 coarse categories. This dual-setting tests the model's capability in handling both semantic hierarchies. 

\textbf{TCG-15}~\cite{tcg} is a traffic control gesture dataset captured by IMU sensors at 100 Hz. It records 3-axis positions for 17 joints. It includes 550 sequences across 15 distinct gesture classes, totaling approx. 140 minutes of data. 

\subsection{Data Preprocessing}
\label{sec:supp:preprocessing}
To ensure fair comparison and input consistency, we adhere to the standard preprocessing protocols established in prior studies~\cite{MS-GCN,DeST,LaSA}.
\begin{itemize}[leftmargin=*]
    \item \textbf{Resampling:} To unify temporal resolution, LARa (originally 200 Hz) is downsampled to 50 Hz. PKU-MMD (50 Hz), MCFS (30 Hz), and TCG-15 (100 Hz) retain their native frame rates to preserve motion fidelity specific to their domains.
    \item \textbf{Feature Extraction:} We construct frame-level input vectors primarily using joint-wise relative coordinates and temporal displacements to ensure translation invariance.
    \begin{itemize}
        \item For \textbf{PKU-MMD} and \textbf{TCG-15}, we extract 6-channel features (3D relative positions + 3D temporal differences) for each joint.
        \item For \textbf{LARa}, we utilize 12-channel features, incorporating both 3D positions and 3D orientations along with their respective temporal derivatives.
        \item For \textbf{MCFS}, given the 2D nature of OpenPose data, we construct 2-channel features using 2D relative coordinates.
    \end{itemize}
\end{itemize}

\subsection{Evaluation Metrics}
\label{sec:supp:metrics}
\noindent\textbf{Segmentation Metrics.} 
We employ three standard metrics to assess segmentation quality:
\begin{enumerate}[label=(\roman*)]
    \item \textbf{Frame-wise Accuracy (Acc):} The percentage of frames correctly classified. While fundamental, it is insensitive to over-segmentation errors.
    \item \textbf{Segmental Edit Score (Edit):} The normalized Levenshtein distance between predicted and ground-truth segment sequences. This metric explicitly penalizes ordering errors and over-segmentation.
    \item \textbf{Segmental F1 Score (F1@$k$):} The harmonic mean of precision and recall for predicted segments that overlap with ground truth by an Intersection-over-Union (IoU) threshold $k$. We report F1 scores at thresholds $k \in \{0.10, 0.25, 0.50\}$ to evaluate temporal boundary precision at varying strictness levels.
\end{enumerate}

\noindent\textbf{Clustering Metrics.} 
To quantify the discriminability of the learned latent space (as visualized in Fig.~5 of the main paper), we utilize three intrinsic clustering indicators:
\begin{enumerate}[label=(\roman*)]
    \item \textbf{Silhouette Coefficient (SC):} Measures how similar a sample is to its own cluster (cohesion) compared to other clusters (separation). Values range from -1 to +1, where higher values indicate better-defined clusters.
    \item \textbf{Calinski-Harabasz Index (CH):} The ratio of the sum of between-clusters dispersion to within-cluster dispersion. A higher CH score signifies dense and well-separated clusters.
    \item \textbf{Davies-Bouldin Index (DB):} The average similarity measure of each cluster with its most similar cluster. Unlike SC and CH, a lower DB index indicates better separation and compactness.
\end{enumerate}

\subsection{Training Strategy}
\label{sec:supp:training}

\noindent\textbf{Loss Configuration.}
The total objective function is a weighted sum of four components: the segmentation loss $\mathcal{L}_{as}$, the boundary loss $\mathcal{L}_{br}$, the action-text contrastive loss $\mathcal{L}_{atc}$, and our proposed Energy Consistency loss $\mathcal{L}_{EC}$. The balancing hyperparameters are set as follows:
\begin{itemize}
    \item $\lambda_1 = 1.0$ for $\mathcal{L}_{br}$, following the established settings in~\cite{ASRF, DeST, LaSA} to enforce boundary regression.
    \item $\lambda_2 = 0.8$ for $\mathcal{L}_{atc}$, consistent with language-assisted frameworks like~\cite{LaSA, TRG-Net}, ensuring semantic alignment.
    \item $\lambda_3 = 0.1$ for $\mathcal{L}_{EC}$, determined via empirical ablation to balance physical regularization with kinematic learning.
\end{itemize}

\noindent\textbf{Delayed Physics Injection (Warmup Strategy).}
Applying the strict Energy Consistency constraint ($\mathcal{L}_{EC}$) from the initial iteration can destabilize training, as the dynamic estimators (e.g., Inertia estimator $\mathcal{F}_M$, Coriolis estimator $\mathcal{F}_C$) are initialized randomly and require time to converge to plausible physical manifolds. To mitigate this, we design a \textit{Delayed and Phased Warmup} strategy:

\begin{equation}
    \lambda_{3}^{(e)} = 
    \begin{cases} 
    0 & e < \mathcal{Z}, \\
    \frac{e - \mathcal{Z}}{\mathcal{Z}_{w}} \cdot \lambda_{3} & \mathcal{Z} \le e < \mathcal{Z} + \mathcal{Z}_{w}, \\
    \lambda_{3} & e \ge \mathcal{Z} + \mathcal{Z}_{w},
    \end{cases}
\end{equation}
where $e$ denotes the current training epoch, and $\lambda_{3}^{(e)}$ is the weighting coefficient for the $\mathcal{L}_{EC}$ at epoch $e$. The physical regularization is deactivated for the first $\mathcal{Z}$ epochs (post-start), allowing the model to first establish a coarse kinematic representation. Subsequently, the weight is linearly ramped up to $\lambda_3=0.1$ over a short period $\mathcal{Z}_{w}$ (typically spanning 300 mini-batches, equating to approx. 3-5 epochs). This curriculum ensures that strict physical laws are enforced only after the dynamical priors have stabilized, effectively preventing gradient conflict during the early training phase.

\subsection{Implementation Details}
\label{sec:supp:implementation}

\noindent\textbf{Network Architecture \& Hyperparameters.}
We incorporate specific constants to guarantee numerical stability in the physics-constrained modules: $\epsilon=10^{-5}$ is added to the diagonal entries of $L(q)$ (Eq.~(5)) to strictly enforce the positive definiteness of the inertia matrix $M(q)$, and $\delta=0.1$ is applied to the denominator of the relative energy residual (Eq.~(12)) to prevent division by zero. The warmup threshold $\mathcal{Z}$ is set adaptively based on dataset scale: $\mathcal{Z}=20$ for LARa (trained for 60 epochs) and $\mathcal{Z}=50$ for all other datasets (trained for 300 epochs).
For the kinematic backbone, we adhere to the configurations of established baselines~\cite{DeST,ASRF}. The multi-scale GCN operates with a scale $K=13$. The Linear Transformer employs 4 attention heads, with the channel dimensions for query, key, and value projections fixed at $C_2=16$. For prediction refinement, we utilize $S_c=1$ stage for the class prediction branch and $S_b=2$ stages for the boundary regression branch.

\noindent\textbf{Text Embeddings.}
For the action-text contrastive loss, semantic representations are derived from the natural language descriptions of each action category. We utilize a pretrained BERT~\cite{BERT} encoder to extract word-level tokens, which are then averaged to produce sentence-level embeddings with a fixed dimensionality of $\mathbb{R}^{768}$. To ensure alignment across modalities, the dimensionality of the action-level visual representation is projected to $C_t = 768$.

\section{Comparative Experiments}
\label{sec:supp:CE}

In this section, we conduct a multi-dimensional comparative analysis to comprehensively substantiate the efficacy and robustness of the proposed LaDy framework. First, we visualize the performance-efficiency trade-off to highlight LaDy's exceptional computational cost-effectiveness (Sec.~\ref{sec:supp:PET}). Subsequently, we quantify boundary localization precision across continuous IoU thresholds, validating the model's superior temporal alignment under stringent evaluation regimes (Sec.~\ref{sec:supp:BPA}). This is followed by a granular per-class breakdown, which reveals how explicit dynamic priors effectively resolve specific kinematic ambiguities (Sec.~\ref{sec:supp:PPA}). Furthermore, extended qualitative visualizations are provided to corroborate the structural coherence and boundary sharpness achieved by our approach (Sec.~\ref{sec:supp:EQA}). Finally, we assess the model's resilience against complex input perturbations, demonstrating the powerful intrinsic regularization provided by our physics-constrained design (Sec.~\ref{sec:supp:robustness}).

\subsection{Performance vs. Efficiency Trade-off}
\label{sec:supp:PET}

\begin{figure}[tb]
    \centering
    \includegraphics[width=\linewidth]{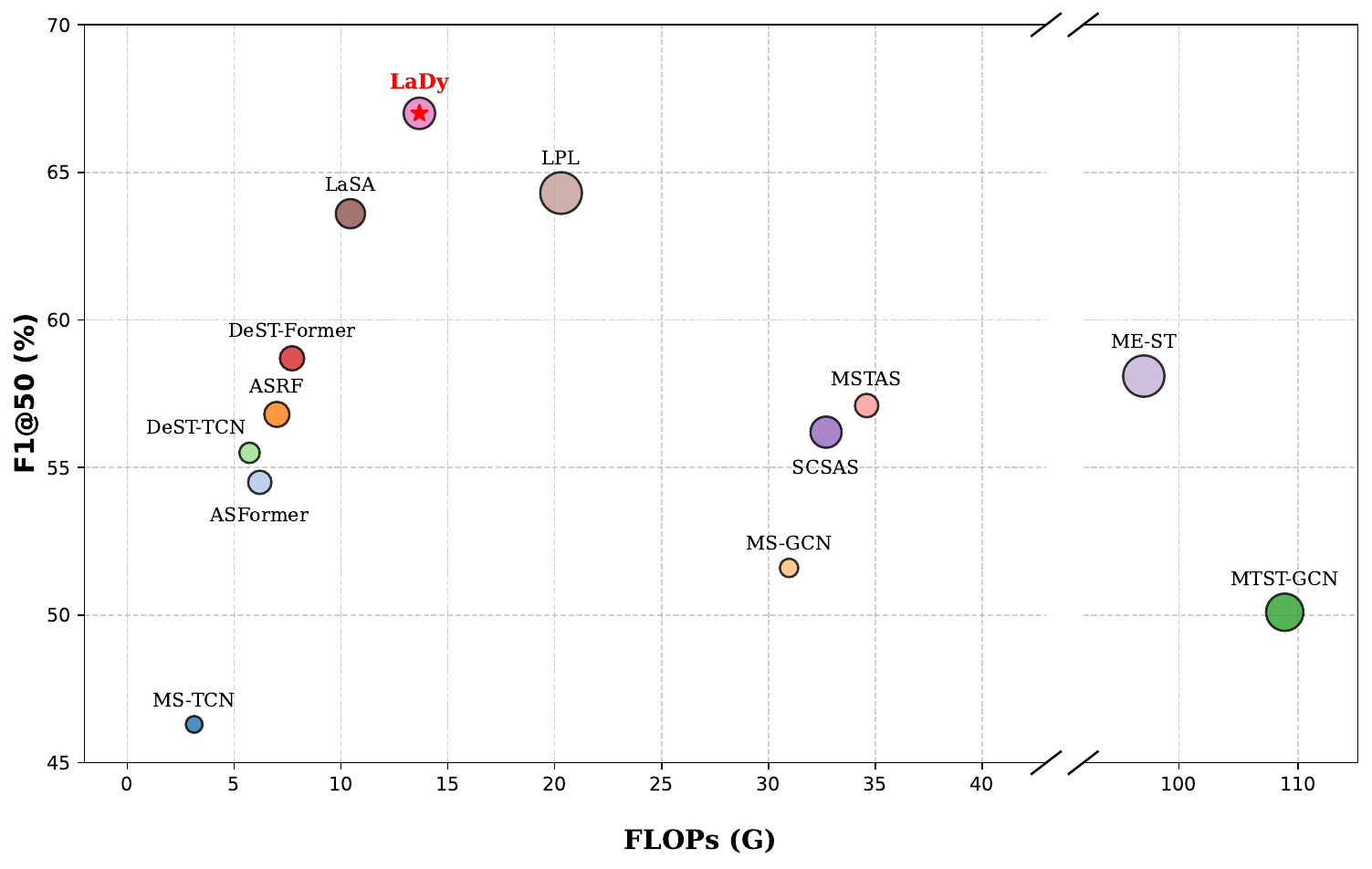}
    
    \caption{Performance vs. efficiency comparison on the PKU-MMD v2 dataset. The y-axis represents the F1@50 score, and the x-axis indicates computational complexity (FLOPs). Circle area is proportional to the model's parameter count. LaDy (red star) achieves state-of-the-art performance with a highly competitive computational footprint.}
    \label{fig:11}
\end{figure}

To provide a more intuitive visualization of the performance vs. efficiency trade-off discussed in the main text (Sec.~4.2 and Tab.~1), we plot the performance scores (F1@50) of recent state-of-the-art methods against their computational complexity (FLOPs) on the PKU-MMD v2 dataset (Fig.~\ref{fig:11}). The area of each circle reflects the corresponding model's parameter count.
As illustrated, our LaDy framework establishes a new state-of-the-art in segmentation accuracy while maintaining a remarkably lightweight footprint (13.67G FLOPs). Compared to recent computationally heavy models such as LPL (20.3G), MSTAS (34.6G), and ME-ST (97.07G), LaDy achieves superior performance with significantly lower costs. This exceptional cost-to-performance ratio stems from our architectural design: the proposed physical branch avoids heavy computational overhead by relying solely on efficient MLPs and structured matrix operations. Specifically, integrating the Lagrangian Dynamics Synthesis (LDS) and Spatio-Temporal Modulation (STM) modules---coupled with the Energy Consistency Loss (ECLoss) during training---introduces a marginal overhead of merely 3.5G FLOPs and 0.51M parameters over the baseline, yet yields a substantial 2.7\% absolute gain in F1@50. This confirms that explicitly leveraging physical priors is a highly efficient strategy for action segmentation, circumventing the need for excessive parametric scaling.

\begin{figure}[t]
  \centering
  \begin{subfigure}[b]{0.48\textwidth}
    \includegraphics[width=\textwidth]{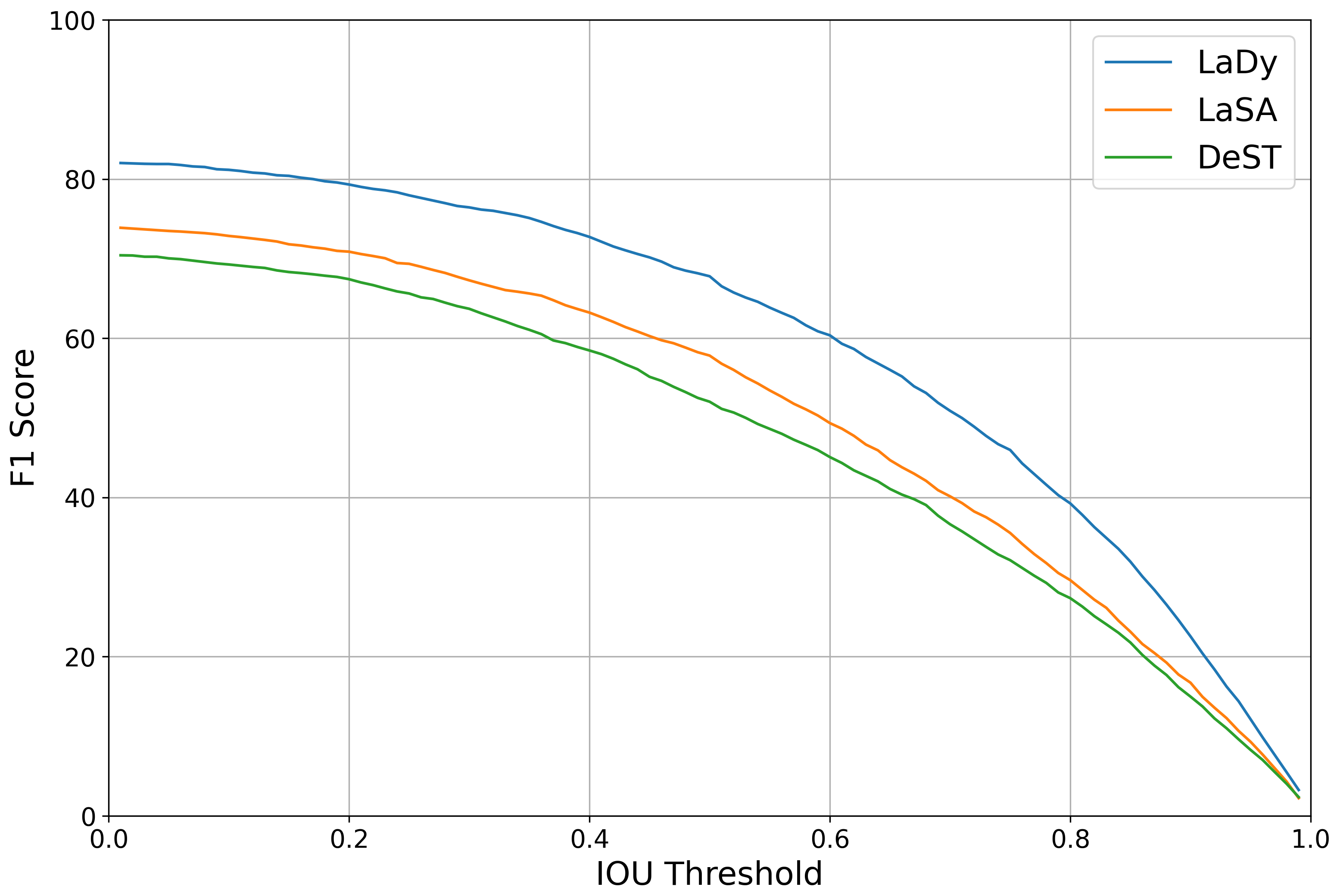}
    \caption{F1 vs. IoU on PKU-MMD v2 (X-view)}
    \label{fig:10-1}
  \end{subfigure}
  \hfill
  \begin{subfigure}[b]{0.48\textwidth}
    \includegraphics[width=\textwidth]{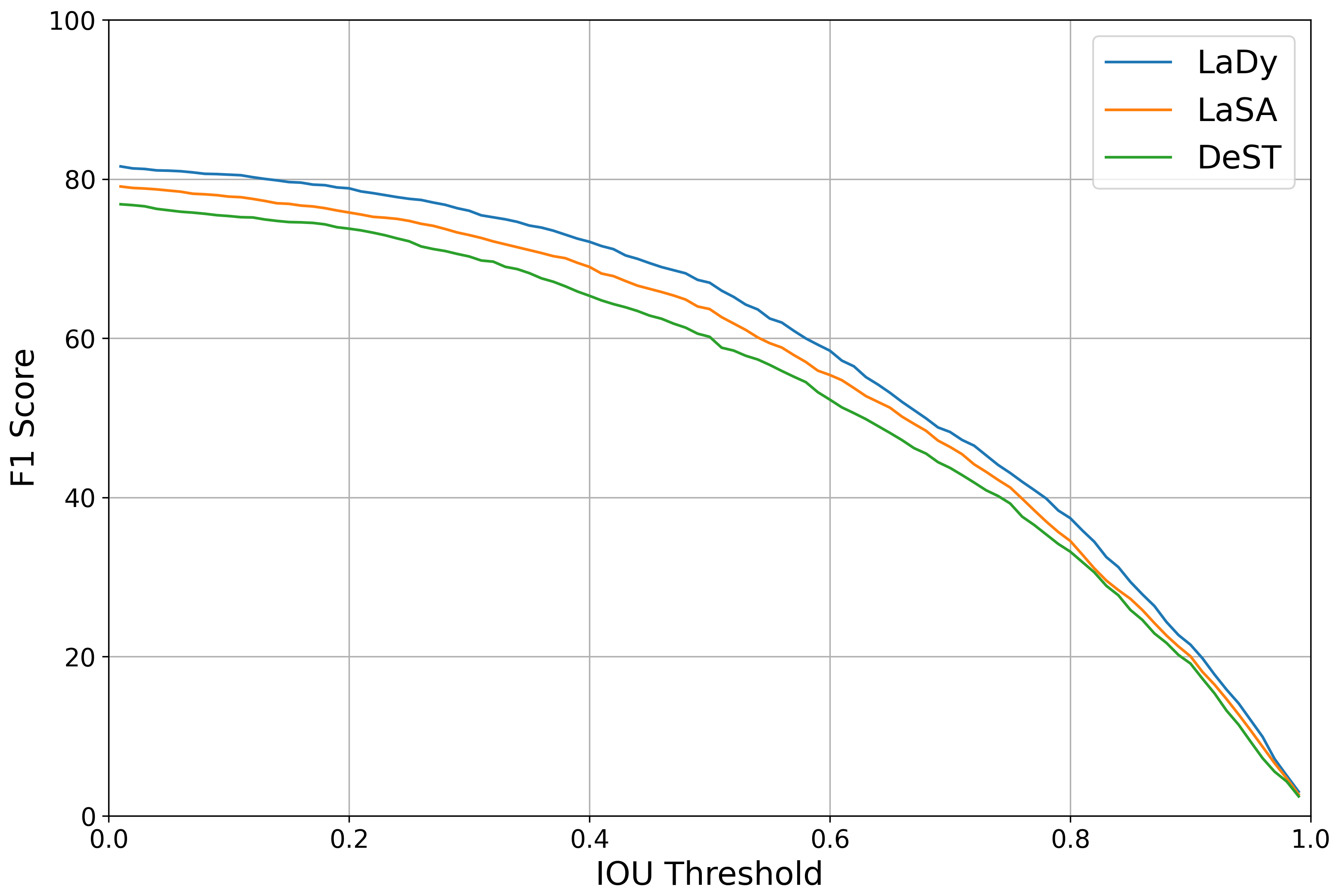}
    \caption{F1 vs. IoU on PKU-MMD v2 (X-sub)}
    \label{fig:10-2}
  \end{subfigure}
  
  \caption{F1 scores vs. IoU thresholds on PKU-MMD v2. LaDy (blue) maintains a distinct performance margin over LaSA (orange) and DeST (green), particularly at strict thresholds ($\text{IoU} > 0.7$). This demonstrates that integrating Lagrangian dynamics (salient dynamic signals) for temporal modulation yields significantly sharper boundary localization.}
  \label{fig:10}
\end{figure}

\subsection{Boundary Precision Analysis}
\label{sec:supp:BPA}

We evaluate the temporal precision of LaDy by analyzing F1 scores across a continuous spectrum of IoU thresholds ($\in [0, 1]$). As shown in Fig.~\ref{fig:10}, LaDy consistently outperforms the previous state-of-the-art kinematic models (LaSA~\cite{LaSA}, DeST~\cite{DeST}) across all thresholds on PKU-MMD v2. Crucially, LaDy exhibits superior robustness against localization strictness. While all methods display a monotonic performance decay as the IoU threshold tightens, LaDy demonstrates a significantly slower decay rate. Consequently, the relative performance margin between LaDy and competing methods widens considerably in the high-precision regime ($\text{IoU} \ge 0.7$). This sustained accuracy indicates that our predicted segments possess higher structural alignment with the ground truth, effectively mitigating the boundary jitter common in kinematic-only models. This precision stems from the Spatio-Temporal Modulation (STM) mechanism. Unlike kinematic transitions, which are often smoothed and ambiguous, action boundaries are physically defined by abrupt shifts in force profiles. By explicitly leveraging the salient dynamic signals (e.g., torque change) for temporal gating, LaDy captures these sharp dynamic transients, transforming physical force variations into precise cues for boundary localization.

\begin{figure*}[tb]
    \centering
    \includegraphics[width=\linewidth]{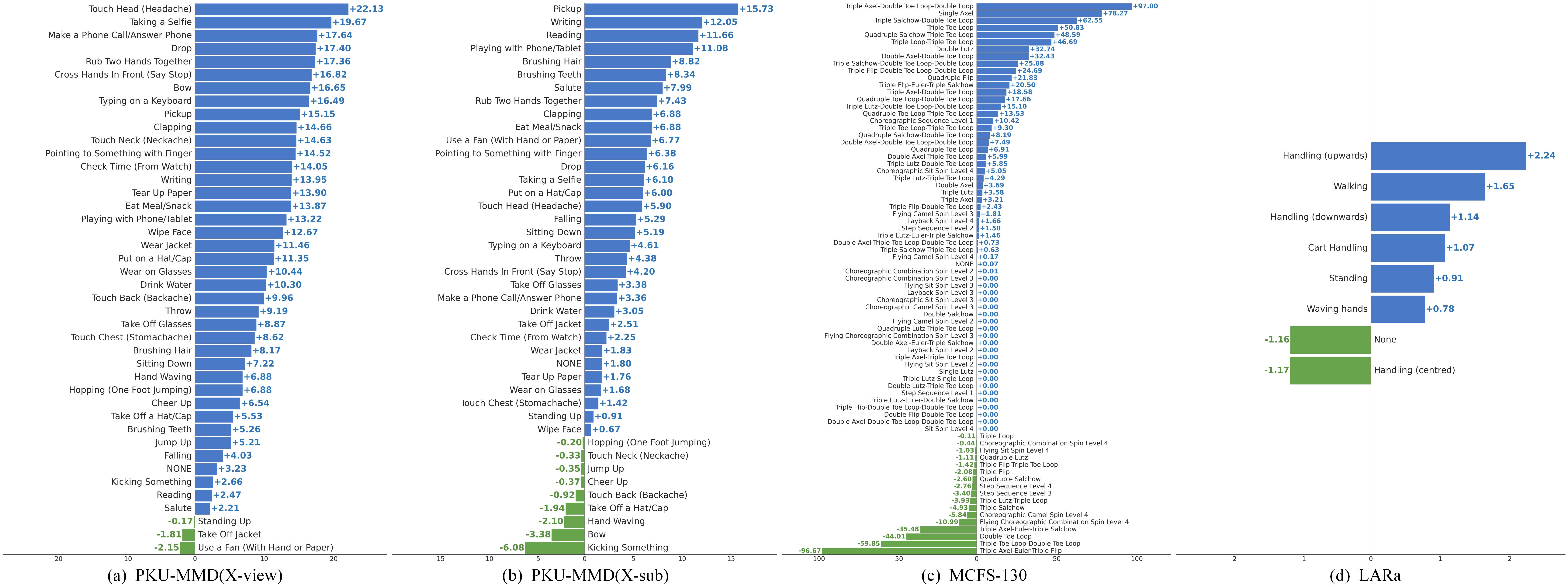}
    
    \caption{Per-class frame-wise F1 improvement of LaDy over the previous method (LaSA) on the PKU-MMD v2 (X-view and X-sub), MCFS-130, and LARa datasets. Blue bars indicate performance gains, while green bars denote declines. The results demonstrate that LaDy yields significant improvements in complex, physically distinct actions, validating that explicitly modeling dynamics effectively resolves kinematic ambiguities.}
    \label{fig:9}
\end{figure*}

\subsection{Per-Class Performance Analysis}
\label{sec:supp:PPA}

To provide a granular understanding of how physical dynamics contribute to segmentation performance, we present a class-wise comparison between LaDy and the state-of-the-art kinematic method (LaSA) across four datasets in Fig.~\ref{fig:9}. The results reveal a clear pattern: LaDy achieves broad and significant improvements across a diverse spectrum of actions, validating the universality of the proposed dynamics-informed modeling.

\textbf{Resolving Kinematic Ambiguities via Dynamic Signatures.}
The most substantial gains are observed in actions that are kinematically similar but dynamically distinct.
On PKU-MMD (Fig.~\ref{fig:9}(a), (b)), subtle interactions such as \textit{Touch Head}, \textit{Make a Phone Call}, \textit{Brush Hair}, and \textit{Brush Teeth} see improvements in both X-sub and X-view settings. Kinematically, these actions all involve bringing the hand near the head, creating severe inter-class confusion for conventional models. However, LaDy effectively distinguishes them by capturing their unique ``dynamic signatures''---the specific torque profiles required to stabilize the limb in these varying postures and the distinct force transients during the initiation of the movement.
Similarly, on MCFS-130 (Fig.~\ref{fig:9}(c)), our method shows substantial gains in recognizing and distinguishing highly complex figure skating maneuvers, particularly the entire family of \textit{Toe Loop}.
These actions are defined by precise physical constraints (e.g., angular momentum conservation). While kinematic models struggle to differentiate the rapid, blurred rotations of a \textit{Triple} vs. \textit{Double} jump, LaDy's LDS module explicitly infers the magnitude of the driving torques, providing a decisive cue for classifying these high-energy, physics-dominated motions.

\textbf{Analysis of Performance Degradation.} 
Despite the overall superiority, we observe performance drops in a few specific categories, such as \textit{Use a fan} (PKU-MMD X-view), \textit{Kicking Something} (PKU-MMD X-sub) and \textit{Handling (centred)} (LARa). We attribute this to two physical factors: 
(1) \textbf{Inaccurate External Force Estimation:} Actions like \textit{Kicking} involve significant impact forces with external objects. Although the generalized force term $F(q, \dot{q})$ accounts for non-conservative and external forces in our Lagrangian equation, it is estimated solely from the input kinematics. Without explicit external contact sensing, the sudden velocity change upon impact cannot be accurately modeled, resulting in a transient violation of the estimated dynamics and confusing the network.
(2) \textbf{Low-Energy Dynamics:} For passive or low-energy motions like \textit{Use a fan}, the signal-to-noise ratio of the estimated joint torques is lower compared to high-energy actions. In these cases, the strong kinematic periodicity dominates, and the auxiliary dynamic features may introduce minor interference. 

Nevertheless, the overwhelming prevalence of positive gains (Blue bars) confirms that for the vast majority of human actions, the integration of physical dynamics serves as a powerful and complementary prior.

\subsection{Extended Qualitative Analysis}
\label{sec:supp:EQA}

\begin{figure*}[t]
  \centering
  \begin{subfigure}{0.48\textwidth} 
    \centering
    \includegraphics[width=\linewidth]{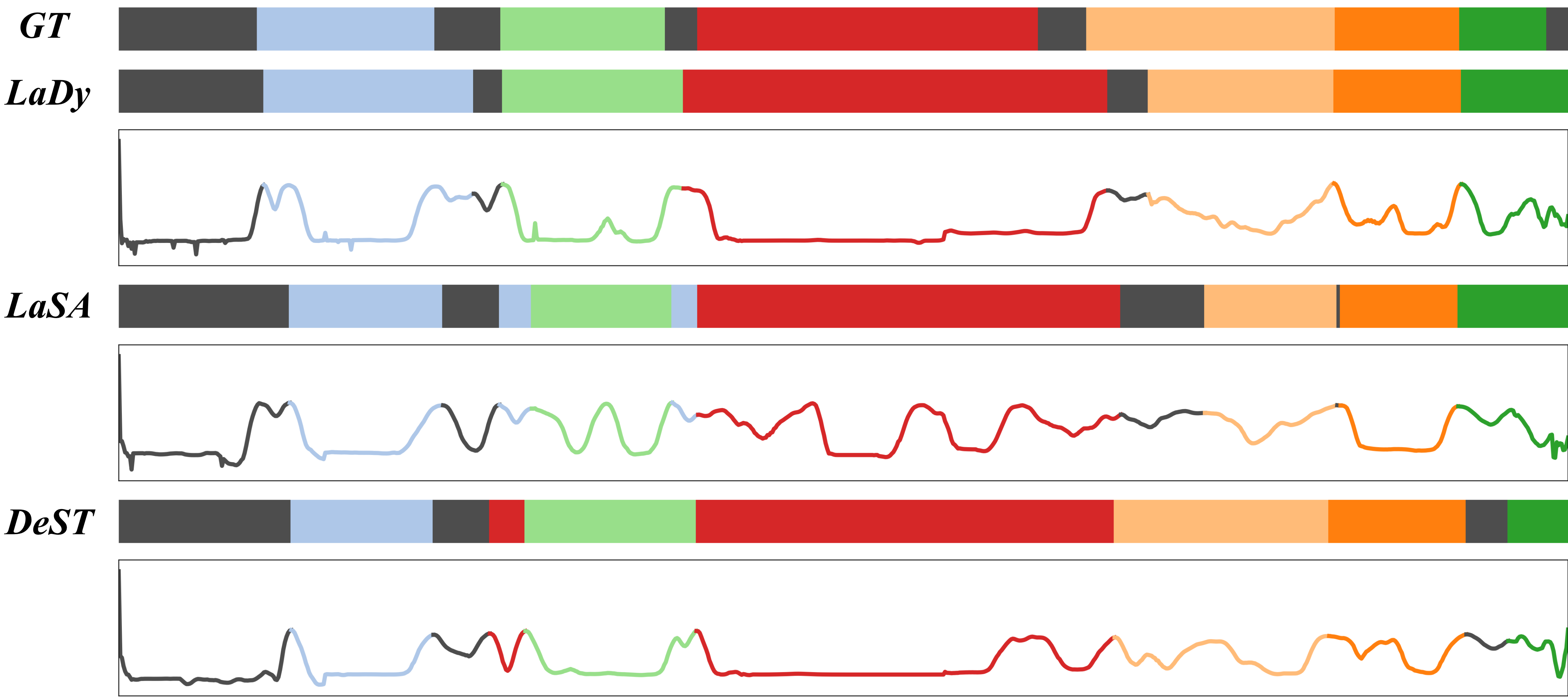} 
    \caption{PKU-MMD v2 (X-view)}
    \label{fig:8-1}
  \end{subfigure}
  \begin{subfigure}{0.48\textwidth}
    \centering
    \includegraphics[width=\linewidth]{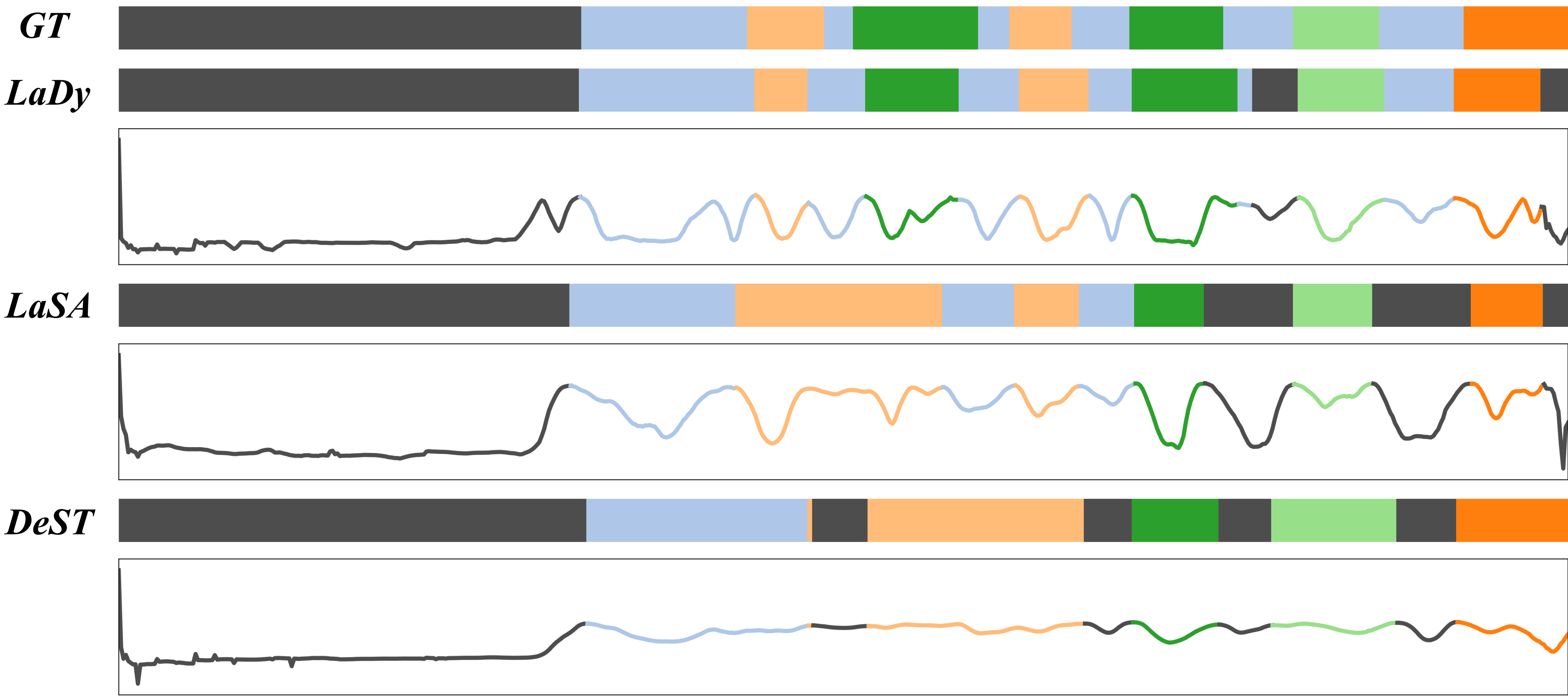} 
    \caption{PKU-MMD v2 (X-sub)}
    \label{fig:8-2}
  \end{subfigure}
    \begin{subfigure}{0.48\textwidth} 
    \centering
    \includegraphics[width=\linewidth]{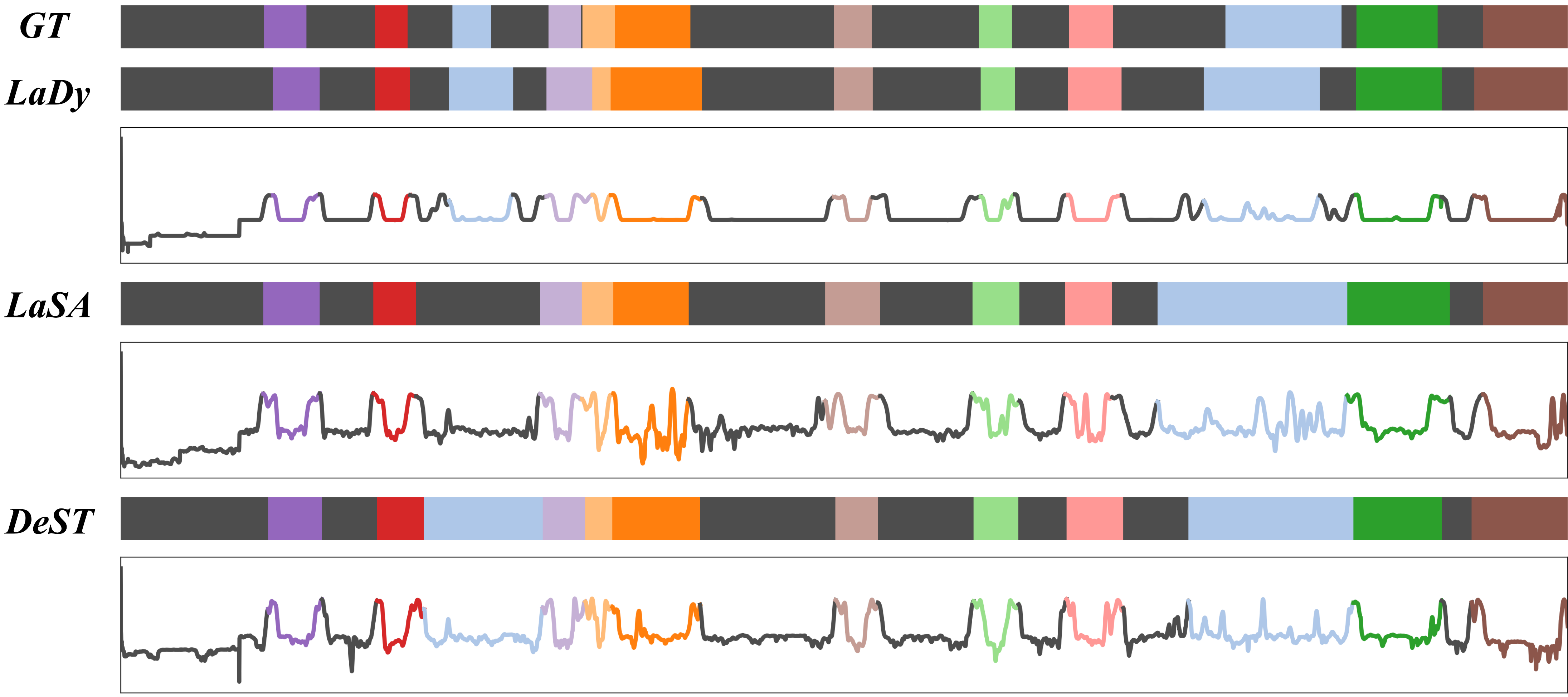} 
    \caption{MCFS-22}
    \label{fig:8-3}
  \end{subfigure}
  \begin{subfigure}{0.48\textwidth}
    \centering
    \includegraphics[width=\linewidth]{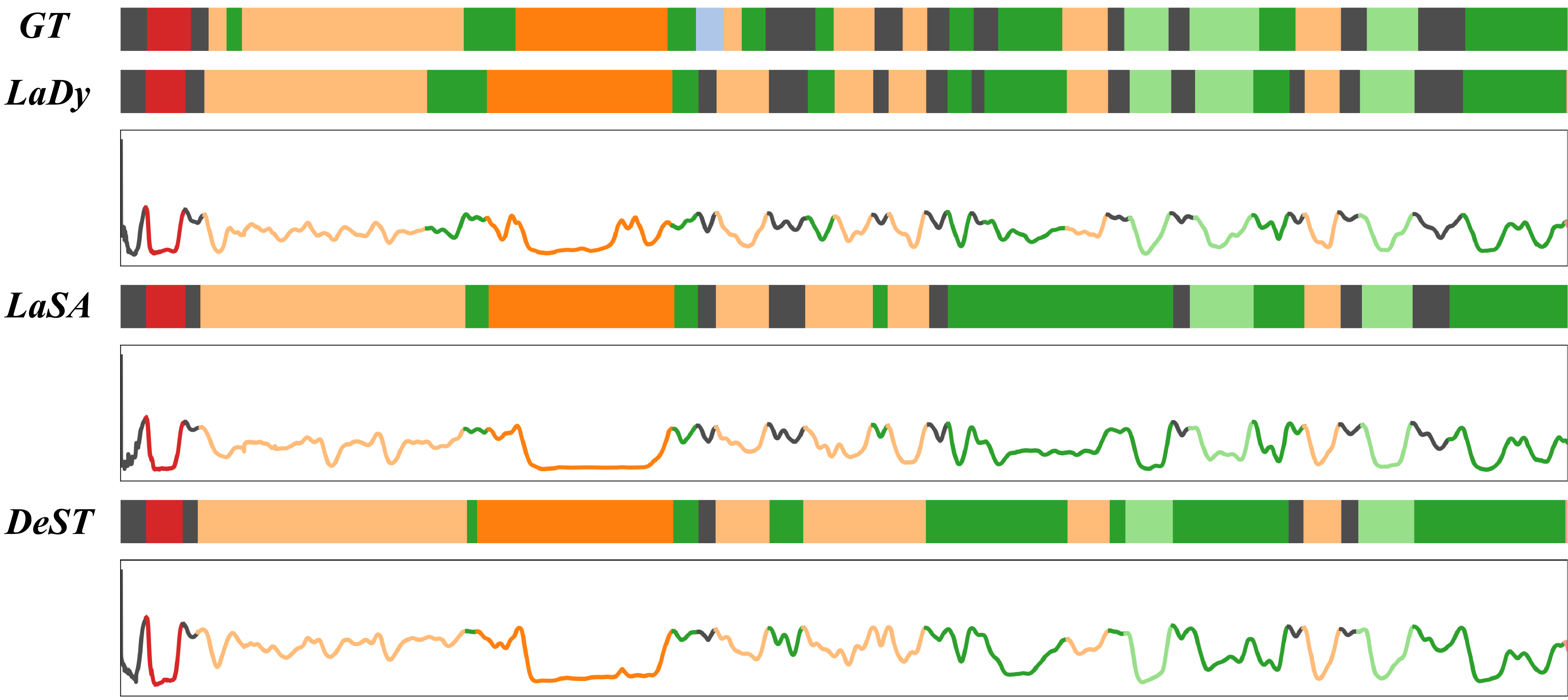} 
    \caption{LARa}
    \label{fig:8-4}
  \end{subfigure}

  \caption{Qualitative results on PKU-MMD v2 (X-view and X-sub), MCFS-22, and LARa. The top row is the Ground Truth, followed by the segmentation results (bars) and boundary confidence scores (curves) for LaDy, LaSA, and DeST. Different colors denote distinct action classes.}
  \label{fig:8}
\end{figure*}

To provide a comprehensive assessment of the model's segmentation capability, we present extended qualitative comparisons across four benchmarks: PKU-MMD v2 (X-view/X-sub), MCFS-22, and LARa, as illustrated in Fig.~\ref{fig:8}.

\textbf{Segmentation Consistency and Boundary Sharpness.} 
Consistent with the main paper, LaDy generates segmentation predictions that are topologically more congruent with the Ground Truth (GT) compared to LaSA~\cite{LaSA} and DeST~\cite{DeST}. The improvements are physically interpretable through the boundary confidence curves (plotted below the segmentation bars). LaDy's curves exhibit two distinct characteristics: (1) \textbf{Intra-segment Stability:} Within action segments, the confidence scores remain flat and low, effectively suppressing the ``over-segmentation'' noise often triggered by kinematic jitter in previous methods (e.g., the fluctuating curves in LaSA). (2) \textbf{Inter-segment Sharpness:} At action transitions, LaDy produces prominent, narrow peaks. This confirms that the salient dynamic (e.g., torque change) signal serves as a decisive cue, enabling the model to ``lock on'' to the physical onset of motion rather than drifting with visual ambiguity.

\textbf{Analysis of Residual Errors and Failure Modes.} 
While LaDy significantly reduces errors, it is not immune to failure. As observed in the dense transitions of LaDy's prediction bars, minor temporal boundary shifts persist. We attribute these deviations to ``soft transitions''---scenarios where the dynamic force profile changes gradually rather than abruptly (e.g., slowly transitioning from a stance to a walk), making the precise definition of a ``boundary'' inherently ambiguous even for physics-informed models. Furthermore, sporadic classification errors remain in semantically overlapping classes (e.g., the middle segment in Fig.~\ref{fig:8}a). This suggests that while Lagrangian dynamics resolve kinematic ambiguities, future work could further benefit from integrating high-level semantic context to address these residual categorical confusions.

\subsection{Robustness to Input Perturbations}
\label{sec:supp:robustness}

\begin{table*}[tb]
    \footnotesize
    \centering
    \caption{Robustness evaluation against various input perturbations (Gaussian noise, joint occlusion, and missing frames) on the PKU-MMD v2 (X-sub) dataset. All models are trained on clean data and evaluated on corrupted test sets. For each method, the first row reports the absolute performance, and the second row indicates the relative performance drop compared to the clean setting.}
    \setlength{\tabcolsep}{3pt}
    \begin{tabular}{c|ccccc|ccccc|ccccc}
        \toprule
        \textbf{Noise} & \multicolumn{5}{c|}{\textbf{Gaussian Noise}} & \multicolumn{5}{c|}{\textbf{Joint Occlusion}} & \multicolumn{5}{c}{\textbf{Missing Frames}} \\
        \cmidrule(lr){2-6} \cmidrule(lr){7-11} \cmidrule(lr){12-16}
        
        Metric & Acc & Edit & \multicolumn{3}{c}{F1@\{10, 25, 50\}} & Acc & Edit & \multicolumn{3}{c}{F1@\{10, 25, 50\}} & Acc & Edit & \multicolumn{3}{c}{F1@\{10, 25, 50\}} \\
        \midrule
        
        DeST& 65.5 & 63.8 & 69.8 & 65.4 & 51.7 & 58.3 & 57.4 & 61.2 & 56.6 & 43.7 & 66.9 & 65.9 & 71.7 & 67.2 & 52.7 \\
        & \scriptsize \textcolor{olive}{-6.8\%} & \scriptsize \textcolor{olive}{-7.9\%} & \scriptsize \textcolor{olive}{-6.4\%} & \scriptsize \textcolor{olive}{-7.9\%} & \scriptsize \textcolor{olive}{-12.0\%} & \scriptsize \textcolor{olive}{-17.1\%} & \scriptsize \textcolor{olive}{-17.1\%} & \scriptsize \textcolor{olive}{-17.9\%} & \scriptsize \textcolor{olive}{-20.2\%} & \scriptsize \textcolor{olive}{-25.6\%} & \scriptsize \textcolor{olive}{-4.8\%} & \scriptsize \textcolor{olive}{-5.0\%} & \scriptsize \textcolor{olive}{-3.7\%} & \scriptsize \textcolor{olive}{-5.4\%} & \scriptsize \textcolor{olive}{-10.3\%} \\
        
        LaSA& 70.4 & 70.7 & 75.9 & 71.5 & 59.4 & 60.4 & 60.0 & 64.9 & 59.1 & 48.0 & 70.2 & 69.9 & 75.5 & 71.3 & 58.0 \\
        & \scriptsize \textcolor{olive}{-4.2\%} & \scriptsize \textcolor{olive}{-3.7\%} & \scriptsize \textcolor{olive}{-3.1\%} & \scriptsize \textcolor{olive}{-4.4\%} & \scriptsize \textcolor{olive}{-6.6\%} & \scriptsize \textcolor{olive}{-17.8\%} & \scriptsize \textcolor{olive}{-18.2\%} & \scriptsize \textcolor{olive}{-17.1\%} & \scriptsize \textcolor{olive}{-21.0\%} & \scriptsize \textcolor{olive}{-24.5\%} & \scriptsize \textcolor{olive}{-4.5\%} & \scriptsize \textcolor{olive}{-4.8\%} & \scriptsize \textcolor{olive}{\textbf{-3.6\%}} & \scriptsize \textcolor{olive}{-4.7\%} & \scriptsize \textcolor{olive}{-8.8\%} \\
        
        \textbf{LaDy} & \textbf{73.6} & \textbf{73.1} & \textbf{78.2} & \textbf{75.4} & \textbf{64.2} & \textbf{70.1} & \textbf{68.6} & \textbf{73.6} & \textbf{70.1} & \textbf{59.5} & \textbf{73.4} & \textbf{71.7} & \textbf{77.1} & \textbf{74.5} & \textbf{62.1} \\
        & \scriptsize \textcolor{olive}{\textbf{-3.4\%}}& \scriptsize \textcolor{olive}{\textbf{-2.6\%}} & \scriptsize \textcolor{olive}{\textbf{-2.4\%}} & \scriptsize \textcolor{olive}{\textbf{-2.5\%}} & \scriptsize \textcolor{olive}{\textbf{-4.2\%}} & \scriptsize \textcolor{olive}{\textbf{-8.0\%}} & \scriptsize \textcolor{olive}{\textbf{-8.6\%}} & \scriptsize \textcolor{olive}{\textbf{-8.1\%}} & \scriptsize \textcolor{olive}{\textbf{-9.3\%}} & \scriptsize \textcolor{olive}{\textbf{-11.2\%}} & \scriptsize \textcolor{olive}{\textbf{-3.7\%}} & \scriptsize \textcolor{olive}{\textbf{-4.5\%}} & \scriptsize \textcolor{olive}{\underline{-3.7\%}} & \scriptsize \textcolor{olive}{\textbf{-3.6\%}} & \scriptsize \textcolor{olive}{\textbf{-7.3\%}} \\
        \bottomrule
    \end{tabular}
    \label{tab:robust}
\end{table*}

To further evaluate the model's resilience against complex environmental noise, we conduct a robustness analysis under three synthetic noise conditions: Gaussian noise, joint occlusion, and missing frames. All evaluated models are trained solely on the clean training set and subsequently tested on the corrupted versions of the testing set. To ensure a strictly fair comparison, identical random seeds are applied across all methods to generate the exact same noise patterns. As reported in Tab.~\ref{tab:robust}, LaDy consistently achieves the highest absolute performance across all corruption scenarios. More importantly, it exhibits significantly lower relative degradation rates compared to the previous state-of-the-art models, DeST and LaSA. This exceptional resilience substantiates the critical role of our Lagrangian Dynamics Synthesis (LDS) module. The embedded physical constraints, coupled with the Energy Consistency Loss, act as powerful intrinsic regularizers. By enforcing global dynamic consistency and work-energy principles, the physical branch effectively mitigates kinematic noise and fills observational gaps, enabling the model to maintain robust representations even when the raw data is heavily compromised.

\section{Detailed Ablation Studies}
\label{sec:supp:DAS}

In this section, we present a comprehensive ablation analysis to scrutinize the internal mechanisms of LaDy. Our evaluation is organized into three logical parts corresponding to the model's core contributions: (1) validating the structural necessity of the Lagrangian Dynamics Synthesis (LDS) (Sec.~\ref{sec:supp:LDS}); (2) optimizing the formulation and stability of the Energy Consistency Loss (ECLoss) (Sec.~\ref{sec:supp:ECLoss}); and (3) identifying the optimal structure for the Spatio-Temporal Modulation (STM) module (Sec.~\ref{sec:supp:STM}).

\subsection{Ablation on Lagrangian Dynamics Synthesis}
\label{sec:supp:LDS}

We systematically validate the design of the LDS module through four focused experiments. First, we investigate the impact of kinematic smoothing strategies on the computation of generalized kinematics. Second, we verify the superiority of our physics-constrained synthesis over black-box and unconstrained baselines. Third, we dissect the impact of specific geometric constraints imposed on the Inertia ($M$), Coriolis ($C$), and Gravity ($G$) terms. Finally, we assess the completeness of the dynamic formulation by evaluating the contribution of each individual Lagrangian component to the segmentation performance.

\begin{table}[htbp]
    \footnotesize
    \centering
    \setlength{\tabcolsep}{4.5pt}
    \renewcommand{\arraystretch}{1.1}
    
    \caption{Ablation on the calculation of generalized velocities and accelerations on the PKU-MMD v2 (X-sub) dataset. We compare our simple finite difference approach against explicit pre-smoothing techniques (Gaussian and Savitzky-Golay filtering).}
    
    \begin{tabular}{l|ccccc}
        \toprule
        \textbf{Loss Formulation} & \textbf{Acc} & \textbf{Edit} & \multicolumn{3}{c}{\textbf{F1@\{10, 25, 50\}}} \\
        \midrule
        \textbf{Finite difference (Ours)} & \textbf{76.2}& \textbf{75.1}& \textbf{80.1}& \textbf{77.3}& \textbf{67.0} \\
        Gaussian smoothing & 75.8& 74.6& 79.4& 76.8& 66.5 \\
        Savitzky-Golay & 75.9& 74.9& 79.8& 76.8& 66.4\\
        \bottomrule
    \end{tabular}
    \label{tab:FD}
\end{table}

\textbf{Impact of Kinematic Smoothing Strategies.}
To compute the generalized velocities $\dot{q}$ and accelerations $\ddot{q}$, our framework employs a straightforward first-order finite difference (Eq.~(4)). To address concerns that discrete differentiation might amplify input noise, we compare this simple approximation against explicit pre-smoothing techniques, specifically Gaussian and Savitzky-Golay filtering. As shown in Tab.~\ref{tab:FD}, introducing these filters does not improve performance; rather, it causes a slight degradation. We attribute this to two primary factors. First, the downstream spatio-temporal backbone, coupled with our physics-constrained modules and Energy Consistency Loss ($\mathcal{L}_{EC}$), inherently functions as a powerful intrinsic regularizer. The network effectively learns to handle input jitter implicitly by optimizing the generalized forces, rendering heuristic pre-denoising redundant. Second, explicit smoothing inevitably suppresses high-frequency kinematic signals. In action segmentation, these high-frequency transients are not merely noise; they are critical dynamic cues that represent abrupt force shifts at genuine action boundaries. Consequently, pre-smoothing artificially blurs these physical transitions, thereby diminishing the temporal gating module's sensitivity to boundary shifts and leading to suboptimal segmentation performance.

\begin{table}[htbp]
    \footnotesize
    \centering
    \setlength{\tabcolsep}{4pt} 
    \renewcommand{\arraystretch}{1.1} 
    
    \caption{Ablation on Dynamics Modeling Strategies on PKU-MMD v2 (X-sub) dataset. We compare direct estimation against equation-based synthesis with and without physical constraints and energy supervision.}
    
    \begin{tabular}{lc|ccccc}
        \toprule
        \textbf{Modeling Strategy} & $\mathcal{L}_{EC}$ & \textbf{Acc} & \textbf{Edit} & \multicolumn{3}{c}{\textbf{F1@\{10, 25, 50\}}} \\
        \midrule
        \textit{Baseline} (w/o LDS) & - & 73.6 & 73.0 & 78.2 & 74.6 & 64.3 \\
        Direct MLP Mapping & - & 75.0 & 74.1 & 79.0 & 76.2 & 65.8 \\
        Unconstrained Eq. & - & 75.2 & 74.5 & 79.7 & 76.5 & 66.2 \\
        Unconstrained Eq. & \checkmark & 75.3 & 74.7 & 79.3 & 76.5 & 66.7 \\
        Physics-Constrained Eq. & - & 75.9 & 74.6 & 79.6 & 76.9 & 66.3 \\
        \textbf{Physics-Constrained Eq.} & \checkmark & \textbf{76.2} & \textbf{75.1} & \textbf{80.1} & \textbf{77.3} & \textbf{67.0} \\
        \bottomrule
    \end{tabular}
    \label{tab:force_modeling}
\end{table}

\textbf{Impact of Force Modeling Strategies.} 
We further investigate how the generalized forces are modeled within the LDS module, comparing three paradigms: black-box estimation, unconstrained equation synthesis, and our physics-constrained synthesis. 
As shown in Tab.~\ref{tab:force_modeling}, the \textit{Direct Mapping} (a simple MLP taking concatenated states as input) yields limited gains over the baseline, suggesting that implicit black-box modeling struggles to capture the underlying physical causality. 
Explicitly structuring the estimation via the Lagrangian formulation (\textit{Unconstrained Eq.}) improves performance, yet remains suboptimal without geometric guarantees. 
Crucially, enforcing intrinsic physical constraints---specifically the positive definiteness of the inertia matrix and the skew-symmetry of the Coriolis term (\textit{Physics-Constrained Eq.})---provides a significant boost, confirming that structural validity is essential for learning robust dynamics. 
Finally, the Energy Consistency Loss ($\mathcal{L}_{EC}$) consistently enhances all equation-based models, with our full LaDy framework achieving the best results. This validates that combining structural physical constraints with energy-based supervision ensures the synthesized forces are not only semantically discriminative but also physically coherent.

\begin{table}[htbp]
    \footnotesize
    \centering
    \setlength{\tabcolsep}{3.5pt}
    \renewcommand{\arraystretch}{1.1}
    
    \caption{Ablation on Physical Constraints for Lagrangian Terms on PKU-MMD v2 (X-sub) dataset. We evaluate different constraint formulations for Inertia ($M$), Coriolis ($C$), and Gravity ($G$) estimation.}
    
    \begin{tabular}{ll|ccccc}
        \toprule
        \textbf{Term} & \textbf{Constraint Formulation} & \textbf{Acc} & \textbf{Edit} & \multicolumn{3}{c}{\textbf{F1@\{10, 25, 50\}}} \\
        \midrule
        \multirow{4}{*}{$\bm{M}$} 
         & \scriptsize{Unconstrained} & 75.4 & 74.6 & 79.6 & 76.6 & 66.0 \\
         & \scriptsize{Symmetric Only} & 75.3 & 74.1 & 79.3 & 76.3 & 66.0 \\
         & \scriptsize{Sym. + Semi-Positive Definite} & 75.9 & 74.2 & 78.9 & 76.5 & 65.9 \\
         & \scriptsize{\textbf{Sym. + Positive Definite (SPD)}} & \textbf{76.2} & \textbf{75.1} & \textbf{80.1} & \textbf{77.3} & \textbf{67.0} \\
        \midrule
        \multirow{3}{*}{$\bm{C}$} 
         & \scriptsize{Unconstrained $C$} & 75.5 & 74.2 & 79.6 & 76.7 & 66.2 \\
         & \scriptsize{Unconstrained $C = 0.5(\dot{M}-N)$} & 75.9 & 74.8 & 79.7 & 76.6 & 66.6 \\
         & \scriptsize{\textbf{Passivity ($\dot{M}-2C$ skew-sym.)}} & \textbf{76.2} & \textbf{75.1} & \textbf{80.1} & \textbf{77.3} & \textbf{67.0} \\
        \midrule
        \multirow{2}{*}{$\bm{G}$} 
         & \scriptsize{Conservative ($\nabla$ Potential Energy)} & 75.6 & 74.5 & 79.4 & 76.8 & 65.7 \\
         & \scriptsize{\textbf{Unconstrained (Direct MLP)}} & \textbf{76.2} & \textbf{75.1} & \textbf{80.1} & \textbf{77.3} & \textbf{67.0} \\
        \bottomrule
    \end{tabular}
    \label{tab:constraints}
\end{table}

\textbf{Effectiveness of Physical Constraints.} 
We further dissect the impact of imposing specific physical constraints on the constituent Lagrangian terms ($M$, $C$, and $G$). 
As detailed in Tab.~\ref{tab:constraints}, for the Inertia Matrix $M$, strictly enforcing \textbf{Symmetric Positive Definiteness (SPD)} is crucial, yielding the peak performance. Relaxing this to semi-definiteness or mere symmetry degrades performance, confirming that maintaining the valid Riemannian geometry of the mass matrix is fundamental for stable force synthesis.
For the Coriolis Matrix $C$, the \textbf{Passivity Constraint} (ensuring $\dot{M}-2C$ is skew-symmetric) outperforms unconstrained estimation. This structural coupling ensures the synthesized dynamics respect energy conservation principles, preventing artificial energy generation that would contradict our ECLoss.
Interestingly, for the Gravity vector $G$, the direct unconstrained mapping proves superior to the strictly conservative modeling via potential energy gradients ($\nabla E_P$). We hypothesize that while theoretically rigorous, learning a scalar potential field to derive forces introduces a challenging optimization landscape with unstable high-order derivatives, whereas the direct mapping offers a more flexible and trainable approximation for the complex latent gravity manifold.

\begin{table}[htbp]
    \footnotesize
    \centering
    \setlength{\tabcolsep}{4.5pt}
    \renewcommand{\arraystretch}{1.1}
    
    \caption{Ablation on the necessity of each Lagrangian term on PKU-MMD v2 (X-sub) dataset. We report results when removing the Inertia ($M$), Coriolis ($C$), Gravity ($G$), or Non-conservative ($F$) terms individually.}
    
    \begin{tabular}{l|ccccc}
        \toprule
        \textbf{Model Variant} & \textbf{Acc} & \textbf{Edit} & \multicolumn{3}{c}{\textbf{F1@\{10, 25, 50\}}} \\
        \midrule
        w/o Inertia ($M$) & 74.5 & 73.2 & 78.5 & 75.9 & 65.9 \\
        w/o Coriolis ($C$) & 75.5 & 74.4 & 78.9 & 76.2 & 66.2 \\
        w/o Gravity ($G$) & 75.0 & 73.9 & 79.1 & 76.0 & 65.9 \\
        w/o Non-Conservative ($F$) & 75.2 & 73.5 & 78.8 & 76.5 & 65.7 \\
        \textbf{Full Model (LaDy)} & \textbf{76.2} & \textbf{75.1} & \textbf{80.1} & \textbf{77.3} & \textbf{67.0} \\
        \bottomrule
    \end{tabular}
    \label{tab:terms_necessity}
\end{table}

\textbf{Necessity of Lagrangian Terms.} 
We evaluate the contribution of each dynamic component by selectively removing the Inertia ($M$), Coriolis ($C$), Gravity ($G$), or Non-conservative ($F$) terms from the LDS module. 
As shown in Tab.~\ref{tab:terms_necessity}, removing any single term leads to a distinct performance drop compared to the full model. 
This degradation is attributed to two factors: 
(1) \textbf{Dynamic Completeness}: Human motion is physically governed by the interplay of all these forces; omitting one yields a deficient dynamic representation that fails to capture the full kinematics-dynamics causality.
(2) \textbf{Energy Consistency}: The efficacy of our $\mathcal{L}_{EC}$ relies on the closed-loop Work-Energy theorem. A missing term disrupts this physical equilibrium, rendering the energy supervision mathematically ill-posed and less effective.
Notably, the exclusion of non-conservative forces ($F$) results in a significant decline (e.g., \textbf{-1.3\%} in F1@50). Although explicitly inferring external interactions and friction solely from skeletal pose is theoretically challenging, this result confirms that the $F$ branch successfully learns to approximate these critical dissipative effects, which are essential for recognizing actions involving object interactions or sudden stops.

\subsection{Ablation on Energy Consistency Loss}
\label{sec:supp:ECLoss}

Here, we investigate the numerical stability and optimization dynamics of $\mathcal{L}_{EC}$. Our analysis sequentially covers: (1) the internal loss formulation, validating the critical roles of relative normalization and noise masking; (2) the robustness of different regression metrics (e.g., Huber loss) against outlier noise; (3) the sensitivity analysis of the regularization weight $\lambda_3$; and (4) the necessity of the Delayed Physics Injection strategy for resolving the initialization conflict between kinematic learning and physical constraints.

\begin{table}[htbp]
    \footnotesize
    \centering
    \setlength{\tabcolsep}{4.5pt}
    \renewcommand{\arraystretch}{1.1}
    
    \caption{Ablation on the internal design of the Energy Consistency Loss ($\mathcal{L}_{EC}$) on PKU-MMD v2 (X-sub) dataset. We compare the naive formulation against versions with Relative Normalization and Noise Masking.}
    
    \begin{tabular}{l|ccccc}
        \toprule
        \textbf{Loss Formulation} & \textbf{Acc} & \textbf{Edit} & \multicolumn{3}{c}{\textbf{F1@\{10, 25, 50\}}} \\
        \midrule
        Naive Formulation & 75.7 & 74.3 & 79.7 & 76.9 & 66.6 \\
        w/ Relative Normalization & 75.9 & 74.6 & 79.6 & 76.9 & 66.8 \\
        w/ Noise Masking $\mathcal{M}(t)$ & 76.2 & 74.4 & 79.8 & 77.3 & 66.8 \\
        \textbf{Full Formulation (Ours)} & \textbf{76.2} & \textbf{75.1} & \textbf{80.1} & \textbf{77.3} & \textbf{67.0} \\
        \bottomrule
    \end{tabular}
    \label{tab:ecloss_design}
\end{table}

\textbf{Design of Energy Consistency Loss.} 
We strictly investigate the formulation of $\mathcal{L}_{EC}$ by dissecting its two key regularization components: relative normalization and the noise mask. 
As reported in Tab.~\ref{tab:ecloss_design}, a \textit{Naive Formulation} (Smooth-L1 distance between work and kinetic energy change) yields suboptimal results, primarily because the optimization is biased towards high-energy actions. 
Incorporating \textit{Relative Normalization} alleviates this by enforcing scale-invariant supervision, ensuring that subtle, low-energy motions contribute equally to the gradient, which notably improves boundary precision (Edit: \textbf{+0.3\%}). 
Furthermore, the \textit{Noise Masking} strategy $\mathcal{M}(t)$ proves critical for stability. By filtering out numerical singularities during static or micro-movement phases (where the energy denominator approaches zero), it prevents noise amplification and false penalties, significantly boosting frame-wise accuracy (Acc: \textbf{+0.5\%}). 
The \textit{Full Formulation}, combining both strategies, achieves the best trade-off between discriminability and boundary smoothness, validating that robust physical supervision requires both scale balance and noise resilience.

\begin{table}[htbp]
    \footnotesize
    \centering
    \setlength{\tabcolsep}{5pt}
    \renewcommand{\arraystretch}{1.1}
    
    \caption{Ablation on the distance metric for $\mathcal{L}_{EC}$ on PKU-MMD v2 (X-sub) dataset. We compare standard regression losses against the robust Huber loss.}
    
    \begin{tabular}{l|ccccc}
        \toprule
        \textbf{Distance Function} & \textbf{Acc} & \textbf{Edit} & \multicolumn{3}{c}{\textbf{F1@\{10, 25, 50\}}} \\
        \midrule
        $\ell_1$ Loss (MAE Loss) & 76.0 & 74.7 & 79.9 & 77.2 & 66.5 \\
        $\ell_2$ Loss (MSE Loss) & 76.0 & 74.9 & 80.0 & 77.4 & 67.0 \\
        \textbf{Huber Loss (Smooth $\ell_1$)} & \textbf{76.2} & \textbf{75.1} & \textbf{80.1} & \textbf{77.3} & \textbf{67.0} \\
        \bottomrule
    \end{tabular}
    \label{tab:loss_metric}
\end{table}

\textbf{Regression Loss for Energy Residual.} 
We ablate the distance function used to minimize the relative energy residual $r_E$. 
As shown in Tab.~\ref{tab:loss_metric}, the \textit{Huber Loss} (Smooth $\ell_1$) consistently outperforms both $\ell_1$ and $\ell_2$ losses. 
While the $\ell_2$ loss (MSE) provides strong gradients for large errors, it is overly sensitive to outliers---common in instantaneous power estimates due to sensor noise---leading to training instability. Conversely, the $\ell_1$ norm (MAE) is robust to outliers but suffers from optimization difficulties near zero due to its non-differentiability and constant gradient magnitude. 
The Huber loss strikes an optimal balance: it behaves quadratically for small residuals to ensure smooth, precise convergence towards physical equilibrium, while transitioning to a linear behavior for large deviations to maintain robustness against dynamic transients. This dual characteristic proves essential for learning physically coherent dynamics without being disrupted by transient noise.

\begin{table}[htbp]
    \footnotesize
    \centering
    \setlength{\tabcolsep}{6pt}
    \renewcommand{\arraystretch}{1.1}
    
    \caption{Ablation on the regularization weight $\lambda_3$ for $\mathcal{L}_{EC}$ on PKU-MMD v2 (X-sub) dataset.}
    
    \begin{tabular}{c|ccccc}
        \toprule
        \textbf{Coefficient $\lambda_3$} & \textbf{Acc} & \textbf{Edit} & \multicolumn{3}{c}{\textbf{F1@\{10, 25, 50\}}} \\
        \midrule
        1.0    & 75.9 & 74.2 & 79.8 & 76.9 & 66.3 \\
        \textbf{0.1} & \textbf{76.2} & \textbf{75.1} & \textbf{80.1} & \textbf{77.3} & \textbf{67.0} \\
        0.01   & 75.7 & 74.5 & 79.9 & 76.9 & 66.7 \\
        0.001  & 75.7 & 74.9 & 79.6 & 76.9 & 66.5 \\
        0.0001 & 76.1 & 75.0 & 79.8 & 77.1 & 66.6 \\
        \bottomrule
    \end{tabular}
    \label{tab:loss_weight}
\end{table}

\textbf{Sensitivity to Loss Weight $\lambda_3$.} 
We investigate the optimal balancing coefficient $\lambda_3$ for the Energy Consistency Loss ($\mathcal{L}_{EC}$). 
As presented in Tab.~\ref{tab:loss_weight}, the performance follows a bell-shaped trend, peaking at $\lambda_3 = 0.1$. 
Setting the weight too high ($\lambda_3=1$) leads to performance degradation (e.g., \textbf{-0.7\%} in F1@50), as the strong physical regularization overshadows the primary segmentation objective, forcing the network to prioritize equation satisfaction over semantic feature learning. 
Conversely, reducing the weight below $0.1$ diminishes the corrective impact of the physical constraints, yielding results similar to the unconstrained baseline. 
Thus, $\lambda_3=0.1$ provides the optimal trade-off, offering sufficient physical guidance to regularize the latent dynamics without disrupting the learning of discriminative spatio-temporal features.

\begin{table}[htbp]
    \footnotesize
    \centering
    \setlength{\tabcolsep}{6pt}
    \renewcommand{\arraystretch}{1.1}
    
    \caption{Ablation on the Delayed Physics Injection strategy on PKU-MMD v2 (X-sub) dataset. We evaluate the impact of Delayed Start ($\mathcal{Z}$) and Linear Warmup ($\mathcal{Z}_{w}$) on training stability.}
    
    \begin{tabular}{l|ccccc}
        \toprule
        \textbf{Injection Strategy} & \textbf{Acc} & \textbf{Edit} & \multicolumn{3}{c}{\textbf{F1@\{10, 25, 50\}}} \\
        \midrule
        Direct Injection & 75.5 & 74.0 & 78.8 & 76.4 & 66.1 \\
        Delayed Start ($\mathcal{Z}$) only & 75.9 & 74.7 & 79.7 & 76.9 & 66.8 \\
        Linear Warmup ($\mathcal{Z}_{w}$) only & 75.7 & 74.9 & 79.2 & 76.5 & 66.4 \\
        \textbf{Delayed + Phased Warmup} & \textbf{76.2} & \textbf{75.1} & \textbf{80.1} & \textbf{77.3} & \textbf{67.0} \\
        \bottomrule
    \end{tabular}
    \label{tab:warmup}
\end{table}

\textbf{Impact of Delayed Physics Injection.} 
We strictly validate the necessity of our \textit{Delayed and Phased Warmup} strategy for training stability. 
As shown in Tab.~\ref{tab:warmup}, directly imposing the Energy Consistency Loss from the first iteration (\textit{Direct Injection}) compromises performance (e.g., \textbf{-0.9\%} in F1@50 vs. Full). 
This deficit arises because the dynamic estimators ($\mathcal{F}_M$, $\mathcal{F}_C$, $\mathcal{F}_G$, $\mathcal{F}_F$) are randomly initialized; enforcing strict physical constraints on these erratic priors induces severe gradient conflicts, effectively disrupting the model's ability to learn basic kinematic representations in the early phase. 
Introducing a \textit{Delayed Start} ($\mathcal{Z}$) significantly boosts performance by allowing the network to first establish a coarse kinematic manifold before physical regularization kicks in. 
Adding \textit{Linear Warmup} ($\mathcal{Z}_{w}$) further smooths this transition. 
The full strategy, combining both, acts as a physical curriculum, ensuring that strict laws are enforced only after the dynamical priors have stabilized, thus achieving the optimal convergence.

\begin{table}[htbp]
    \footnotesize
    \centering
    \setlength{\tabcolsep}{5pt}
    \renewcommand{\arraystretch}{1.1}
    
    \caption{Ablation on the fusion stage for Spatial Modulation on PKU-MMD v2 (X-sub) dataset. We compare fusing dynamics at the raw input level (Early), after spatial-channel merging (Late), and after GCN encoding (Mid).}
    
    \begin{tabular}{l|ccccc}
        \toprule
        \textbf{Fusion Stage} & \textbf{Acc} & \textbf{Edit} & \multicolumn{3}{c}{\textbf{F1@\{10, 25, 50\}}} \\
        \midrule
        Early (Input-level) & 75.4 & 74.8 & 79.8 & 77.0 & 66.4 \\
        Late (Spatial Merging) & 75.7 & 74.8 & 79.6 & 76.7 & 66.4 \\
        \textbf{Mid (Post-GCN, Ours)} & \textbf{76.2} & \textbf{75.1} & \textbf{80.1} & \textbf{77.3} & \textbf{67.0} \\
        \bottomrule
    \end{tabular}
    \label{tab:spatial_mod}
\end{table}

\subsection{Ablation on Spatio-Temporal Modulation}
\label{sec:supp:STM}

Finally, we conduct an extensive architectural search to determine the optimal paradigm for integrating dynamic priors. We first identify the most effective fusion stage and mechanism for Spatial Modulation. Subsequently, we analyze the temporal gating topology, validating the hierarchical evolving design. We then dissect the dynamic signal composition, evaluating the synergy of Torque, Change, and Power, followed by an investigation into their optimal fusion logic and interaction strategies (Decoupled vs. Coupled). Lastly, we verify the sufficiency of channel-agnostic scalar gating.

\textbf{Optimal Stage for Spatial Modulation.} 
We investigate the most effective insertion point for fusing the synthesized dynamics ($F_{dyn}$) with the kinematic features. 
As shown in Tab.~\ref{tab:spatial_mod}, \textit{Early Fusion} (concatenating forces with raw input coordinates) yields suboptimal performance, likely due to the semantic gap between raw signals and latent features. 
Similarly, \textit{Late Fusion} (injecting dynamics after spatial-channel aggregation) also degrades results. By compressing the spatial dimension before fusion, this strategy aggregates local joint details into a global vector prematurely. This prevents the dynamic context from interacting with specific active joints, depriving the model of the ability to spatially localize dynamic effects. 
In contrast, our \textit{Mid-Level Fusion} (post-GCN) achieves the best performance. By injecting the global dynamic context into the uncompressed GCN features, we allow the physical intent to explicitly modulate specific local joint representations before spatial information is lost. This global-to-local modulation ensures that the dynamics can selectively emphasize kinematically relevant limbs, thereby maximizing spatial discriminability.

\begin{table}[htbp]
    \footnotesize
    \centering
    \setlength{\tabcolsep}{5pt}
    \renewcommand{\arraystretch}{1.1}
    
    \caption{Ablation on fusion mechanisms for Spatial Modulation on PKU-MMD v2 (X-sub) dataset. We compare our Channel Concatenation strategy against node-level concatenation, cross-attention, and element-wise operations.}
    
    \begin{tabular}{l|ccccc}
        \toprule
        \textbf{Fusion Mechanism} & \textbf{Acc} & \textbf{Edit} & \multicolumn{3}{c}{\textbf{F1@\{10, 25, 50\}}} \\
        \midrule
        Node Concatenation & 75.0 & 74.0 & 79.1 & 76.2 & 66.4 \\
        Element-wise Addition & 75.6 & 74.7 & 79.7 & 76.9 & 66.0 \\
        Element-wise Product & 75.9 & 73.9 & 79.2 & 76.3 & 66.2 \\
        Cross-Attention (Kin-Query) & 75.7 & 74.7 & 79.8 & 76.8 & 66.7 \\
        \textbf{Channel Concatenation (Ours)} & \textbf{76.2} & \textbf{75.1} & \textbf{80.1} & \textbf{77.3} & \textbf{67.0} \\
        \bottomrule
    \end{tabular}
    \label{tab:fusion_mech}
\end{table}

\textbf{Fusion Mechanism for Spatial Modulation.} 
We compare different operations for integrating the global synthesized dynamics with local kinematic features. 
As shown in Tab.~\ref{tab:fusion_mech}, naive \textit{Node Concatenation} (appending force as an extra vertex) performs poorly, as it fails to explicitly interact with body joints. 
Element-wise operations (\textit{Addition} and \textit{Product}) also yield poor results, likely due to feature conflict or information washout. 
Notably, while \textit{Cross-Attention} (using kinematics to adaptively weight dynamic features per joint) improves performance by enabling dynamic allocation, it is still outperformed by our proposed \textit{Channel Concatenation}. 
We attribute this to the fact that concatenating the broadcasted global dynamics preserves the complete, uncompressed information of both modalities. This allows the subsequent spatial-channel mixing layers to learn unrestricted, high-dimensional non-linear correlations between forces and poses. This ``learnable fusion'' proves more robust and expressive than the constrained inductive bias of soft-attention mechanisms, effectively baking the global dynamic intent into the spatial representation.

\begin{table}[htbp]
    \footnotesize
    \centering
    \setlength{\tabcolsep}{3.5pt}
    \renewcommand{\arraystretch}{1.1}
    
    \caption{Ablation on the architecture of Temporal Modulation on PKU-MMD v2 (X-sub) dataset. We compare single-stage gating against multi-stage strategies with static, parallel, or hierarchical topologies.}
    
    \begin{tabular}{l|ccccc}
        \toprule
        \textbf{Gating Architecture} & \textbf{Acc} & \textbf{Edit} & \multicolumn{3}{c}{\textbf{F1@\{10, 25, 50\}}} \\
        \midrule
        Single-Stage (Input only) & 75.5 & 74.4 & 79.4 & 76.7 & 66.5 \\
        Single-Stage (Output only) & 75.3 & 74.2 & 79.1 & 76.5 & 66.3 \\
        Multi-Stage (Static/Repeated) & 76.0 & 74.4 & 79.9 & 77.0 & 66.5 \\
        Multi-Stage (Parallel Projection) & 75.8 & 74.2 & 79.6 & 76.6 & 66.6 \\
        \textbf{Multi-Stage (Hierarchical Evolving)} & \textbf{76.2} & \textbf{75.1} & \textbf{80.1} & \textbf{77.3} & \textbf{67.0} \\
        \bottomrule
    \end{tabular}
    \label{tab:temporal_arch}
\end{table}

\textbf{Architecture of Temporal Modulation.} 
We analyze the structural design of the temporal gating mechanism to determine the optimal topology for injecting dynamic cues across the $L$ temporal stages. 
As shown in Tab.~\ref{tab:temporal_arch}, limiting gating to a \textit{Single-Stage} (Input or Output temporal modeling stage) is insufficient, as it fails to regulate the intermediate feature evolution. 
Extending to all stages improves results, yet the topology matters. 
The \textit{Static} strategy (using the same initial signal for all $L$ temporal gating layers) lacks adaptability. 
The \textit{Parallel} strategy (using $L$ independent heads to transform the initial signal for each gate) also falls short. We attribute this to a topological mismatch: the temporal backbone processes features serially, accumulating receptive fields and semantics layer-by-layer, whereas parallel gating treats each stage independently, failing to capture this accumulated context. 
Our \textit{Hierarchical Evolving} strategy achieves the best performance by structurally mirroring the backbone. By treating the gate as a serial stream that evolves alongside the main features, we ensure precise receptive field alignment and semantic co-evolution. This allows the dynamic signals to mature from local, sharp cues in shallow layers to abstract, boundary-aware signals in deep layers, perfectly matching the needs of the temporal hierarchy.

\begin{table}[htbp]
    \footnotesize
    \centering
    \setlength{\tabcolsep}{7.5pt}
    \renewcommand{\arraystretch}{1.1}
    
    \caption{Ablation on the combination of Salient Dynamic Signals for temporal gating on PKU-MMD v2 (X-sub) dataset. We evaluate the synergistic effect of Torque ($g_\tau$), Torque Change ($g_{\dot{\tau}}$), and Power ($g_P$).}
    
    \begin{tabular}{ccc|ccccc}
        \toprule
        \multicolumn{3}{c|}{\textbf{Dynamic Signals}} & \multirow{2}{*}{\textbf{Acc}} & \multirow{2}{*}{\textbf{Edit}} & \multicolumn{3}{c}{\multirow{2}{*}{\textbf{F1@\{10, 25, 50\}}}} \\
        $g_\tau$ & $g_{\dot{\tau}}$ & $g_P$ & & &  &  &  \\
        \midrule
        \checkmark &- &- & 75.2 & 74.2 & 79.2 & 75.9 & 66.2 \\
        -& \checkmark &- & 74.7 & 74.1 & 78.9 & 76.0 & 66.2 \\
        -& -& \checkmark & 75.3 & 74.5 & 79.8 & 76.9 & 66.7 \\
        \midrule
        \checkmark & \checkmark & -& 75.2 & 74.1 & 79.1 & 76.0 & 66.3 \\
        \checkmark & -& \checkmark & 75.4 & 74.6 & 79.8 & 76.8 & 66.6 \\
        -& \checkmark & \checkmark & 76.0 & 74.8 & 79.8 & 76.9 & 66.9 \\
        \midrule
        \checkmark & \checkmark & \checkmark & \textbf{76.2} & \textbf{75.1} & \textbf{80.1} & \textbf{77.3} & \textbf{67.0} \\
        \bottomrule
    \end{tabular}
    \label{tab:dynamic_signals}
\end{table}

\textbf{Impact of Salient Dynamic Signals.} 
We investigate the contribution of different dynamic cues---Torque ($g_\tau$), Torque Change ($g_{\dot{\tau}}$), and Power ($g_P$)---to the temporal gating mechanism. 
As detailed in Tab.~\ref{tab:dynamic_signals}, individual signals yield limited gains, as they capture isolated physical aspects: $g_\tau$ reflects actuation magnitude, $g_{\dot{\tau}}$ detects transient shifts, and $g_P$ measures energy intensity. 
Pairwise combinations provide partial improvements, but the \textit{Full Combination} achieves the best performance (e.g., \textbf{+0.8\%} F1@50 vs. single $g_{\dot{\tau}}$). 
This confirms that these signals are mutually complementary rather than redundant. By integrating all three, the model constructs a holistic dynamic profile that simultaneously encodes the ``how much'' (Torque/Power) and ``when'' (Torque Change) of the motion. This synergy enables the temporal hierarchy to adaptively attend to both high-energy semantic segments and subtle boundary transitions, maximizing segmentation precision.

\begin{table}[htbp]
    \footnotesize
    \centering
    \setlength{\tabcolsep}{4pt}
    \renewcommand{\arraystretch}{1.1}
    
    \caption{Ablation on the feature fusion strategy within the Temporal Modulation module on PKU-MMD v2 (X-sub) dataset. We compare linear summation methods against learnable concatenation-based fusion.}
    
    \begin{tabular}{l|ccccc}
        \toprule
        \textbf{Fusion Strategy} & \textbf{Acc} & \textbf{Edit} & \multicolumn{3}{c}{\textbf{F1@\{10, 25, 50\}}} \\
        \midrule
        Naive Summation & 75.4 & 74.7 & 79.6 & 76.6 & 66.1 \\
        Static Weighted Sum & 75.3 & 74.2 & 79.3 & 76.5 & 66.1 \\
        Adaptive Weighted Sum & 75.4 & 74.9 & 80.0 & 77.0 & 66.6 \\
        \textbf{Concat + Conv Fusion (Ours)} & \textbf{76.2} & \textbf{75.1} & \textbf{80.1} & \textbf{77.3} & \textbf{67.0} \\
        \bottomrule
    \end{tabular}
    \label{tab:feature_fusion}
\end{table}

\textbf{Feature Fusion in Temporal Modulation.} 
We analyze how to effectively recombine the three modulated feature streams (gated by Torque, Change, and Power) within each temporal stage. 
As shown in Tab.~\ref{tab:feature_fusion}, simple \textit{Naive Summation} or \textit{Static Weighted Sum} (learnable weight parameters) yields lower performance, as these linear superpositions enforce a fixed interaction mode that ignores frame-specific variations. 
While \textit{Adaptive Weighted Sum} (generating instance-specific weights via an MLP) improves results by introducing dynamic adaptivity, it remains restricted to a linear combination, leading to potential feature conflict where distinct dynamic cues might cancel each other out. 
Our \textit{Concat + Conv Fusion} outperforms all summation-based methods. We attribute this to its ability to perform non-linear synthesis: by concatenating the features, we preserve the independent semantic subspaces of each dynamic cue without premature merging. The subsequent convolution then acts as a learnable projection, enabling the model to construct complex, high-dimensional correlations between power, torque, and transients, thus maximizing the expressive power of the fused representation.

\begin{table}[htbp]
    \footnotesize
    \centering
    \setlength{\tabcolsep}{4pt}
    \renewcommand{\arraystretch}{1.1}
    
    \caption{Ablation on the interaction strategy of dynamic signals within the hierarchical gating mechanism on PKU-MMD v2 (X-sub) dataset. We compare merging signals (Fusion), interacting channels (Coupled), and independent processing (Decoupled).}
    
    \begin{tabular}{l|ccccc}
        \toprule
        \textbf{Interaction Strategy} & \textbf{Acc} & \textbf{Edit} & \multicolumn{3}{c}{\textbf{F1@\{10, 25, 50\}}} \\
        \midrule
        Early Fusion (Input-level) & 75.4 & 74.2 & 79.3 & 76.4 & 66.5 \\
        Step-wise Fusion (Gate-level) & 74.8 & 74.3 & 79.1 & 76.0 & 65.9 \\
        Coupled Evolution (Inter-signal) & 75.4 & 74.7 & 80.0 & 76.8 & 66.7 \\
        \textbf{Decoupled Evolution (Ours)} & \textbf{76.2} & \textbf{75.1} & \textbf{80.1} & \textbf{77.3} & \textbf{67.0} \\
        \bottomrule
    \end{tabular}
    \label{tab:signal_interaction}
\end{table}

\textbf{Interaction Strategy for Dynamic Signals.} 
We investigate how the three salient dynamic signals (Torque, Change, and Power) should interact during the hierarchical gating process. 
As shown in Tab.~\ref{tab:signal_interaction}, strategies that merge signals into a single gate channel, either initially (\textit{Early Fusion}) or at each stage (\textit{Step-wise Fusion}), yield inferior performance. This suggests that forcing distinct physical cues to compete within a single-channel bottleneck dilutes their specific semantic roles. 
Furthermore, allowing three signals to interact during evolution via channel concatenation and convolution (\textit{Coupled Evolution}) also proves suboptimal. We attribute this to feature homogenization: mixing heterogeneous physical quantities (e.g., energy magnitude vs. derivative transients) blurs their distinct functional properties. 
Our \textit{Decoupled Evolution} strategy achieves the best results by treating the signals as heterogeneous, independent streams. By preserving the ``physical purity'' of each cue throughout the hierarchy, we enable the network to learn specialized modulations---where one stream exclusively highlights high-energy frames and another isolates boundary transitions---maximizing their complementarity when fused at the high-level feature space.

\begin{table}[htbp]
    \footnotesize
    \centering
    \setlength{\tabcolsep}{5pt}
    \renewcommand{\arraystretch}{1.1}
    
    \caption{Ablation on the channel dimensionality of dynamic gating signals on PKU-MMD v2 (X-sub) dataset. We compare channel-wise modulation ($C$ channels) against channel-agnostic scalar modulation (1 channel).}
    
    \begin{tabular}{l|ccccc}
        \toprule
        \textbf{Signal Dimensionality} & \textbf{Acc} & \textbf{Edit} & \multicolumn{3}{c}{\textbf{F1@\{10, 25, 50\}}} \\
        \midrule
        Multi-Channel ($T \times C$) & 76.0 & 75.0 & 80.1 & 77.0 & 66.8 \\
        \textbf{Channel-Agnostic ($T \times 1$)} & \textbf{76.2} & \textbf{75.1} & \textbf{80.1} & \textbf{77.3} & \textbf{67.0} \\
        \bottomrule
    \end{tabular}
    \label{tab:signal_channel}
\end{table}

\textbf{Channel Dimensionality of Dynamic Signals.} 
We assess whether the dynamic gating signals should retain channel-wise granularity or be compressed into channel-agnostic scalars. 
As shown in Tab.~\ref{tab:signal_channel}, the \textit{Multi-Channel Gating} strategy (mapping dynamics to $C$ channels for element-wise modulation) yields slightly inferior performance compared to the scalar approach. 
While channel-wise gating offers theoretical flexibility, it increases model complexity and risks overfitting to channel-specific noise. 
In contrast, our \textit{Channel-Agnostic Gating} (compressing dynamics to 1D scalars via norm aggregation) achieves the best results. By aggregating dynamic energy across all dimensions, this strategy forces the gate to focus purely on frame-level temporal salience rather than channel variations. This global perspective ensures that the modulation acts as a stable temporal attention mechanism, prioritizing key timestamps (e.g., action boundaries) where the aggregate dynamic profile shifts, thus maximizing boundary precision with lower computational overhead.

\section{Discussion}
\label{sec:discussion}

In this section, we present a broader discussion to contextualize the versatility and boundaries of the proposed LaDy framework. Specifically, we first explore its generalization capabilities across different tasks and non-human subjects (Sec.~\ref{sec:Generalization}), followed by a critical analysis of its failure cases, intrinsic limitations, and directions for future research (Sec.~\ref{sec:limit_future}).

\subsection{Generalization Capabilities}
\label{sec:Generalization}

To demonstrate the versatility of the proposed framework, we analyze its generalizability across different downstream tasks and non-human subject domains.

\textbf{Generalization to Other Tasks: Action Recognition.}
While LaDy is primarily designed for action segmentation, its core physical components are inherently task-agnostic. The Lagrangian Dynamics Synthesis (LDS) module can serve as a general-purpose, physics-informed feature enhancer for any skeleton-based architecture. Specifically, the ``dynamic signatures'' extracted by the LDS (e.g., distinct torque profiles for actions like ``throwing'' or ``clapping'') explicitly model the causal intent of the motion, providing a strong prior that boosts inter-class discriminability.

To validate this empirically, we conduct preliminary experiments on the skeleton-based action recognition task. We integrate the LDS module into the established CTR-GCN~\cite{CTR-GCN} baseline and evaluate it on the widely used NTU-RGB+D (X-Sub)~\cite{NTU-RGBD} dataset. The implementation simply involves deriving the generalized forces supervised by the Energy Consistency Loss ($\mathcal{L}_{EC}$) and injecting them into the early stages of the CTR-GCN backbone via our Spatial Modulation (SM) mechanism. As reported in Tab.~\ref{tab:SAR}, equipping the baseline with our physical modules yields consistent performance improvements across all input modalities (Joint, Bone, and their temporal motions) and the final multi-stream ensemble. Notably, these gains are achieved seamlessly without dataset-specific hyperparameter tuning, validating that modeling physical dynamics provides a universally beneficial inductive bias for general action understanding.

\begin{table}[htbp]
    \footnotesize
    \centering
    \setlength{\tabcolsep}{4pt}
    \renewcommand{\arraystretch}{1.1}
    \caption{Preliminary evaluation of LaDy as a plug-and-play feature enhancer for skeleton-based action recognition on the NTU-RGB+D (X-Sub) dataset. The integration consistently improves the CTR-GCN baseline across all individual modalities and the multi-stream ensemble.}
    
    \begin{tabular}{l|ccccc}
        \toprule
        \textbf{Method} & \textbf{Joint}& \textbf{Bone}& \textbf{J-Motion}& \textbf{B-Motion}& \textbf{Ensemble4}\\
        \midrule
        CTR-GCN& 90.1& 90.2& 87.8& 87.2& 92.4\\
        \textbf{+LaDy (Ours)}& \textbf{90.3}& \textbf{90.4}& \textbf{88.3}& \textbf{88.9}& \textbf{92.9}\\
        \bottomrule
    \end{tabular}
    \label{tab:SAR}
\end{table}

\textbf{Generalization to Non-Human Subjects.}
The proposed framework is not fundamentally restricted to human biomechanics. The core Lagrangian dynamic formulation (Eq.~\eqref{Eq:MCFG}) relies mathematically on the definition of an open kinematic chain, which is represented by the skeleton graph topology $\mathcal{G}$. Because the underlying physical axioms---such as energy conservation, mass-inertia positive definiteness, and torque generation---apply universally to any articulated rigid-body system, the framework is theoretically subject-agnostic. By simply redefining the joint connectivity within $\mathcal{G}$ and establishing a corresponding root orientation to match a target morphology (e.g., a quadrupedal animal or a robotic manipulator), the LDS module can be directly transferred to analyze animal behavior or robotic motion. This universal adaptability underscores the broader potential of physics-informed modeling in multi-domain motion analysis.

\subsection{Failure Cases, Limitations, and Future Work}
\label{sec:limit_future}

\textbf{Failure Cases and Intrinsic Limitations.}
While LaDy establishes a robust physics-informed paradigm for action segmentation, an analysis of its failure cases (as detailed in Sec.~\ref{sec:supp:PPA} and Sec.~\ref{sec:supp:EQA}) reveals specific intrinsic limitations. Performance bottlenecks primarily manifest in the following scenarios:
(1) \textbf{Inaccurate External Force Estimation:} For interaction-heavy actions (e.g., \textit{Kicking Something}), sudden kinematic shifts are dictated by unmodeled external objects rather than internal actuation. Inferring the non-conservative force $F(q, \dot{q})$ solely from skeletal kinematics without explicit external contact sensing remains fundamentally ill-posed, occasionally resulting in transient violations of the estimated dynamics that confuse the network.
(2) \textbf{Low-Energy Dynamics:} In passive or micro-movements (e.g., \textit{Use a fan}), the actual physical actuation is minimal, leading to a naturally low signal-to-noise ratio in the estimated joint torques. In such cases, strong kinematic periodicity dominates, and enforcing auxiliary dynamic modulation can introduce minor interference.
(3) \textbf{Ambiguous ``Soft'' Transitions:} During gradual motion changes (e.g., slowly transitioning from a stance to a walk), the dynamic force profile shifts smoothly rather than abruptly. The absence of sharp force transients makes the precise temporal localization of a boundary inherently ambiguous, even for physics-guided architectures.
(4) \textbf{Semantic Overlaps:} Sporadic classification errors persist among semantically overlapping classes, indicating that resolving kinematic ambiguities via dynamics is sometimes insufficient without broader contextual understanding.

\textbf{Future Work.}
Future research will address these challenges by extending the current framework into a physics-grounded multimodal architecture. To resolve the ambiguity of external forces, we plan to integrate visual context or object-centric representations to explicitly model interaction dynamics, thereby grounding the non-conservative force term $F$ in observable physical contacts. Additionally, synergizing our low-level dynamic priors with the high-level semantic reasoning capabilities of Large Multimodal Models (LMMs) will help resolve residual categorical confusions, moving towards a holistic framework that understands both the \textit{mechanics} of how humans move and the \textit{semantics} of why they move. To better handle low-energy actions, we aim to explore adaptive gating mechanisms that dynamically adjust the network's reliance on the physics branch based on the instantaneous energy scale of the motion. Finally, to address the inherent ambiguity of soft transitions, future iterations could transition from deterministic boundary localization to probabilistic temporal modeling, allowing the network to explicitly quantify transition uncertainties.

\end{document}